\pdfoutput=1
\documentclass{article}

\usepackage{arxiv}

\usepackage[utf8]{inputenc}
\usepackage[T1]{fontenc}
\usepackage{url}
\usepackage{booktabs}
\usepackage{amsmath}
\usepackage{amsfonts}
\usepackage{amssymb}
\usepackage{mathrsfs}
\usepackage{amsthm}
\usepackage{nicefrac}
\usepackage{microtype}
\usepackage{graphicx}
\usepackage{fontawesome5}
\usepackage[square,numbers]{natbib}
\usepackage{tikz}
\usetikzlibrary{matrix,positioning,arrows.meta,fit,calc,shapes.geometric,shadows,patterns}
\usepackage{subcaption}
\usepackage{pgfplots}
\pgfplotsset{compat=1.18}
\usepackage[skins,breakable]{tcolorbox}
\usepackage{todonotes}
\usepackage[normalem]{ulem}
\usepackage{setspace}
\usepackage[scaled=0.7]{beramono}
\usepackage{xcolor}
\usepackage{tabularx}
\usepackage{xspace}
\usepackage{hyperref}
\usepackage{doi}
\usepackage[capitalise,noabbrev]{cleveref}

\definecolor{ravengold}{HTML}{9A7016} 
\definecolor{ravenink}{HTML}{252923}  

\newcommand{\modelname}[1]{\mbox{#1}}
\newcommand{\raven}{Raven\xspace}
\newcommand{\ravenresearch}{Raven-Research\xspace}
\newcommand{\ravencode}{Raven-Code\xspace}
\newcommand{\ravendesign}{Raven-Design\xspace}
\newcommand{\ravenoncall}{Raven-Oncall\xspace}

\newcounter{ravenalgorithm}
\newenvironment{ravenalgorithm}[1]{%
    \begin{figure}[!htbp]
    \refstepcounter{ravenalgorithm}%
    \begin{tcolorbox}[
        enhanced,colback=white,colframe=black!45,
        boxrule=0.5pt,arc=1pt,left=7pt,right=7pt,top=5pt,bottom=5pt,
        colbacktitle=black!4,coltitle=black,fonttitle=\bfseries,
        title={Algorithm~\theravenalgorithm: #1}]
    \small
}{\end{tcolorbox}\end{figure}}
\crefname{ravenalgorithm}{Algorithm}{Algorithms}

\theoremstyle{plain}
\newtheorem{theorem}{Theorem}
\newtheorem{lemma}{Lemma}
\newtheorem{corollary}{Corollary}
\newtheorem{proposition}{Proposition}
\theoremstyle{definition}
\newtheorem{definition}{Definition}
\theoremstyle{plain}
\crefname{theorem}{Theorem}{Theorems}
\crefname{lemma}{Lemma}{Lemmas}
\crefname{corollary}{Corollary}{Corollaries}
\crefname{proposition}{Proposition}{Propositions}
\crefname{definition}{Definition}{Definitions}
\crefname{section}{Section}{Sections}
\crefname{subsection}{Section}{Sections}
\crefname{subappendix}{Appendix}{Appendices}
\crefname{appendix}{Appendix}{Appendices}
\crefname{equation}{Eq.}{Eqs.}

\newcolumntype{Y}{>{\centering\arraybackslash}X} 
\newcolumntype{Z}{>{\raggedleft\arraybackslash}X} 

\definecolor{promptcolor}{RGB}{255, 152, 0}
\definecolor{thinkingcolor}{RGB}{33, 150, 243}
\definecolor{answercolor}{RGB}{76, 175, 80}
\definecolor{systempromptcolor}{RGB}{244, 67, 54}
\newcommand{\berafamily}{\fontfamily{fvm}\selectfont}

\newtcolorbox{promptBox}[1][User Prompt]{
    enhanced,
    breakable,
    colback=promptcolor!15,
    colframe=promptcolor!80!black,
    boxrule=0.8pt,
    fonttitle=\bfseries,
    title=#1,
    sharp corners,
    before upper={\begin{spacing}{0.9}\footnotesize\berafamily},
    after upper={\end{spacing}}
}

\newtcolorbox{systemPromptBox}[1][System Prompt]{
    enhanced,
    breakable,
    colback=systempromptcolor!15,
    colframe=systempromptcolor!80!black,
    boxrule=0.8pt,
    fonttitle=\bfseries,
    title=#1,
    sharp corners,
    before upper={\begin{spacing}{0.9}\footnotesize\berafamily},
    after upper={\end{spacing}}
}

\newtcolorbox{thinkingBox}[1][Chain-of-Thought]{
    enhanced,
    breakable,
    colback=thinkingcolor!10,
    colframe=thinkingcolor!75!black,
    boxrule=0.8pt,
    fonttitle=\bfseries,
    title=#1,
    sharp corners,
    before upper={\begin{spacing}{0.9}\footnotesize\berafamily},
    after upper={\end{spacing}}
}

\newtcolorbox{answerBox}[1][Final Answer]{
    enhanced,
    breakable,
    colback=answercolor!15,
    colframe=answercolor!70!black,
    boxrule=0.8pt,
    fonttitle=\bfseries,
    title=#1,
    sharp corners,
    before upper={\begin{spacing}{0.9}\footnotesize\berafamily},
    after upper={\end{spacing}}
}

\title{Raven: The Harness of Harnesses for Composable Agentic Intelligence}
\author{%
  EverMind AI\\[0.6em]
  {\normalfont\small
    \href{https://evermind.ai/}{\textcolor{blue!50!black}{\nolinkurl{https://evermind.ai/}}}}\\[0.6em]
  {\normalfont\small
    \href{https://github.com/EverMind-AI/Raven}{%
      \textcolor{ravenink}{\faGithub}\hspace{0.6em}%
      \textcolor{blue!50!black}{\nolinkurl{https://github.com/EverMind-AI/Raven}}}}%
}

\renewcommand{\headeright}{}
\renewcommand{\undertitle}{}

\fancypagestyle{evermindtitle}{%
  \fancyhf{}
  \fancyhead[L]{%
    \raisebox{-0.5\height}{\includegraphics[height=24pt]{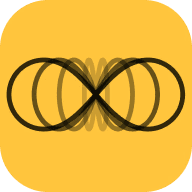}}%
    \hspace{6pt}%
    \raisebox{-0.5\height}{\includegraphics[height=14pt]{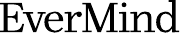}}%
  }

}

\hypersetup{
  colorlinks=true,
  linkcolor=ravengold,
  citecolor=ravengold,
  urlcolor=ravengold,
  pdftitle={Raven: The Harness of Harnesses for Composable Agentic Intelligence},
  pdfauthor={EverMind AI},
}

\begin{document}
\raggedbottom

\maketitle
\thispagestyle{evermindtitle}
\suppressfloats[t]

\begin{abstract}
As large language models advance, AI agents are moving beyond isolated, domain-specific tasks toward long-horizon, cross-domain workflows.
This transition exposes two challenges: increasing harness complexity makes manual design difficult to scale, while tighter coupling to specific domains limits the generality of a single harness.
The central question thus shifts from how to engineer a stronger harness for one domain to how to autonomously construct specialized harnesses, improve them through experience, and orchestrate them across domains.
We introduce \raven, \emph{The Harness of Harnesses}, an open-source
multi-agent ecosystem that automatically constructs and evolves modular
harnesses for specific models and domains, treating each executable
model--harness pair as a composable unit of intelligence.
To support an \emph{All-Domain Collaboration Network}, its Host Agent
decomposes goals, matches subtasks to specialized agents, coordinates
execution dependencies, and integrates results, while a host archive and EverOS preserve experience across tasks and Skill Forge makes that experience available as reusable procedures.
Our theory establishes sufficient conditions for such composition to
expand reliable task coverage beyond that of the available individual agents
under a shared resource budget.
On complex and long-horizon tasks, \raven significantly outperforms the state-of-the-art agent systems, pushing the frontier of composable agentic intelligence.
\end{abstract}

\par\medskip
\noindent\begin{minipage}{\linewidth}
  \centering
  \includegraphics[width=\linewidth]{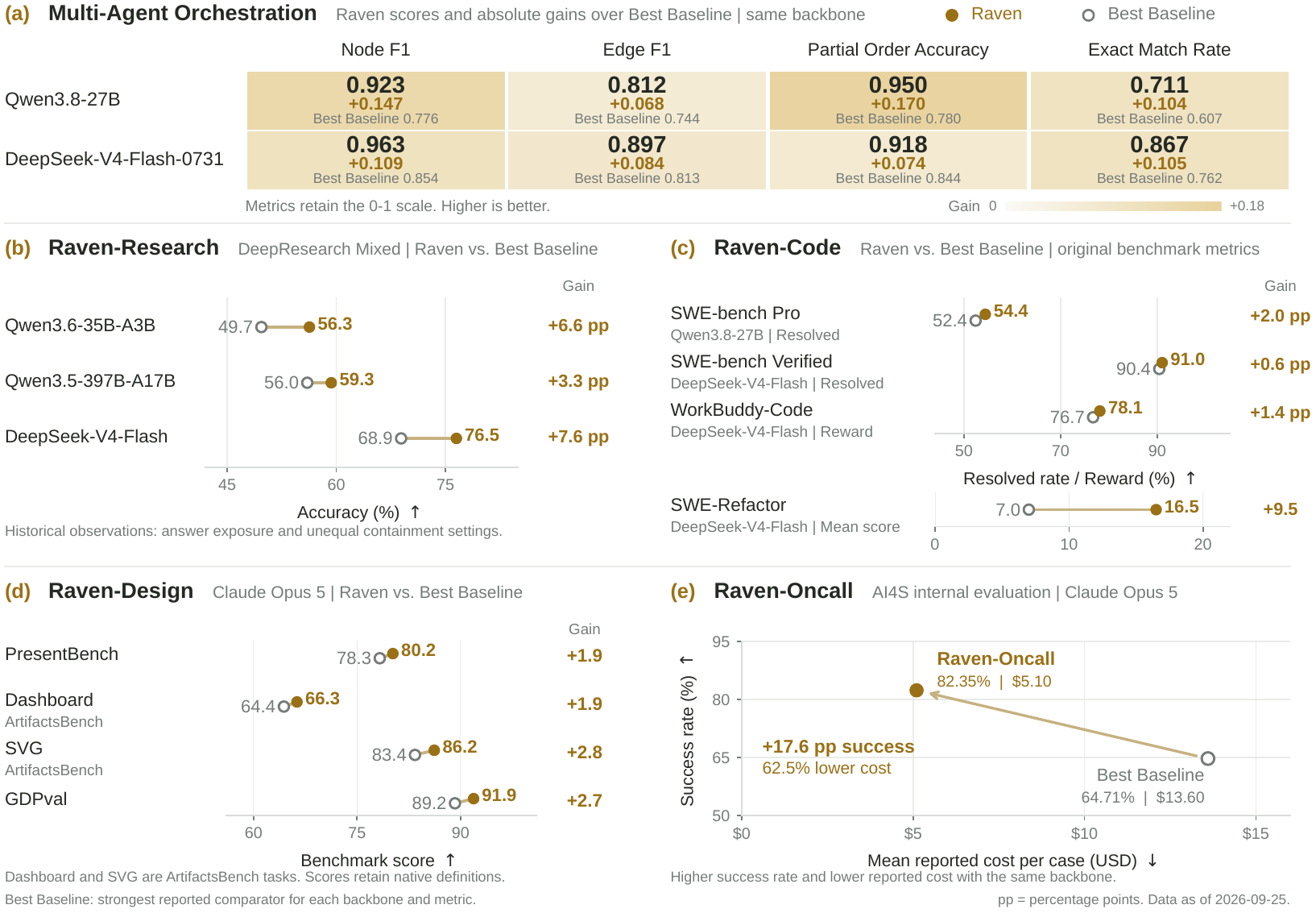}
  \captionsetup{font=small,skip=6pt,hypcap=false}
  \captionof{figure}{Raven performance overview. Selected results for
  multi-agent orchestration, research, coding, design, and on-call tasks,
  compared with the strongest reported alternative in each setting.}
  \label{fig:performance-overview}
\end{minipage}
\par\medskip

\clearpage
\pagestyle{empty}
\setcounter{tocdepth}{3}
\tableofcontents
\clearpage
\pagestyle{fancy}

\section{Introduction}\label{sec:introduction}
\suppressfloats[t]

Advances in large language models (LLMs), including instruction following
\citep{ouyang2022training} and code generation \citep{chen2021evaluating},
have provided a foundation for agents that interpret user goals and act
through software. LLMs can also learn to invoke external tools
\citep{schick2023toolformer}. Agent systems organize these capabilities
into multi-step interactions with external environments.
For example, interleaving reasoning with actions and observations allows
an agent to gather information and revise its decisions in response to
feedback \citep{yao2023react}. Specialized execution interfaces further shape
what agents can accomplish, as demonstrated in automated software
engineering \citep{yang2024sweagent}. These developments raise the question
of how agents with different expertise and execution mechanisms can work
toward a shared goal.

Complex goals often require several specializations within one workflow
\citep{fourney2024magentic}.
Developing and operating a first-person shooter (FPS) game, for example,
can involve requirements research, gameplay programming, visual design,
integration testing, and sustained operation. Each stage has different
tools and completion criteria, and its outputs must support the work
that follows. An agent's practical capability therefore depends on both
its model and its \emph{harness}: the tool interfaces, context management,
skills, execution policies, and recovery mechanisms surrounding the model.
These mechanisms govern how an agent applies its model's capabilities
within a domain~\citep{wang2023survey}.

This dependence on the harness creates two related challenges. First,
supporting more tools and execution conditions increases the number of
interacting design choices that must be configured and validated, making
manual harness development difficult to scale across models and domains
\citep{shang2025agentsquare,hu2024adas}.
Second, specialized harnesses encode assumptions about their tools,
working context, and expected outputs. Combining their capabilities
requires matching subtasks to suitable executors and establishing
compatible handoffs. Analyses of multi-agent failures identify inter-agent
misalignment and inadequate task verification as recurring problems
\citep{cemri2025mast}. Multi-agent conversation and role-based workflows
provide mechanisms for organizing collaboration
\citep{wu2023autogen,hong2023metagpt}. However, the benefit of a composition
still depends on task structure, local capabilities, and coordination cost
\citep{kim2026scaling}. The central problem is thus to construct and improve
specialized harnesses while making their capabilities composable across
domains.

To address these challenges, we introduce \raven, \emph{The Harness Of
Harnesses}, an open-source multi-agent ecosystem that treats each
executable model--harness pair as a unit of composition. As illustrated
in \Cref{fig:raven-ecosystem}(a), the ecosystem includes the native
specialists \ravenresearch, \ravencode,
\ravendesign, and \ravenoncall, alongside independently developed agents
such as Claude Code, Codex, Hermes Agent, and OpenClaw. Execution adapters
connect these agents to a shared orchestration interface while preserving
their tools and internal execution policies. The agent registry
describes their capabilities to the Host Agent and can accommodate
additional agents as they are integrated.

\begin{figure}[t]
    \centering
    \includegraphics[width=\linewidth]{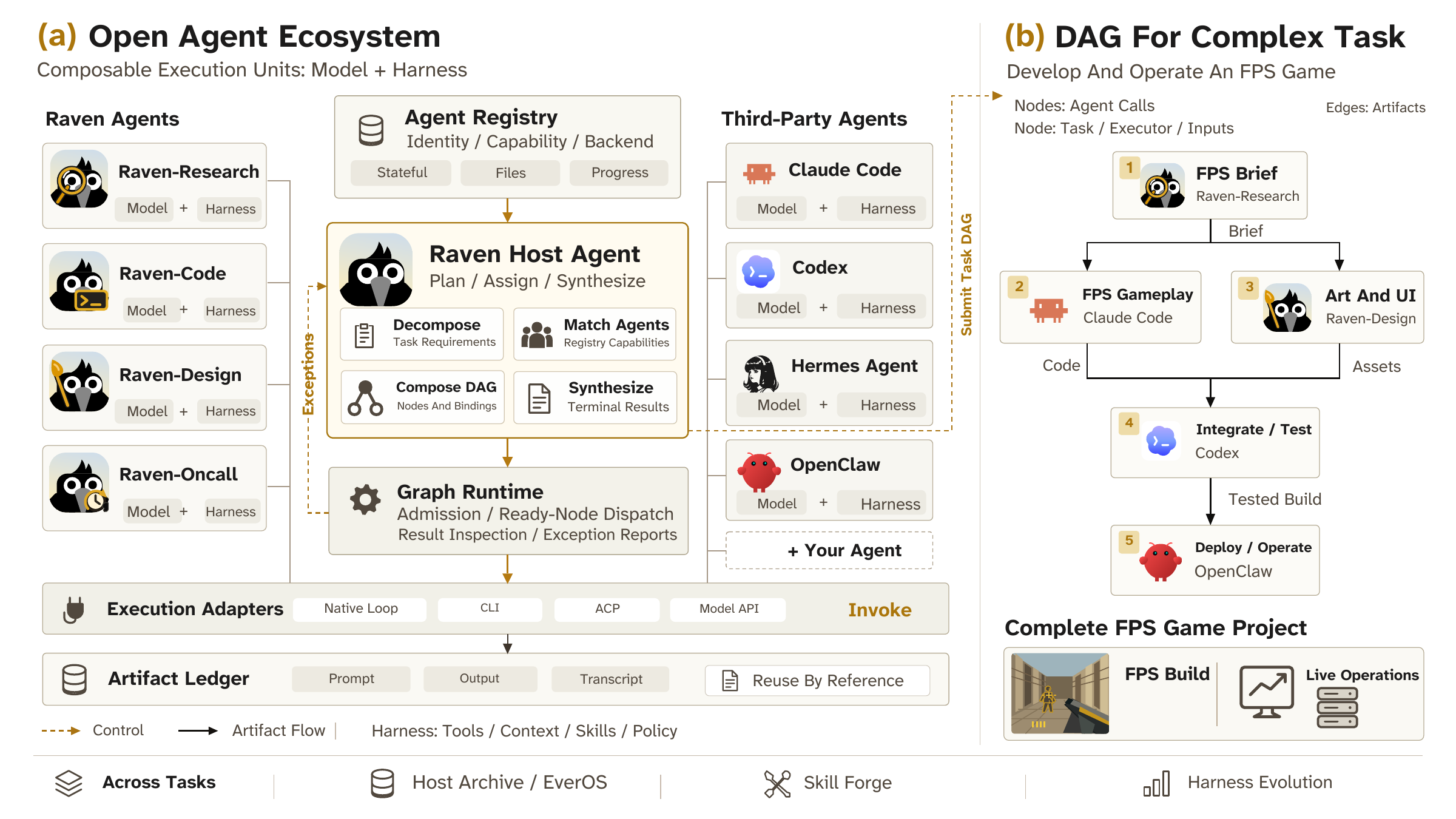}
    \caption{\textbf{Raven As An Open Multi-Agent Ecosystem.}
    (a) Built-in and third-party agents expose executable model--harness
    pairs through execution adapters. The Host Agent uses registry
    capabilities to compose and assign tasks, while the runtime validates
    the plan, schedules ready nodes, and records artifacts for handoff
    and reuse. (b) An illustrative FPS game project combines gameplay
    development and visual design before integration, testing, and
    operation. Nodes denote agent calls, and directed edges denote
    artifact dependencies coordinated by Raven. Persistent context,
    reusable skills, and harness evolution support adaptation across tasks.}
    \label{fig:raven-ecosystem}
\end{figure}

Within this ecosystem, the Host Agent organizes a collaboration by
matching subtasks to registered capabilities and representing their
dependencies as a directed acyclic graph (DAG). Each node invokes an
agent, and edges identify the dependencies between their outputs and
subsequent work. The runtime checks the plan before dispatch, schedules
ready nodes, and records
artifacts for handoff and reuse. The host handles execution exceptions
and integrates the resulting deliverables. In the illustrative FPS
workflow in \Cref{fig:raven-ecosystem}(b), a research brief informs
parallel gameplay development and visual design. Their outputs then
converge for integration and testing before deployment and operation.
The graph specifies the inputs each specialist requires and how its
output contributes to the requested deliverable.

Beyond coordinating individual workflows, \raven combines harness
adaptation with the reuse of experience across tasks. Prior methods
adapt agents by retaining insights from past executions
\citep{zhao2024expel} or distilling interactions into reusable skills
\citep{zheng2025skillweaver}. In \raven, modular harnesses expose execution
policies that an external evolver can modify within defined boundaries.
Building on our prior work, HarnessBank \citep{luo2026harnessbank}, the
evolution process diagnoses failures, proposes candidate harnesses, and evaluates
their behavior with a frozen task model. Persistent memory complements
these policy changes. A host archive retains user context, while the
optional EverOS backend \citep{hu2026evermemos,evermind2026everos}
provides semantic access to past experience. For procedural reuse,
Skill Forge retrieves task-relevant procedures from local skills,
memory-derived skills, and SkillHub, whose corpus and retrieval design
build on SkillCorpus
\citep{wang2026skillcorpus}.

Alongside this system design, our theoretical analysis characterizes when
composition can extend the capabilities of the available agents. For a
specified agent pool and task family, we formalize capability as reliable
task coverage under a common resource budget. We establish sufficient
conditions under which
complementary local capabilities, compatible handoffs, and bounded
planning and execution errors permit the composed system to solve tasks
that the available individual agents cannot reliably solve alone under
the same budget. The analysis accounts for the cost of planning and
coordination alongside worker execution.

To assess \raven empirically, we evaluate planning quality, specialist
execution, harness adaptation, and skill reuse separately. We introduce
the Multi-Agent Orchestration Benchmark (MAOB), which measures specialist
selection and dependency prediction by comparing proposed DAGs with
reference graphs before worker
execution. \raven ranks first among the compared systems on all four
graph metrics under both tested backbones, with Exact Match gains of
$10.4$ and $10.5$ percentage points over the strongest baseline for the
two backbones, respectively. Domain-specific evaluations characterize the
four native specialists, while the published HarnessBank experiments assess
the evolution method with a frozen backbone. The published SkillCorpus
experiments show that a curated skill library improves Raven on three
benchmarks, with larger gains than OpenClaw on two of them.

In summary, our main contributions are:
\begin{itemize}
    \item \textbf{An Open Ecosystem For Harness Composition.}
    We present an architecture that coordinates native and third-party
    agents through a shared interface and explicit execution dependencies,
    with modular harness evolution, persistent memory, and skill reuse
    (\crefrange{sec:collaboration}{sec:skillforge}).
    \item \textbf{A Theory Of Composable Agentic Intelligence.}
    We formalize task-relative capabilities and sufficient conditions for
    reliable composition and expanded task coverage under a shared
    resource budget
    (\cref{sec:theory}).
    \item \textbf{A Benchmark And Evaluation Across Capability Levels.}
    We introduce MAOB to evaluate specialist selection and dependency
    prediction, and organize empirical evidence across orchestration,
    the four native specialists, harness adaptation, and skill reuse
    (\cref{sec:evaluation}).
\end{itemize}

\section{Theory}\label{sec:theory}

\raven composes executable model--harness pairs through a host that controls
their assignment, information exchange, and execution. To analyze this
composition, we formalize
\emph{composable agentic intelligence} as reliable task coverage under a common
resource budget. The analysis begins with typed execution contracts, establishes
deterministic composition soundness, and then derives probabilistic
reliability bounds and conditions for capability beyond individual agents.
A final result addresses coverage of a task family specified independently
of observed system successes. These are sufficient conditions for reliable
composition, and their practical applicability depends on whether the
implemented agents and host satisfy them.
\Cref{fig:theory-composition} illustrates the composition.
\Cref{tab:theory-notation} collects the principal notation, while
\cref{app:method-notation} states the notation conventions of the method chapters.

\begin{figure}[htbp]
    \centering
    \includegraphics[width=\linewidth]{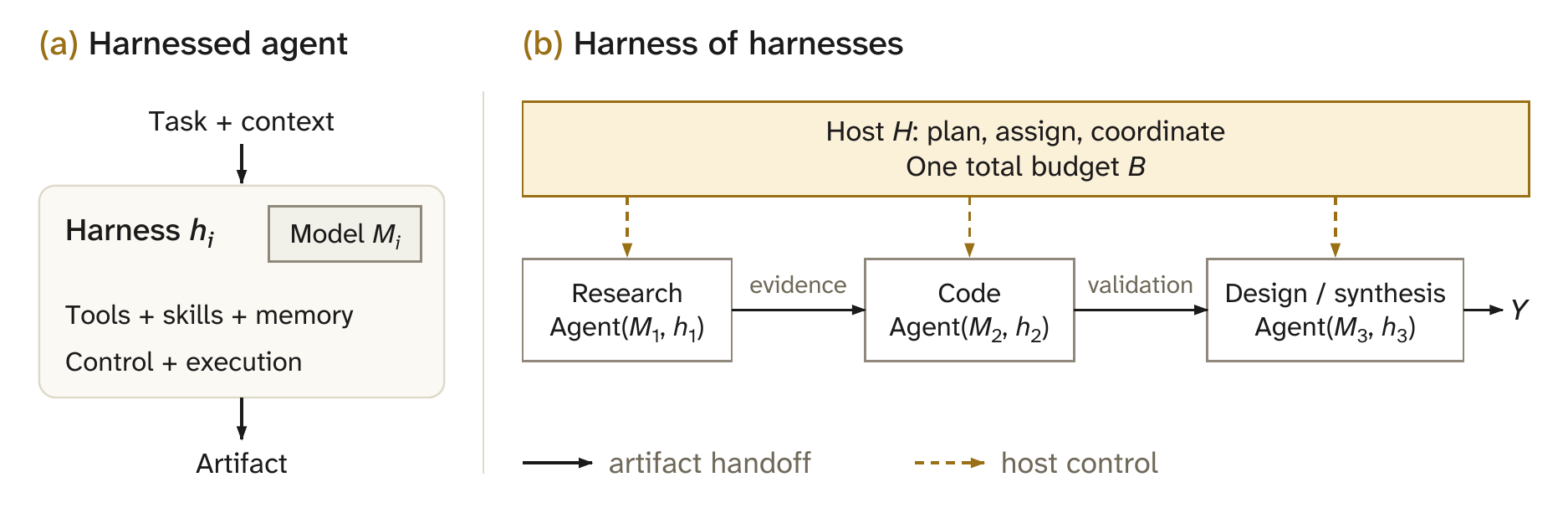}
    \caption{Harness composition. (a) A model and its harness form a callable
    execution unit. (b) A host assigns and coordinates these units. Solid
    arrows carry artifacts and dashed arrows denote control. The
    research--code--design chain is an illustrative plan with explicit final
    synthesis by a specialist. All operations share one budget. Model labels
    need not identify different foundation models.}
    \label{fig:theory-composition}
\end{figure}

\subsection{Harnessed Agents and Task-Relative Capabilities}
\label{sec:theory-capabilities}

A harness turns a model into an interactive execution policy. Let $M_i$ be a
model and $h_i$ its tool interfaces, context and memory mechanisms, skills,
control rules, and execution interfaces. The induced agent is
\begin{equation}
    a_i=\operatorname{Agent}(M_i,h_i):
       \mathcal H_i\longrightarrow\Delta(\mathcal U_i),
    \qquad \mathcal A=\{a_0,a_1,\ldots\},\quad 2\le|\mathcal A|<\infty.
    \label{eq:harnessed-agent}
\end{equation}
Here $\mathcal H_i$ is the space of the agent's permitted observation
histories, including its request and local memory. $\mathcal U_i$ contains
tool actions, artifact-return actions, and an abort action. $\Delta$
denotes probability distributions. The construction $\operatorname{Agent}$
therefore defines a stochastic policy without restricting the model's internal
parameterization. Although the analysis describes the full environment state,
the policy has access only to its permitted observation history.

The executor pool includes $a_0$, the host's callable policy with
cross-roster delegation disabled. The controller $H$ is a stochastic policy
on its permitted history, with additional actions for plan submission and
invoking pool executors. The notation $S_H=H[\mathcal A]$ denotes the
interaction of these policies through the execution protocol below.
A synthesis node can invoke $a_0$ or a specialist. Each
$a_i$ is also a standalone comparator, run on the original request with its
native interfaces and without calls to other roster members. Models,
harness versions, and initial memory-state laws are fixed for a comparison.
Several agents may share a model, and one agent may supply several worker
instances.

A task is a specification of inputs, interaction, and correctness:
$t=(r_t,\mu_t,\mathsf E_t,\mathsf V_t)\in\mathcal T$. The request $r_t$ is observable.
$\omega\sim\mu_t$ takes values in the reference space $\mathcal W$ and
represents the initial input and environment conditions.
$\mathsf E_t$ gives the distribution of the next observation and environment
state conditional on the joint interaction history and a tool action.
$\mathsf V_t:\mathcal W\times\mathcal Y_{\rm rec}\to\{0,1\}$ judges a record
$Y\in\mathcal Y_{\rm rec}$. The record includes
delivered artifacts and relevant environment changes. Agent initial-state
laws belong to the fixed system configuration. A fixed benchmark instance
corresponds to a degenerate input law. For an aborted or nonterminating run, set
$Y=\bot$ and $\mathsf V_t(\omega,\bot)=0$, with the distinguished failure
record $\bot$ included in $\mathcal Y_{\rm rec}$.

We use countable spaces of finitely encoded observations, plans, artifacts,
and runtime states, including the represented reference conditions.
Whenever an outcome or padded runtime record can equal $\bot$, its space
includes that sentinel. State and ledger predicates are extended to be
false on padded arguments. We write $\mathbb I[\mathcal E]$ for the indicator of
an event or predicate $\mathcal E$. For an event $\mathcal E$,
$\mathcal E^c$ denotes its complement.
Together, the task, system policies, scheduler, and
budget convention induce a probability law $\mathbb P_{t,S,B}$ on the
trace space $\Omega_{\rm tr}$, for $S=S_H$ or a standalone $a_i$.
The measurable trace function $Y_S$ extracts the delivered record, or
$\bot$ on failure. The nonnegative, possibly infinite total cost $C_S$
charges every operation exactly once using a fixed additive measure,
such as total inference and tool cost. Capability is
defined by
\begin{align}
    p_S(t;B)
      &=\mathbb P_{t,S,B}
        \bigl(\mathsf V_t(\omega,Y_S)=1,\ C_S\le B\bigr),
        &&B\ge0, \label{eq:budgeted-success}\\
    \mathcal C_S(B,\delta)
      &=\{t\in\mathcal T:p_S(t;B)\ge1-\delta\},
        &&0\le\delta<1. \label{eq:capability-set}
\end{align}
For $S_H$, cost includes planning, workers, handoffs, runtime verification,
and retries. Nontermination is unsuccessful even if its accrued monetary
cost is finite. Elapsed time can be a separate constraint, but it is not added
to tokens or monetary cost. Objective correctness $\mathsf V_t$ remains distinct
from an agent or judge declaring completion.

\subsection{Resource Limits and the Scope of Composition}
\label{sec:theory-limits}

Resource limits make capability task-dependent. No-free-lunch results equate
search performance
measured from sampled objective values for non-revisiting algorithms at a
fixed evaluation budget, under uniform averaging over functions between
finite input and value sets \citep{wolpert1997nfl}. Structured task
distributions can instead favor appropriately matched prior knowledge, and
enough queries permit exhaustive search. The finite-query argument in
\cref{app:theory-search} shows that $n_{\rm qry}<N_{\rm pos}$ black-box queries cannot guarantee
finding every target among $N_{\rm pos}$ positions. It applies equally to an individual
agent and a multi-agent system with the same total information budget.

Empirical limitations motivate examining the particular agents available
for composition. AgentBench identifies deficiencies in long-term reasoning,
decision-making, and instruction following among its evaluated systems
\citep{liu2024agentbench}. Such results support capability profiling on the
intended task distributions, with conclusions limited to the evaluated systems.
Routing selects an agent for a complete request, whereas our composition
analysis concerns the distinct local contracts needed within
one request. \Cref{def:complementarity} formalizes local complementarity after
introducing contracts and their realization. A comparison with complete
single-agent executions remains a separate step. Individual limitations
alone do not imply complementary strengths.

\subsection{Host-Mediated Harness Composition}
\label{sec:theory-composition}

A plan specifies an executable graph together with its semantic and resource
obligations:
\begin{equation}
    \pi=(G,\sigma,\Phi,\Gamma,\lambda,\mathbf b),\qquad
    G=(V,E),\qquad \sigma:V\longrightarrow\mathcal A.
    \label{eq:composition-plan}
\end{equation}
The finite DAG $G$ has a nonempty node set $V$ of worker calls and a set of
handoff edges $E$. The assignment $\sigma$ selects an executor, including
$a_0$, for each call. The annotation $\Gamma$ contains contracts, handoff
providers, invocation bindings to states and instances, and the invariants
defined below. An analytical rule fixed before execution may supply the
contracts and invariants without requiring runtime proof certificates.
This rule depends only on the task, planning record, and submitted plan.
The schedule is a bijection
$\lambda:\{1,\ldots,K_{\rm op}\}\to V\sqcup E$, with $K_{\rm op}=|V|+|E|$ and
$\lambda_j=\lambda(j)$. It places $u$ before $e=(u,v)$ and $e$ before $v$.
The disjoint union $\sqcup$ distinguishes node identifiers from edge identifiers.
The last operation is the plan's designated output node.

Artifacts are stored in an append-only ledger $\ell$, separate from the
mutable state $s\in\mathcal S$. The initial ledger retains the task packet
and prior execution record, including planning actions, so the final record
can account for the complete run. Only the authorized packet and declared
messages become worker inputs. The state includes relevant environment and
agent-local state and has a read-only reference projection
$\operatorname{ref}:\mathcal S\to\mathcal W$, equal to the sampled
$\omega$ along an actual trace. Let $\mathcal X_v,\mathcal Y_v$ be a node's input and
output spaces, $\mathcal M_e$ an edge's message space, and
$\operatorname{In}(v)$ its incoming edges. For an authorized task/context
packet $d$ in the encoded-packet space $\mathcal D_{\rm pkt}$, the prescribed input is
\begin{equation}
    x_v=\Phi_v\bigl(d,(m_e)_{e\in\operatorname{In}(v)}\bigr),
    \qquad
    \Phi_v:\mathcal D_{\rm pkt}\times
        \prod_{e\in\operatorname{In}(v)}\mathcal M_e
        \longrightarrow\mathcal X_v .
    \label{eq:input-assembly}
\end{equation}
At a source, $\Phi_v$ uses only $d$. Edge transfers first
produce messages under their contracts, and $\Phi_v$ then assembles those
messages into the receiving node's input. Edge transformations and transport
are charged to the edge, while receiver assembly is charged to the node.
Both stages may use typed projections or serialization. Other computation
requires an explicit node, so the accounting covers each operation once.

To specify correctness for these operations, we define node and transfer
contracts as typed predicates:
\begin{align}
    P_v&:\mathcal X_v\times\mathcal S\to\{0,1\},\qquad
    Q_v:\mathcal X_v\times\mathcal S\times
                 \mathcal Y_v\times\mathcal S\to\{0,1\},
                 \label{eq:node-contracts}\\
    J_e&:\mathcal Y_u\times\mathcal S\times
                 \mathcal M_e\times\mathcal S\to\{0,1\},
                 \qquad e=(u,v). \label{eq:handoff-contract}
\end{align}
$P_v(x,s)$ specifies the precondition, $Q_v(x,s,y,s')$ specifies the output
and state transition, and $J_e(y,s,m,s')$ specifies faithful transfer and its state
effects. Predicates may inspect the analytical reference component of $s$,
although executors cannot. In particular, a judge's acceptance is not the
definition of $Q_v$.

For operation $\lambda_j$, define its entry predicate
$A_j(s,\ell)$ and transition predicate $D_j(s,\ell,s',\ell')$. A node entry
requires its incoming messages and $P_v(x_v,s)$. Its transition requires
the recorded actual input to equal the prescribed input, a fresh output record,
and $Q_v(x_v,s,y_v,s')$. A handoff entry requires its source artifact. Its
transition requires a fresh message record and $J_e(y_u,s,m_e,s')$.
Both transitions preserve all earlier ledger records. Missing or ill-typed
operation records make the corresponding transition predicate false.
The ledger stores actual inputs as well as outputs.
\Cref{app:theory-protocol} gives the exact predicates and their behavior
under ledger extension. A plan is well-typed when its input, output, and
message records have the declared types and all its identifier references resolve.

Beyond type correctness, composition requires compatibility that preserves
the assertions needed by later operations, including assertions about mutable
state. Fix a set
$\mathcal I\subseteq\mathcal S\times\mathcal L$ of admitted initial
state--ledger pairs, where $\mathcal L$ is the ledger space.
The annotations are instantiated for the task and authorized packet.
They provide predicates
$I_j:\mathcal S\times\mathcal L\to\{0,1\}$ for $j=0,\ldots,K_{\rm op}$ and a
specified projection $\operatorname{rec}_\pi(\ell,s)$ into the final execution record.
The projection serializes actual stored artifacts, actions, and state,
including the planning record, without computing an additional artifact.

\begin{definition}[Compatible plan]\label{def:compatible-plan}
A well-typed plan is compatible with task $t$ and initial set $\mathcal I$
if $I_0$ holds throughout $\mathcal I$ and the following implications hold
for every well-typed state and ledger transition:
\begin{align}
    I_{j-1}(s,\ell)&\ \Longrightarrow\ A_j(s,\ell),\label{eq:contract-enable}\\
    I_{j-1}(s,\ell)\land D_j(s,\ell,s',\ell')
       &\ \Longrightarrow\ I_j(s',\ell'),\label{eq:contract-preserve}\\
    I_{K_{\rm op}}(s,\ell)&\ \Longrightarrow\
       \mathsf V_t\bigl(\operatorname{ref}(s),\operatorname{rec}_\pi(\ell,s)\bigr)=1 .
       \label{eq:contract-terminal}
\end{align}
The first two conditions apply to $j=1,\ldots,K_{\rm op}$.
\end{definition}

The enabling implication checks a join's inputs jointly, while preservation
maintains the assertions needed later. This induction uses Hoare-style
assertions \citep{hoare1969axiomatic}. Local assumptions and guarantees
support componentwise reasoning \citep{kwiatkowska2010assume}.
We prove the execution result for the serial order $\lambda$. It also
applies to parallel executions with a certified linearization preserving
observations, records, state effects, acceptance, and charged cost
(\cref{app:theory-protocol}). A DAG or disjoint output filenames alone do
not establish this property.

The allocation vector
$\mathbf b=(b_H,(b_v)_{v\in V},(b_e)_{e\in E})$ has nonnegative entries for
planning, node calls, and transfers. Resource feasibility means
\begin{equation}
    b_H+\sum_{v\in V}b_v+\sum_{e\in E}b_e\le B.
    \label{eq:composition-budget}
\end{equation}
The execution protocol attempts $\lambda_1,\ldots,\lambda_{K_{\rm op}}$ in order. Write
$(\Sigma_j,L_j)$ for the random state and ledger after operation $j$, and $C_j$
for its cost. $(\Sigma_0,L_0)$ is the post-planning entry.
An operation succeeds if it finishes, is accepted, satisfies $D_j$, and
respects $b_{\lambda_j}$. Records after abandonment are padded with $\bot$.
Node costs include assembly, verification, continuations, and assigned
control work. The last operation also includes serialization and delivery.
Its success means that $\operatorname{rec}_\pi(L_{K_{\rm op}},\Sigma_{K_{\rm op}})$ has been delivered. After all
$K_{\rm op}$ successes the system terminates with
\[
    Y_{S_H}=\operatorname{rec}_\pi(L_{K_{\rm op}},\Sigma_{K_{\rm op}}),\qquad
    C_{S_H}=C_H+\sum_{j=1}^{K_{\rm op}} C_j,
\]
where $C_H$ is the incurred planning cost. Final delivery is charged to the
last operation, and a delivery failure counts as failure of that operation.

\begin{lemma}[Deterministic composition soundness]
\label{lem:composition-soundness}
Starting from $\mathcal I$, success of the first $j$ operations of a
compatible plan establishes $I_j$ and, for $j<K_{\rm op}$, enables $A_{j+1}$.
If every operation succeeds, planning costs at most $b_H$, and
\cref{eq:composition-budget} holds, the protocol delivers a correct record
within $B$ and terminates.
\end{lemma}
\begin{proof}
$I_0$ holds initially. Given $I_{j-1}$, \cref{eq:contract-enable} establishes
the next operation's precondition. Its successful transition satisfies
$D_j$, so \cref{eq:contract-preserve} establishes $I_j$ without assuming any
later success. Equation~\eqref{eq:contract-enable} then gives $A_{j+1}$
when $j<K_{\rm op}$. At $j=K_{\rm op}$, \cref{eq:contract-terminal} proves correctness.
The terminal protocol delivers that record and stops. Summing the
successful operations' costs and planning cost proves the budget claim.
\end{proof}

In the research--code--design example, the ledger preserves the scoped
claim, evidence, implementation, and test outcomes. An invariant before
design records which claims were validated. The design contract
must preserve that distinction in the final presentation. Passing a file path alone does not establish the enabling or terminal
implications.

\subsection{Reliability of Composed Execution}
\label{sec:theory-reliability}

The soundness result assumes that each operation succeeds. To account for
execution failures, we define local capability through contract realization
in a specified invocation context. A provider $\kappa$ is a worker call or handoff implementation,
including its assembly, control, and acceptance procedures. It has a
conditional outcome kernel $\mathsf K_\kappa(o\mid c_{\rm cfg})$ on encoded configurations
$c_{\rm cfg}\in\mathcal Z_\kappa$ and outcomes $o\in\mathcal O_\kappa$.
Configurations contain the entry state, ledger, input, and relevant prior
or latent information. Outcomes contain the resulting state and ledger,
acceptance flag, and cost, or $\bot$ for noncompletion. The kernel may
depend on full history through $c_{\rm cfg}$, without assuming memoryless agents.
An entry law $\nu$ is a distribution over $\mathcal Z_\kappa$.
Define the local experiment by
\begin{equation}
 \Pr_{\nu,\kappa}(\operatorname{ok}_{D,b})
   :=\sum_{c_{\rm cfg},o}\nu(c_{\rm cfg})\mathsf K_\kappa(o\mid c_{\rm cfg})
                  \operatorname{ok}_{D,b}(c_{\rm cfg},o).
 \label{eq:local-experiment}
\end{equation}
Here $\operatorname{ok}_{D,b}(c_{\rm cfg},o)$ equals one exactly for accepted completion
in finite time satisfying transition predicate $D$ at cost at most $b$,
and equals zero otherwise. Analytical information in $c_{\rm cfg}$ does not expand an
executor's permitted observations.

\begin{definition}[Contract realization]\label{def:contract-realization}
For a nonempty family $\mathfrak L$ of entry laws supported on entry
predicate $A$, write $\kappa\models_{\mathfrak L}(A,D;b,\epsilon)$,
with $b\ge0$ and $\epsilon\in[0,1]$, if
\begin{equation}
    \inf_{\nu\in\mathfrak L}
       \Pr_{\nu,\kappa}(\operatorname{ok}_{D,b})
       \ge1-\epsilon .
    \label{eq:local-realization}
\end{equation}
\end{definition}

Applying this definition to a workflow requires the local experiment to
match the actual conditional execution, as established below. An isolated
benchmark average does not establish that relationship.

To specify these contexts, fix $t,B,H,\mathcal A$ and abbreviate
$\mathbb P=\mathbb P_{t,S_H,B}$. Let $Z$ be the finite planning transcript,
including observations and incurred cost, and let $\Pi$ be the selected
annotated plan. If planning is abandoned without an execution plan, set $\Pi=\bot$.
Nonterminating planning also gives $Z=\bot$. A direct fallback is instead
represented as a one-node plan for $a_0$, with earlier planning costs retained.
Both variables have countable ranges. For a positive-probability pair
$(Z,\Pi)=(\zeta,\pi)$, $d_\zeta$ is determined by the transcript and
$\mathcal I_{\zeta,\pi}$ is the support of the conditional entry
$(\Sigma_0,L_0)$. A pair is \emph{valid} if its planning cost is at most $b_H$,
the plan satisfies \cref{eq:composition-budget}, its protocol has the
specified serial or certified parallel semantics, and it is compatible
with $\mathcal I_{\zeta,\pi}$. The compatibility premise must hold for every
entry in this conditional support.
The valid-planning event
\begin{equation}
    \mathsf G_t
      =\{(Z,\Pi)\text{ is a valid planning pair}\}
    \label{eq:good-planning-event}
\end{equation}
is $(Z,\Pi)$-measurable, that is, determined by the planning variables. Validity is an analytical property of a
planning pair and need not be decidable by the runtime.

For a fixed valid pair, let $\Xi_j,O_j$ be operation $\lambda_j$'s actual
configuration and outcome. Its provider, law family, allocation, and error
bound are $(\kappa_j,\mathfrak L_j,b_{\lambda_j},\epsilon_j)$. Node and edge
subscripts refer to the same objects through $j=\lambda^{-1}(v)$ or
$j=\lambda^{-1}(e)$. Let $R_{j-1}$ be the host-visible execution record:
messages, returned artifacts, statuses, and observed costs, excluding
hidden reference truth. For parallel execution it is the corresponding
prefix of the certified serialization, not necessarily a wall-clock prefix.
Define the success events
\[
    \mathsf{Ok}_j=\{\operatorname{ok}_{D_j,b_{\lambda_j}}(\Xi_j,O_j)=1\},\qquad
    F_0=\Omega_{\rm tr},\qquad F_j=\bigcap_{k=1}^{j}\mathsf{Ok}_k .
\]
A skipped or nonterminating operation does not satisfy $\mathsf{Ok}_j$. The reliability
premise for each positive-probability successful prefix is
\begin{equation}
    \mathbb P(\mathsf{Ok}_j^c\mid F_{j-1},Z=\zeta,\Pi=\pi)\le\epsilon_j .
    \label{eq:conditional-local-error}
\end{equation}
$\epsilon_j\in[0,1]$ may depend on $t,B,\zeta,\pi$. Those arguments are
suppressed in the local bounds.

We call the selected plan \emph{locally covered} when each provider realizes
its contract and, conditional on each
$(R_{j-1}=\varrho,F_{j-1},Z=\zeta,\Pi=\pi)$ with record value $\varrho$, its entry law belongs
to $\mathfrak L_j$ and its outcome kernel is $\mathsf K_{\kappa_j}$.
The kernel-consistency condition and averaging proof are stated in
\cref{lem:local-to-workflow}. That lemma derives
\cref{eq:conditional-local-error} from local coverage.
The condition retains $F_{j-1}$ because visible records do not necessarily
reveal objective correctness. It permits correlated errors through the
conditional entry laws and kernels.

All conditional quantities below refer to the same actual system law:
\begin{equation}
    \mathsf S_t=\{\mathsf V_t(\omega,Y_{S_H})=1,\ C_{S_H}\le B\},
    \qquad
    q_t(\zeta,\pi)
       =\mathbb P(\mathsf S_t\mid Z=\zeta,\Pi=\pi).
    \label{eq:selected-plan-success}
\end{equation}
Thus $q_t$ accounts for incurred planning cost and the input distribution
selected by the host, rather than a fresh standalone execution of $\pi$.

\begin{theorem}[Reliability of harness composition]
\label{thm:composition-reliability}
For a valid planning pair $(\zeta,\pi)$ satisfying
\cref{eq:conditional-local-error}, define
\begin{equation}
    \epsilon(\zeta,\pi)
      =\sum_{j=1}^{K_{\rm op}}\epsilon_j
      =\sum_{v\in V}\epsilon_v+\sum_{e\in E}\epsilon_e.
    \label{eq:plan-error-budget}
\end{equation}
Then the conditional success probability satisfies
\begin{equation}
    q_t(\zeta,\pi)\ge[1-\epsilon(\zeta,\pi)]_+,
    \qquad [x]_+=\max\{x,0\}.
    \label{eq:composition-reliability}
\end{equation}
\end{theorem}
\begin{proof}
By \cref{lem:composition-soundness}, $F_{K_{\rm op}}$ implies $\mathsf S_t$ under the
conditional law. The disjoint events $F_{j-1}\cap \mathsf{Ok}_j^c$ partition $F_{K_{\rm op}}^c$.
For each positive-probability prefix,
\begin{align*}
 &\mathbb P(F_{j-1}\cap \mathsf{Ok}_j^c\mid Z=\zeta,\Pi=\pi)\\
 &\quad=\mathbb P(F_{j-1}\mid Z=\zeta,\Pi=\pi)
   \mathbb P(\mathsf{Ok}_j^c\mid F_{j-1},Z=\zeta,\Pi=\pi)
   \le\epsilon_j .
\end{align*}
Zero-probability prefixes contribute zero. Summing gives
$\mathbb P(F_{K_{\rm op}}^c\mid Z=\zeta,\Pi=\pi)\le\epsilon(\zeta,\pi)$.
Taking the complement and using nonnegativity proves the bound. No
independence assumption is required.
\end{proof}

The preceding bound conditions on a valid planning pair. We next account
for the host's probability of selecting such a pair.

\begin{proposition}[Host-level reliability]
\label{prop:host-reliability}
Suppose $\mathbb P(\mathsf G_t)\ge1-\eta_H(t;B)$, where
$\eta_H(t;B)\in[0,1]$. On every positive-probability valid planning pair,
assume the local premises of \cref{thm:composition-reliability} and the
uniform envelope
\begin{equation}
    \min\{1,\epsilon(\zeta,\pi)\}\le\bar\epsilon(t;B)\le1.
    \label{eq:host-assumptions}
\end{equation}
Clipping limits the accumulated error bound to the range of a probability.
Then
\begin{equation}
    p_{S_H}(t;B)
      \ge(1-\eta_H(t;B))(1-\bar\epsilon(t;B)).
    \label{eq:host-reliability}
\end{equation}
\end{proposition}
\begin{proof}
For positive-probability planning pairs use $q_t$ from
\cref{eq:selected-plan-success}. A pair with $\Pi=\bot$ has zero success.
Completed but invalid plans retain their actual conditional success
probability. Since $\mathsf G_t$ is $(Z,\Pi)$-measurable, conditional expectation gives
\begin{align*}
    p_{S_H}(t;B)
       &\ge\mathbb P(\mathsf S_t\cap\mathsf G_t)
        =\mathbb E[\mathbb I[\mathsf G_t]q_t(Z,\Pi)]\\
       &\ge(1-\bar\epsilon(t;B))\mathbb P(\mathsf G_t)
        \ge(1-\bar\epsilon(t;B))(1-\eta_H(t;B)).
\end{align*}
The second inequality uses \cref{thm:composition-reliability} uniformly over
valid pairs. It makes no independence claim about planning and execution.
\end{proof}

The bound requires a compatible plan, a host that can find such a plan
within its budget, and providers that realize their contracts in the
selected contexts. Additional nodes may reduce subtask difficulty while
increasing handoff cost and error. Internal continuations are included in a node's
realization. Recovery on failed paths preserves this lower bound if the
first plan's all-success path still delivers and terminates within budget.
Quantifying additional recovery gains requires accounting for the later
plans and reached states.

\subsection{Capabilities Beyond Individual Agents}
\label{sec:theory-expansion}

To examine capability gains from reliable composition, we first define
local complementarity before selecting executors.
A \emph{decomposition} fixes $(G,\Phi,\Gamma,\lambda,\mathbf b)$, omitting
$\sigma$, together with the node law families and error allocations.
Let $\kappa_{v,i}$ be the bounded invocation of $a_i$
at node $v$, including its prescribed input bindings and acceptance procedure.
Providers compared at this node use a common encoded configuration space
$\mathcal Z_v$ for the full state, ledger, input, and invocation
bindings. Each provider retains its own permitted-observation projection
and outcome kernel. Thus $\mathfrak L_v$ is a family of laws on the same
domain for every candidate provider. Define the eligible realizers by
\begin{equation}
    \mathcal R_v=
       \{a_i\in\mathcal A:
         \kappa_{v,i}\models_{\mathfrak L_v}
             (A_v,D_v;b_v,\epsilon_v)\}.
    \label{eq:eligible-realizers}
\end{equation}
Here $(A_v,D_v)$ are the operation predicates for node $v$, as indexed
earlier by its position $j$ in $\lambda$.

\begin{definition}[Decomposition-level complementarity]
\label{def:complementarity}
The decomposition has complementary local capabilities if
\begin{equation}
    \mathcal R_v\ne\varnothing\quad\text{for every }v\in V,
    \qquad
    \bigcap_{v\in V}\mathcal R_v=\varnothing.
    \label{eq:local-complementarity}
\end{equation}
\end{definition}

The first condition ensures that each node has an eligible executor. The
second excludes assigning the same eligible executor to every node under
the specified allocations.
\Cref{prop:eligible-assignment} connects eligible assignments to the
reliability theorem, provided transfer contracts, compatibility, and actual
invocation-law consistency also hold. However, a different single-agent strategy
may still solve the complete task. Hence an independent upper bound on
individual end-to-end performance is needed to establish nonempty
\emph{new-task coverage},
\[
    \mathcal C^{\rm new}(B,\delta)
      :=\mathcal C_{S_H}(B,\delta)
          \setminus\bigcup_{a\in\mathcal A}\mathcal C_a(B,\delta).
\]

\begin{samepage}
\begin{corollary}[Capability beyond individual agents]
\label{cor:capability-expansion}
For fixed $B,\delta$, suppose a task $t$ satisfies the hypotheses of
\cref{prop:host-reliability} and
\begin{equation}
    (1-\eta_H(t;B))(1-\bar\epsilon(t;B))
       \ge1-\delta>
       \max_{a\in\mathcal A}p_a(t;B).
    \label{eq:strict-expansion-condition}
\end{equation}
Then
\begin{equation}
    t\in\mathcal C^{\rm new}(B,\delta).
    \label{eq:strict-expansion}
\end{equation}
\end{corollary}
\begin{proof}
\Cref{prop:host-reliability} and the first inequality place $t$ in the system's
capability set. The strict second inequality excludes it from every
individual capability set by \cref{eq:capability-set}.
\end{proof}
\end{samepage}

\paragraph{An Explicit Complementary Construction.}
Let $U,W$ be independent fair bits. In the task
$t=(r_t,\mu_t,\mathsf E_t,\mathsf V_t)$, $\mu_t$ is their uniform law,
$r_t$ asks for their parity without revealing them, $\mathsf E_t$ answers
permitted bit queries deterministically, and $\mathsf V_t$ checks that the returned
bit equals $U\oplus W$. Take $\mathcal A=\{a_0,a_1,a_2\}$.
The native read interface of $a_1$ exposes only $U$, that of $a_2$ only $W$,
and $a_0$ has neither interface. The original request contains neither bit.
For individual executions, let all initial auxiliary information and
decision randomness be included in $\xi_i$, with
$(\xi_0,\xi_1,\xi_2)$ independent of $(U,W)$. The available transcripts are
\begin{equation}
    T_0=f_0(\xi_0),\qquad T_1=f_1(U,\xi_1),\qquad
    T_2=f_2(W,\xi_2).
    \label{eq:parity-information}
\end{equation}
Each returned answer $\widehat Y_i$ is a function of $T_i$. Any additional
decision randomness is already in $\xi_i$. Conditional on $(U,\xi_1)$,
$W$ remains fair, hence
\begin{equation}
    \mathbb P_{t,a_1,B}(\widehat Y_1=U\oplus W\mid U,\xi_1)\le\tfrac12.
    \label{eq:parity-single-bound}
\end{equation}
Equality holds for a bit-valued answer. An invalid answer or abort cannot
improve it. Averaging, and applying the symmetric argument to $a_2$ and the
no-input argument to $a_0$, bounds every $p_{a_i}(t;B)$ by $1/2$.

Specify the local policies constructively: on a read-and-report call,
$a_1$ returns $U$ and $a_2$ returns $W$ exactly. On a call supplied with two
bits, $a_0$ returns their parity exactly. These input-dependent policies are
fixed before evaluation. A three-node plan with two exact-copy handoffs
therefore satisfies the typed contracts with zero error. Its fixed host
schedule has zero planning error. Index the three nodes $1,2,0$ by their
executors. Choose finite allocations that bound their actual calls,
handoffs, final delivery, and planning cost. If
$b_H+b_1+b_2+b_0+b_{(1,0)}+b_{(2,0)}\le B$, then
$p_{S_H}(t;B)=1$, and \cref{cor:capability-expansion} applies for
$\delta<1/2$. Here $b_0$ is the synthesis-node allocation. Under zero-error read
contracts, the same missing-bit argument makes the read-$U$ and read-$W$
eligible sets disjoint, while the respective copying policies realize them.
Together with $a_0$ at synthesis, this also witnesses
\cref{def:complementarity}. The same initial input law and budget are used
for every comparator. The gain comes from
making complementary interface results available through costed messages.
A single agent equipped with both read interfaces could also solve the task.
This construction does not assert superiority over that different harness.

Although this result establishes \emph{new-task coverage}, it does not
establish strict set inclusion or an improvement in average performance. Retaining the whole union of
individual capability sets additionally requires feasible single-agent
execution paths and a selector that preserves their reliability. Empirical
tests must therefore include strong individual-agent baselines with
comparable models, tools, initial information, and total cost. The
research--code--design workflow is a candidate for such a comparison, not
evidence of a gain by itself.

\subsection{Reliable Task-Family Coverage and Implications for Raven}
\label{sec:theory-completeness}

Extending the task-level analysis to a task family requires specifying
that family independently of the system's successes.
\Cref{app:theory-family} defines a typed grammar of primitive, sequential,
and fork--join tasks, gives their relational semantics, and constructs
compatible finite DAGs by structural induction. Forking, joining, and
transferring artifacts are explicit costed operations. The construction
proves plan existence. Its compiler-based specialization
(\cref{cor:compiler-coverage}) translates correct input binding and
compilation into valid planning. The following general result also permits
other valid plans without requiring the host to reproduce the compiler.

\begin{samepage}
\begin{corollary}[Reliable task-family coverage]
\label{cor:relative-completeness}
Fix $\mathcal T_\star\subseteq\mathcal T$, $B\ge0$, and $0\le\delta<1$.
Suppose the hypotheses of \cref{prop:host-reliability} hold for every
$t\in\mathcal T_\star$, and, uniformly on this family,
\begin{equation}
    \eta_H(t;B)\le\eta_\star,\qquad
    \bar\epsilon(t;B)\le\epsilon_\star,\qquad
    \eta_\star+(1-\eta_\star)\epsilon_\star\le\delta,
    \qquad \eta_\star,\epsilon_\star\in[0,1].
    \label{eq:family-risk-budget}
\end{equation}
Then
\begin{equation}
    \mathcal T_\star\subseteq\mathcal C_{S_H}(B,\delta).
    \label{eq:relative-completeness}
\end{equation}
\end{corollary}
\begin{proof}
For each task, \cref{prop:host-reliability} yields
$p_{S_H}(t;B)\ge(1-\eta_\star)(1-\epsilon_\star)
=1-\eta_\star-(1-\eta_\star)\epsilon_\star\ge1-\delta$.
The capability definition gives the inclusion.
\end{proof}
\end{samepage}

For a grammar with bounded expanded node and edge counts, the appendix
derives uniform resource and error envelopes. These supply explicit
sufficient conditions for the corollary without assuming that arbitrary
composition preserves a fixed budget or reliability level. Allowing
a finite budget $B_t$ chosen separately for each task can establish
$p_{S_H}(t;B_t)\ge1-\delta$ for every $t\in\mathcal T_\star$ without
providing one finite budget valid for the entire family.
The analysis also makes no computability distinction between an orchestrated
system and a single program that can simulate it with the same interfaces
and resources.

The theory thus gives \raven's components distinct analytical roles.
Agent Orchestration (\cref{sec:orchestration}) constructs and executes plans.
The built-in and connected agents (\cref{sec:agents}) provide candidate
contract realizers. EverOS Memory (\cref{sec:memory}) and Skill Forge
(\cref{sec:skillforge}) help construct context and configure skills.
Harness Self-Evolution (\cref{sec:evolver}) changes the policies being
compared. Their effects on capability, cost, and conditional
reliability require evaluation.

Prior routing results support considering collaboration modes, roles, and
model assignments jointly \citep{yue2025masrouter}. Controlled scaling
experiments also find that coordination benefits depend on task structure
and may be outweighed by overhead \citep{kim2026scaling}. Accordingly, the
graph agreement measures in \cref{sec:multi-agent-evaluation} should be
complemented by semantic contract checks, end-to-end outcomes, cost accounting,
and failure attribution. Individual-agent evaluations
(\cref{sec:evaluation}) characterize candidate providers, while matched
comparisons test new-task coverage. The formal results give sufficient conditions under stated assumptions.
Measured average success rates alone do not establish the conditional
reliability premises.

\section{Multi-Agent Collaboration}
\label{sec:collaboration}
\label{sec:orchestration}

Multi-agent frameworks such as AutoGen and MetaGPT organize collaborators
within a shared framework \citep{wu2023autogen,hong2023metagpt}. Role
prompts and configurations distinguish agents that use the framework's
message interface and control loop. This design supports division of
labor, while leaving the system builder responsible for configuring and
maintaining the roles within a common harness architecture. By contrast, \raven
composes heterogeneous harnesses. Its roster combines the
built-in specialists with independently developed agents, such as Claude
Code and Codex, that connect through execution adapters. Collaboration
therefore joins execution capabilities that differ in tools, context
handling, and control policy. These differences can provide the complementary local capabilities
formalized in \cref{def:complementarity}.

This heterogeneity introduces coordination requirements. Agents differ in their
input interfaces, local-file access, conversation continuity, and failure
reporting. When a host relays results through free-form messages,
restating those results consumes context and may omit details. Without
explicit artifact references, later outputs may also be difficult to trace
to their sources. \raven
therefore makes the collaboration structure explicit. The Host Agent
submits a typed dependency graph, the runtime checks it before any worker
starts, and every handoff passes through a recorded artifact.

Collaboration proceeds in five stages. First, the host loads orchestration guidance on demand and uses it to
compose a graph. Second,
the runtime admits the graph after structural, capability, and environment
checks. Third, the executor dispatches nodes as their dependencies settle,
judges each result, and returns exceptions and clarification requests to
the host. Fourth, a run report returns file locations and terminal outputs
to the host for synthesis. Fifth, later graphs and single delegations reuse
completed nodes by reference. \Cref{sec:collaboration-planning} describes
the first two stages (\cref{fig:collab-planning}).
\Cref{sec:orchestration-execution} describes the next two
(\cref{fig:collab-lifecycle}), and \cref{sec:collaboration-memory} describes
the artifact and memory flow that supports reuse (\cref{fig:collab-flow}).
\Cref{sec:collaboration-group-memory} describes shared group memory, an
optional layer in which the host assesses each agent after a round and
returns the assessment at later dispatches. These mechanisms implement the assignment, input assembly, and scheduling
components of the plan in \cref{eq:composition-plan}. The semantic contracts
in \cref{sec:theory} describe the conditions under which the resulting
execution is correct.

As a running example, consider a user who asks for a survey of recent
progress in AI for AI research, a reproduction of the leading algorithms in
one experimental environment, and a web page that reports the results.
Two research nodes can independently investigate the algorithms and their
benchmarks. A coding node needs both investigations, and a design node
needs the reproduction results. The corresponding plan is
\[
    \{v_{\mathrm{alg}},v_{\mathrm{bench}}\}
    \longrightarrow v_{\mathrm{code}}
    \longrightarrow v_{\mathrm{design}}.
\]
The same agent identity may execute several nodes, but each node is a
distinct invocation with its own inputs, artifacts, and outcome. This
example illustrates both immediate artifact dependencies and the retention
of experience for later workflows.

\subsection{Task Decomposition and Agent Orchestration}
\label{sec:collaboration-planning}

\begin{figure}[t]
    \centering
    \includegraphics[width=\linewidth]{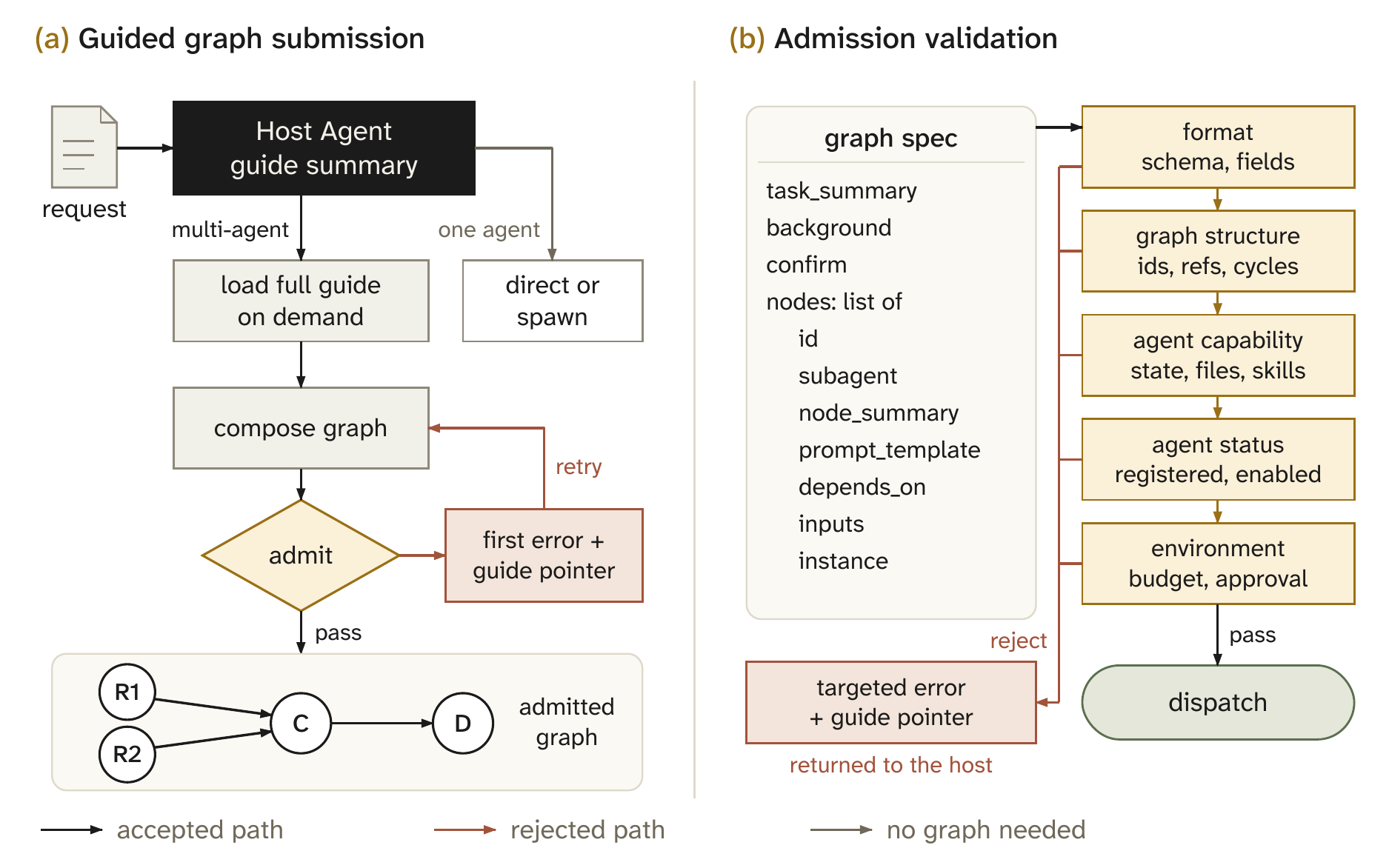}
    \caption{Graph planning and admission. (a) The host context carries
    only a summary of the orchestration guide. The host loads the full
    guide when a request needs several agents and delegates directly
    otherwise. A rejected graph returns its first error with a pointer back
    to the guide. In the admitted graph, R1 and R2 are the running example's
    research nodes $v_{\mathrm{alg}}$ and $v_{\mathrm{bench}}$, C is
    $v_{\mathrm{code}}$, and D is $v_{\mathrm{design}}$. (b) A submission is a flat node list with graph-level
    flags. Five groups of checks run in order before any worker is
    dispatched (\cref{tab:collab-admission}).}
    \label{fig:collab-planning}
    \label{fig:method-collaboration}
\end{figure}

\paragraph{A Heterogeneous Agent Registry.}\label{sec:agents}
An available agent is identified by a registered name and a capability
description. Its registry entry binds that name to one of four execution
backends: the in-process Raven agent loop, a command-line process, an Agent
Client Protocol (ACP) connection, or an OpenAI-compatible model API. The
entry also records three capabilities that constrain planning. A stateful
agent can continue an earlier conversation, a file-capable agent can read
local paths, and a progress-reporting agent streams its intermediate
activity. The default roster contains a general Raven agent and the four
specialists \ravenresearch, \ravencode, \ravendesign, and \ravenoncall.
Configuration presets connect third-party agents, including Claude Code,
Codex, OpenCode, Hermes Agent, and OpenClaw, through ACP. The Host
Agent receives the roster with these capability tags through its
orchestration interface. Its model selects among the entries using the
current context, without a separately trained routing classifier or assumed
oracle capability scores.

\paragraph{Staged Orchestration Guidance.}
A tool schema alone does not convey the rules an executable graph must
satisfy. The host must respect the placeholder syntax, the identifier rules,
and the capabilities of each backend. It must also decide whether a graph is
needed at all. Keeping the full rules in every context would incur an input cost even
on turns that do not require orchestration. \raven therefore discloses its orchestration guidance
in stages (\cref{fig:collab-planning}a). Every turn carries only a short
description of the orchestration guide. When the host decides that a
request requires several agents, the tool description instructs it to load
the full guide before composing a graph. The guide specifies node wiring,
the placeholder syntax, and the concurrency limit, and it states when a
single delegation is the better choice. A request that needs no graph incurs only the context cost of the short
description. Once loaded, the guide remains pinned in the session context
and available for subsequent graph construction. The runtime does not
block a submission when the guide is absent. Instead, every rejected graph
returns the validation error together with a pointer to the guide. This feedback directs the host to the relevant rules when validation fails.
The planning evaluation in \cref{sec:multi-agent-evaluation} measures the
graphs produced under this guidance.

\paragraph{Graph Construction.}
The host constructs the graph by working backward from the requested
deliverable. Each node must have an identifiable output and an executor
able to produce it. The host then adds a dependency whenever producing
one node's input requires another node's output. In the running example,
the two research assignments ask different questions and can start
independently. The coding assignment states how their findings select the
algorithms to reproduce and what executable result is required. The design
assignment names the expected page and the artifacts that support its
claims. A request that requires only one specialist uses the single-node
spawn interface, which shares the backend abstraction and the session's
node namespace.

\paragraph{Executable Node Specification.}
Let \(G_{\rm rt}=(V_{\rm rt},E_{\rm rt})\) be the submitted DAG. We identify
each node with its session-unique identifier, so \(V_{\rm rt}\) and
\(\operatorname{Dep}(v)\) use the same identifier universe. The record
\(n_v\) stored under identifier \(v\) contains the fields
\begin{equation}
    \begin{aligned}
    n_v&=(\mathrm{id}_v,\sigma(v),\mathsf{Tpl}_v,\operatorname{Dep}(v),
          \operatorname{Bind}_v,\mathrm{inst}_v,\operatorname{Cfg}_v),\\
    E_{\rm rt}&=\{(u,v):v\in V_{\rm rt},\ u\in \operatorname{Dep}(v)\cap V_{\rm rt}\}.
    \end{aligned}
    \label{eq:runtime-node}
\end{equation}
Here \(\mathrm{id}_v=v\) is the node identifier, and \(\sigma(v)\)
selects the registered executor, \(\mathsf{Tpl}_v\) is the prompt template, and
\(\operatorname{Dep}(v)\) lists dependencies. The input map \(\operatorname{Bind}_v\) binds template slots
to literal text, a file, or a previous node's output. An optional
instance handle \(\mathrm{inst}_v\) continues a stateful agent conversation.
Nodes that share a handle run sequentially in the same agent session.
The runtime also accepts optional capability overrides \(\operatorname{Cfg}_v\),
which specify the skills and MCP servers available for that invocation.
The model-facing schema does not advertise these overrides. Omitting an
override inherits the agent's configuration, while an empty list
explicitly requests none. Each node also carries a short summary for
progress reporting.
The exact interface fields are listed in \cref{app:method-contracts}.

A submission wraps the node list with a task summary and two flags. The
background flag is enabled by default. It returns control to the host as
soon as the run starts, and progress, exceptions, and the final report
arrive as later messages. The confirmation flag is disabled by default.
When enabled, it requests user approval for the listed nodes before
dispatch. If approval is denied, no node runs.

\paragraph{Prompt Construction and Artifact Binding.}
The task text is obtained by resolving the template against its
declared inputs and the available artifact ledger:
\begin{equation}
    x_v^{\rm task}=\operatorname{Render}
       \bigl(\mathsf{Tpl}_v,\operatorname{Bind}_v,(y_u)_{u\in \operatorname{Dep}(v)},d\bigr).
    \label{eq:runtime-render}
\end{equation}
The authorized request packet \(d\) includes the user's objective and
constraints, and \(y_u\) is the recorded output of node \(u\). The indexed family retains the node identifier associated
with each output, even when two nodes return identical content.
The host must carry the relevant parts into the node's
prompt because the executor does not have access to the host's private
conversation. A template inserts an upstream output with
\texttt{\{\{~id.output~\}\}} or passes its path with
\texttt{\{\{~id.output\_path~\}\}}. Declared inputs appear through
\texttt{\{\{~inputs.k~\}\}} or \texttt{\{\{~inputs.k.path~\}\}}, and files
through \texttt{\{\{~ref:path~\}\}} or \texttt{\{\{~ref\_path:path~\}\}}.
For example, \texttt{\{\{~research\_alg.output\_path~\}\}} gives a
file-capable coding agent access to the algorithm survey. Path references
avoid duplicating a long report in the initial prompt. An executor without
local-file access instead needs the corresponding content. Material too long
for a template, such as a detailed specification drafted by the host, is
written to a file and referenced in the same way. Substituted file and node contents are marked as untrusted data to
distinguish upstream artifacts from host instructions.

\paragraph{Admission Validation.}
The runtime admits a graph only after five groups of checks
(\cref{fig:collab-planning}b and \cref{tab:collab-admission}). Format checks
reject malformed arguments and unknown fields. Graph-structure checks
require unique identifiers, resolvable dependencies and inputs, confined
file paths, and acyclicity. Capability checks reject wiring that the
selected backend cannot honor, such as a shared instance on a stateless
agent or a path reference sent to an agent without local-file access.
Status checks require each named agent to be registered and enabled.
Environment checks refuse recursive delegation from inside a sub-agent run,
enforce a dispatch budget shared by graphs and single delegations, and
collect any requested user approval. Validation stops at the first failure
and returns a message that names the offending node or field. Several
messages also state the repair, for example replacing a path reference with
the content form for an agent without local-file access. The runner repeats
the checks while it reserves node identifiers, so two concurrent submissions
cannot claim the same identifier.

\begin{table}[t]
\centering
\small
\caption{Admission checks applied before dispatch, grouped as in
\cref{fig:collab-planning}b. Validation returns the first failure,
followed by a pointer to the orchestration guide.}
\label{tab:collab-admission}
\begin{tabularx}{\linewidth}{@{}l>{\raggedright\arraybackslash}X@{}}
\toprule
\textbf{Group} & \textbf{Checks} \\
\midrule
Format & Arguments parse as complete JSON. Fields have the required types
and values. No unknown field is present. Identifiers and instance handles
match their patterns. Required node fields are nonblank. \\
\addlinespace[2pt]
Graph structure & Identifiers are unique within the graph and unused in
the session. Dependencies name a node of the graph or a completed earlier
node with output. Every output placeholder is backed by a dependency.
Input values are well formed, and declared and referenced inputs agree.
Path forms apply only to file or node inputs. File and reference paths stay
within the permitted roots. The graph is acyclic. \\
\addlinespace[2pt]
Agent capability & Shared instances use stateful agents. Skill overrides
appear only at the head of an instance chain. Path forms are not sent to
agents without local-file access. \\
\addlinespace[2pt]
Agent status & Each named agent is registered and enabled. \\
\addlinespace[2pt]
Environment & The call does not come from inside a sub-agent run.
Delegation is not paused, and the dispatch budget is available. The user
approves the graph when confirmation is requested. \\
\bottomrule
\end{tabularx}
\end{table}

Admission establishes that the graph can be interpreted consistently. Because validation precedes dispatch, a rejected graph consumes a host
turn without starting a partial worker execution. However, semantic validation must additionally
establish whether the research assignments address the user's question and
whether the coding task has an adequate completion criterion.

\subsection{Host-Mediated Graph Execution}
\label{sec:orchestration-execution}

\begin{figure}[t]
    \centering
    \includegraphics[width=\linewidth]{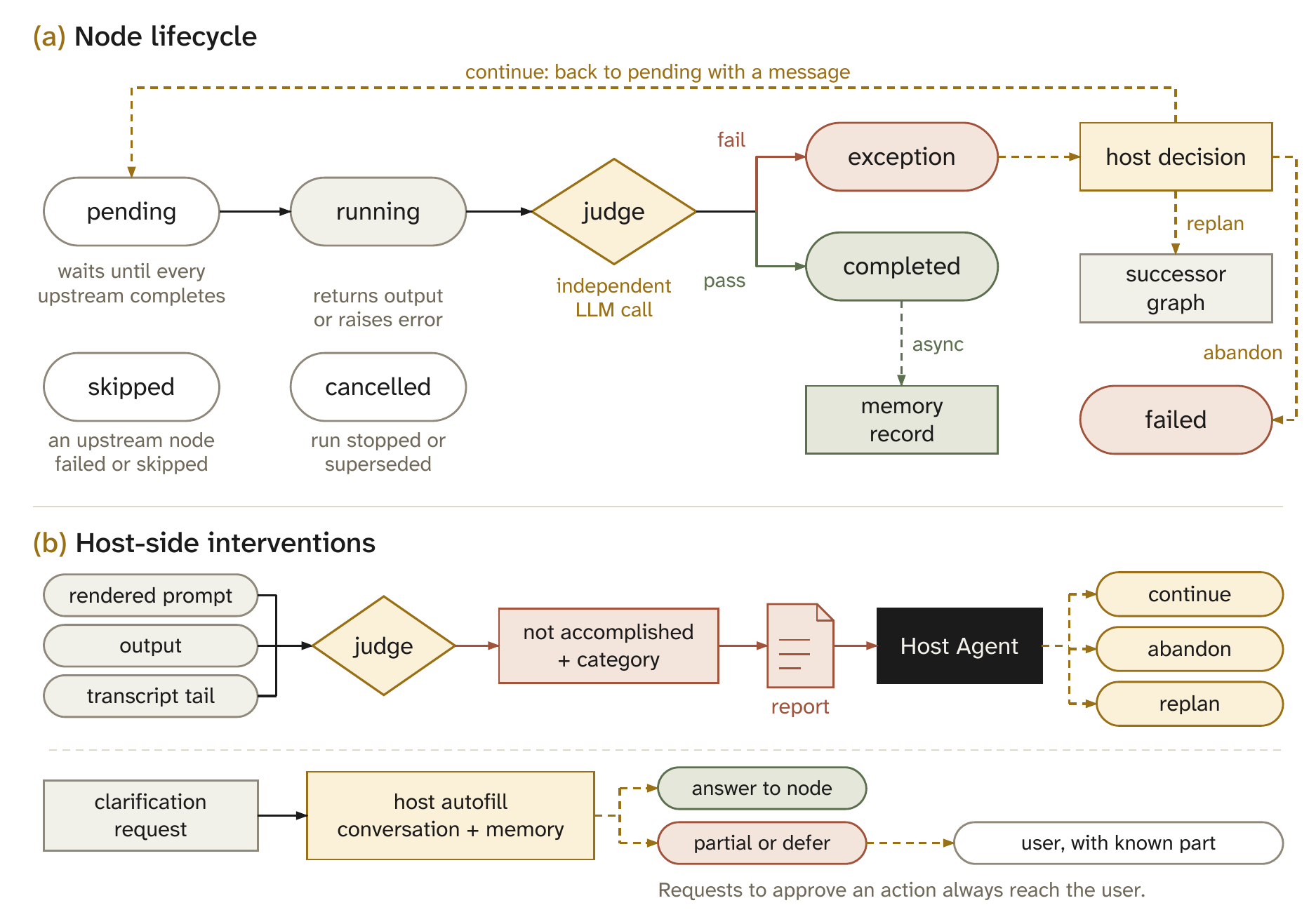}
    \caption{Node execution and host intervention. (a) A node runs once all
    predecessors have completed and settled. Returned output and raised
    errors both reach the judge. A failed verdict suspends the node for a
    host decision when continuations remain. Memory recording runs
    asynchronously and does not delay successors. (b) Top: the judge reads
    the rendered prompt, output, and transcript tail, and a negative verdict
    becomes an exception report for the host. Bottom: a worker's
    clarification request is answered from the host's conversation and
    memory when possible, and otherwise reaches the user.}
    \label{fig:collab-lifecycle}
\end{figure}

\paragraph{Node Lifecycle.}
Once a graph is admitted, each node moves through a small set of runtime states
(\cref{fig:collab-lifecycle}a). Let \(j_{\rm rt}\in\{0,1,\ldots\}\)
index scheduler events. Let \(z_{j_{\rm rt}}(v)\) denote the status of node
\(v\): pending, running, completed, exception, failed, skipped, or
cancelled. A node starts pending and runs once every in-graph predecessor
has completed and settled. The backend either returns an output or raises
an error, and both outcomes go to a completion judge. A positive verdict marks the node as completed. A negative verdict
suspends the node in the exception state if continuations remain and the
host can be reached.
Otherwise the node fails. A failed or skipped predecessor marks its pending
successors as skipped. An exception does not propagate, because the host may
still repair the node. Stopping a run cancels running and suspended nodes
and skips pending ones.

\paragraph{Dependency-Ready Scheduling.}
Two additional sets track submitted work and outcomes awaiting a final
verdict. Let \(\mathsf{Submitted}_{j_{\rm rt}}\) contain calls already assigned to an execution task,
including calls queued for a concurrency slot, and let \(\mathsf{Unsettled}_{j_{\rm rt}}\) contain
outcomes awaiting adjudication. The ready set is
\begin{equation}
    \mathsf{Ready}_{j_{\rm rt}}=
    \left\{v\in V_{\rm rt}:\begin{aligned}
    &z_{j_{\rm rt}}(v)=\mathsf{pending},\quad
      v\notin\mathsf{Submitted}_{j_{\rm rt}},\\
    &\forall u\in\operatorname{Dep}(v)\cap V_{\rm rt},\\
    &\quad z_{j_{\rm rt}}(u)=\mathsf{completed}
           \ \land\ u\notin\mathsf{Unsettled}_{j_{\rm rt}}
    \end{aligned}\right\}.
    \label{eq:runtime-ready}
\end{equation}
Earlier-session dependencies have already been checked at admission.
Excluding \(\mathsf{Submitted}_{j_{\rm rt}}\) prevents a node waiting for a slot from being submitted
twice. Excluding unsettled predecessors prevents a consumer from reading
the output of a call whose completion judgment may still be rejected.

Ready nodes execute under a shared concurrency semaphore. This limit applies across all graph runs and single delegations in the
session, keeping their combined execution within the available capacity.
Ready nodes sharing an instance handle are grouped and executed
sequentially, whereas independent instances can proceed concurrently.
A settled node wakes the scheduler immediately, so the successor of a
short research call need not wait for unrelated work to finish. In the
running example, the coding node becomes eligible only after both research
outcomes settle.

\paragraph{Completion Adjudication.}
A returned response alone does not establish task completion.
An agent may explain why it could not proceed, stop at an output limit, or
report success without producing the requested artifact. An independent
model call therefore judges each finished node
(\cref{fig:collab-lifecycle}b). The judge reads the rendered prompt, the
output, and the tail of the execution transcript. Two flags state whether
the transcript evidence is complete and whether the output was cut short.
The judge must answer through a constrained tool with the outcome
accomplished or not accomplished. A negative verdict names a category, such
as missing user input, a missing credential, a tool failure, unusable
dependency output, or an output limit. It also states what is missing and
cites supporting evidence. A supported negative answer to a research
question counts as accomplished, whereas a report that the research could
not be performed does not. A node judged unaccomplished loses its output
reference, so no successor can consume it.

However, the verdict does not certify objective correctness. Under the availability fallback, a normally returned call is treated as
accomplished if the judge fails or times out. Domain tests and final artifact inspection
provide independent evidence. The observed completed set in
\cref{eq:runtime-ready} therefore differs from the objective
successful-prefix event used in \cref{sec:theory-reliability}.

\paragraph{Exception Resolution.}
A suspended node sends an exception report to the host. The report names
the node and its agent, restates the task, and gives the verdict category,
the missing condition, and the evidence. It also lists the dependents that
the node blocks, the attempts that remain, and the decision deadline, and
it supplies ready-to-use calls of the resolution tool. The host uses the report and its conversation context to select a
resolution, consulting the user when the missing information requires user
input. It can continue the node with a message, abandon it, or replan the
remaining workflow. Continuation returns the node to pending with the message
attached. A stateful instance receives the message alone, whereas a
stateless agent receives its original prompt, its previous result, and the
message. Continuation retains the run and its completed artifacts, with a
bounded number of additional attempts. Abandonment fails the node and
skips its dependents. Replanning validates a successor graph, collects any
requested approval, and starts it under a new run identifier linked to the
old run. Unfinished nodes of the old run are skipped, cancelled, or marked
failed as superseded, and its completed nodes remain available by
reference. A report that receives no decision before its deadline fails the
node. Recovery therefore changes the executable plan without overwriting
earlier execution records.

\paragraph{Clarification Requests.}
A worker sometimes needs information that its prompt omits, such as a
preferred output format or the location of a dataset. Agents connected
through ACP can raise such questions while they run. Forwarding every
question to the user would interrupt the user for details that the
conversation already contains. The host therefore attempts an answer first
(\cref{fig:collab-lifecycle}b, bottom). An independent model call receives
the questions, a snapshot of the host's current conversation, earlier
answers in the session, and memories recalled from EverOS with the
questions as queries. It answers a question only when the user stated the
answer or a recalled preference determines it uniquely. It never answers a
request to approve an action, such as pushing, deleting, sending, or
paying. Unanswered and partly answered questions go to the user, annotated
with what the host already knows. The questions enter this call as untrusted data, preserving the distinction
between a worker's request and user authorization.

\begin{ravenalgorithm}{Host-Mediated Graph Execution}
\label{alg:collaboration}
\textbf{Input:} request packet \(d\), graph specification, registry,
artifact ledger, and concurrency allowance.\\
\textbf{Output:} a run report, terminal artifacts, and a persistent
execution record.
\begin{enumerate}\setlength{\itemsep}{2pt}
    \item Apply the admission checks in \cref{tab:collab-admission}. On the
    first failure, return the error and a pointer to the orchestration
    guide without running any node.
    \item If confirmation is requested, obtain user approval. Reserve the
    identifiers, persist the graph, and treat admitted earlier-session
    dependencies as completed artifacts.
    \item While runnable, running, or suspended work remains:
    \begin{enumerate}\setlength{\itemsep}{1pt}
        \item Propagate terminal failures to dependent nodes and apply
        host decisions to continue, abandon, or replan.
        \item Compute \(\mathsf{Ready}_{j_{\rm rt}}\) by \cref{eq:runtime-ready}. Group shared
        instances and reserve submitted nodes in \(\mathsf{Submitted}_{j_{\rm rt}}\).
        \item When a slot is available, assemble the node input \(x_v\)
        (\cref{eq:runtime-render,eq:group-memory-prefetch}), persist its
        prompt, and invoke the selected backend. Route clarification
        requests to host autofill, and forward the remainder to the user.
        \item Persist the output and transcript, and obtain the judge's
        verdict. Commit an accomplished outcome. Otherwise suspend the node
        and send an exception report, or fail it when no continuation
        remains. Remove settled outcomes from \(\mathsf{Unsettled}_{j_{\rm rt}}\) and wake the
        scheduler.
        \item Release the execution slot and schedule memory recording
        independently of artifact completion.
    \end{enumerate}
    \item Return the run report, with status counts, node output files,
    and terminal outputs, to the Host Agent for final synthesis.
\end{enumerate}
\end{ravenalgorithm}

\paragraph{Run Report and Final Synthesis.}
When a run finishes, the host receives a compact report that summarizes
intermediate outcomes (\cref{fig:collab-flow}b). The report gives the
numbers of completed, failed, cancelled, and skipped nodes and the run
directory. It lists each node with its status, instance handle, and output
file, together with the error of any node that did not complete. Only the
terminal outputs, from completed nodes without dependents, appear in full,
up to a length cap. In the running example, the host receives the design
node's output in full and file paths for the two surveys and the
reproduction. A foreground call returns an exception report as soon as a
node suspends, and the host's resolution call returns the next report or
this final report. A background call returns a run identifier at once and
later delivers progress, exception, and completion reports as messages.
Reports that quote node output are marked as untrusted data. The host then reads further files only if
its final response needs them.

Substantive final reasoning by the host is also an execution operation.
In the theoretical annotation it must appear as a host-policy call with
its own contract and cost. A final response that only delivers existing
artifacts incurs serialization and delivery costs. This distinction keeps
both synthesis and delivery within the plan's cost accounting.

\paragraph{Resource Accounting.}
The executor is event-driven and does not impose a universal wall-clock
limit on a graph. Individual backends, continuation policies, decision
deadlines, and the evaluation protocol supply their own limits.
Consequently, the common budget in \cref{sec:theory} must be instantiated
explicitly when measuring a composed run. Structural scheduling does not
establish a cost advantage by itself. Planning, verification, retries,
clarification, and handoffs must all be included in the measured resource
use.

\subsection{Multi-Agent Memory Flow}
\label{sec:collaboration-memory}

\begin{figure}[t]
    \centering
    \includegraphics[width=\linewidth]{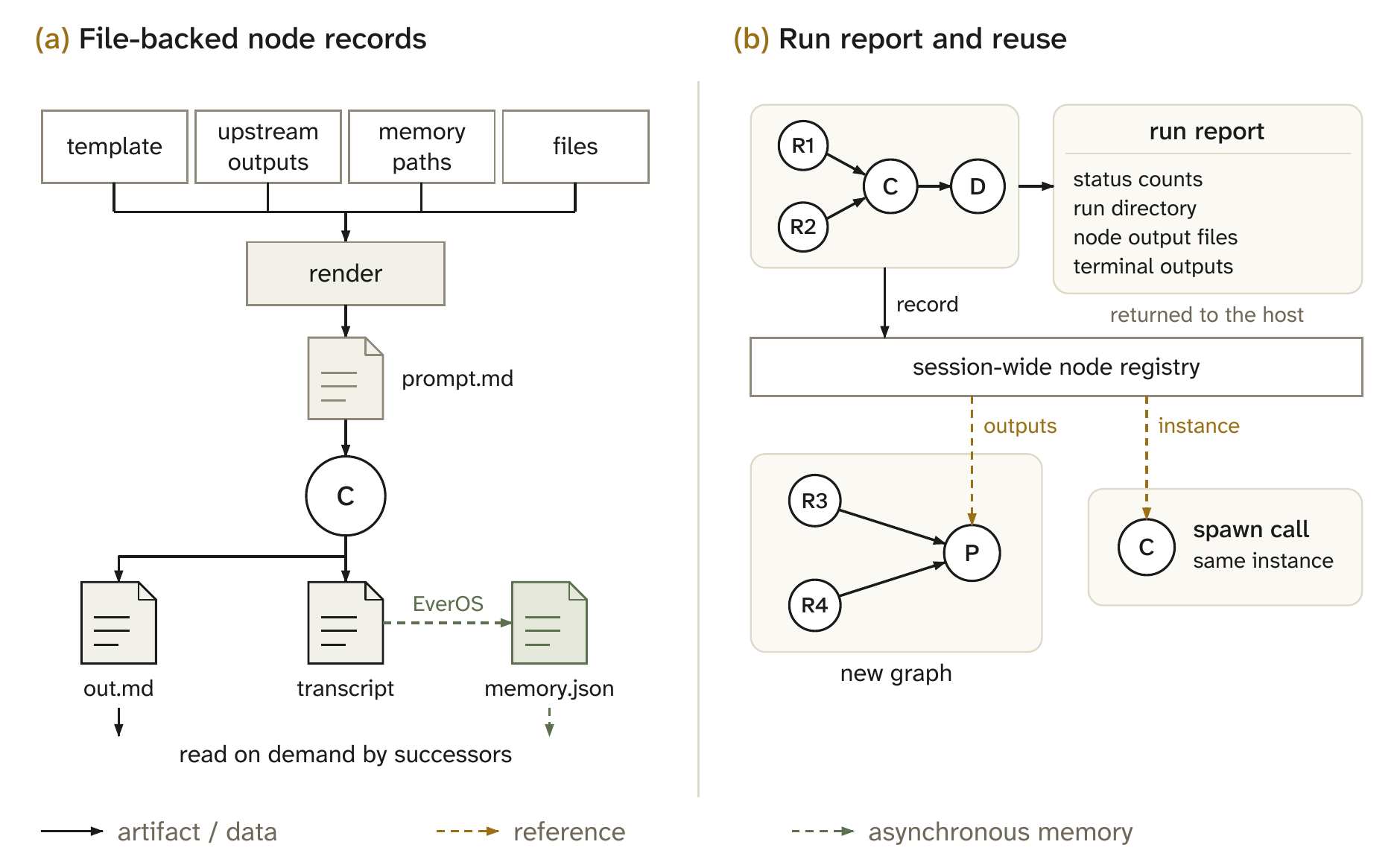}
    \caption{Artifact and memory flow. (a) A node's prompt is rendered from
    its template, upstream outputs, upstream memory paths, and files. The
    runtime writes the prompt, output, and transcript to disk, and a memory
    record follows asynchronously. Successors read these files on demand.
    (b) The host receives a run report with paths and terminal outputs.
    Completed nodes enter a session-wide registry, so a later graph can
    reference their outputs and a single delegation can continue their
    instance. In (b), R3 and R4 are research nodes of a follow-up graph
    whose slide-deck node P references earlier outputs, and the spawn call
    continues the coding instance C.}
    \label{fig:collab-flow}
\end{figure}

\paragraph{Two Information Paths.}
A collaboration carries both task artifacts and retained experience.
An artifact \(y_v\) is the immediate output needed to execute the
submitted graph. A memory record contains what a memory backend
has distilled from that invocation. In the running example, the design
node needs the reproduction results regardless of whether a reusable
case has already been extracted from the coding trace. Artifact
completion therefore releases dependencies, while memory recording proceeds
asynchronously.

\paragraph{File-Backed Node Records.}
Each handoff uses a recorded file containing the producer's output (\cref{fig:collab-flow}a). The runtime stores each run's graph
and manifest separately from a session-wide node directory. For every node
it writes the rendered prompt before dispatch, the output when the backend
returns, and the transcript before the judge reads it. A continued node
keeps a numbered copy of each attempt's prompt, output, and transcript.
A successor reads an upstream output inline or by path, and the host
inspects a file only when it needs the content. File references avoid repeating long artifacts in the host's prompt and
in every successor's prompt. They also give successors access to the
producer's output without an intervening paraphrase. The records identify
which artifact each downstream node consumed, including after a foreground
wait ends.

For a memory-enabled worker, the host also reads that invocation's session
twice: episodes on the user track and cases on the agent track
(\cref{sec:memory-session}). These reads retrieve records by session identifier without relevance
ranking. They do not request profiles or skills. When the worker's
memory source is its own backend, the host reads that backend and does
not write on the worker's behalf. When the source is a captured trace,
the host stores the trace in the host's own backend under an explicit
session identity and marks the write final. The record preserves
whether its source was the worker backend or the host-captured trace.
The two processes can use different backends and identities.

\paragraph{Dependency-Scoped Propagation.}
For a file-capable node \(v\), the runtime appends paths to the memory
records of its transitive predecessors:
\begin{equation}
    \mathsf{Paths}_v=
    \{\operatorname{MemPath}(u):u\in\operatorname{Anc}_{G_{\rm rt}}(v)\}.
    \label{eq:collaboration-memory-paths}
\end{equation}
Here \(\operatorname{MemPath}(u)\) is the deterministic file location
reserved for node \(u\)'s memory record.
This block is generated from the graph and does not require the host to
repeat every memory reference in the template. Paths give the receiving
agent the option to inspect relevant experience without eagerly copying
all ancestor memories into its prompt. The design node can consequently
inspect the coding case and both research records. A backend without
local-file access receives no unusable path block. Any information it
requires must be supplied through its supported input channel.
Write \(\operatorname{PathBlock}_v\) for this rendered block, or empty
text when the backend cannot read files. The base input is
\(x_v^{\rm base}=x_v^{\rm task}\mathbin{+\!\!+}\operatorname{PathBlock}_v\),
where \(\mathbin{+\!\!+}\) concatenates text in order.

\paragraph{Cross-Graph Reuse.}
Node identifiers are unique across the whole conversation, and graph nodes
and single delegations share one registry (\cref{fig:collab-flow}b). A
later graph can therefore depend on a node that an earlier run completed,
or read its output through a placeholder, without executing it again.
Admission verifies that the earlier node completed with a recorded output.
A failed, skipped, cancelled, or still-running node cannot be used as if it
had produced a result. Instance handles also persist across runs. A later
node that names the same handle continues the same agent session with its
accumulated context. In the running example, a follow-up request for a
slide deck can reuse both surveys and the reproduction by reference. A
single delegation can continue the coding instance to repair one
experiment. Repeating work requires a fresh identifier, which preserves the
earlier record. Reuse avoids redundant subtasks and limits the host's context to
references to earlier artifacts.

\paragraph{Cross-Task Reuse and Its Boundary.}
Within a graph, ancestry determines which memory-record paths are
advertised. Across sessions, task-aware retrieval names the user or
the agent, and the server applies its default application and project
scope (\cref{sec:memory-retrieval}). Cases can subsequently support skill extraction
(\cref{sec:skillforge-evolution}), making procedures from individual
executions available to future tasks.
\Cref{sec:collaboration-group-memory} extends this path with assessments that the host records for each agent after a round.
The mechanism provides dependency-scoped provenance and bounded context
transfer while retaining experience for future tasks. Its effect on
final-task success requires evaluation beyond graph agreement.

\subsection{Shared Group Memory}
\label{sec:collaboration-group-memory}

\begin{figure}[t]
    \centering
    \includegraphics[width=\linewidth]{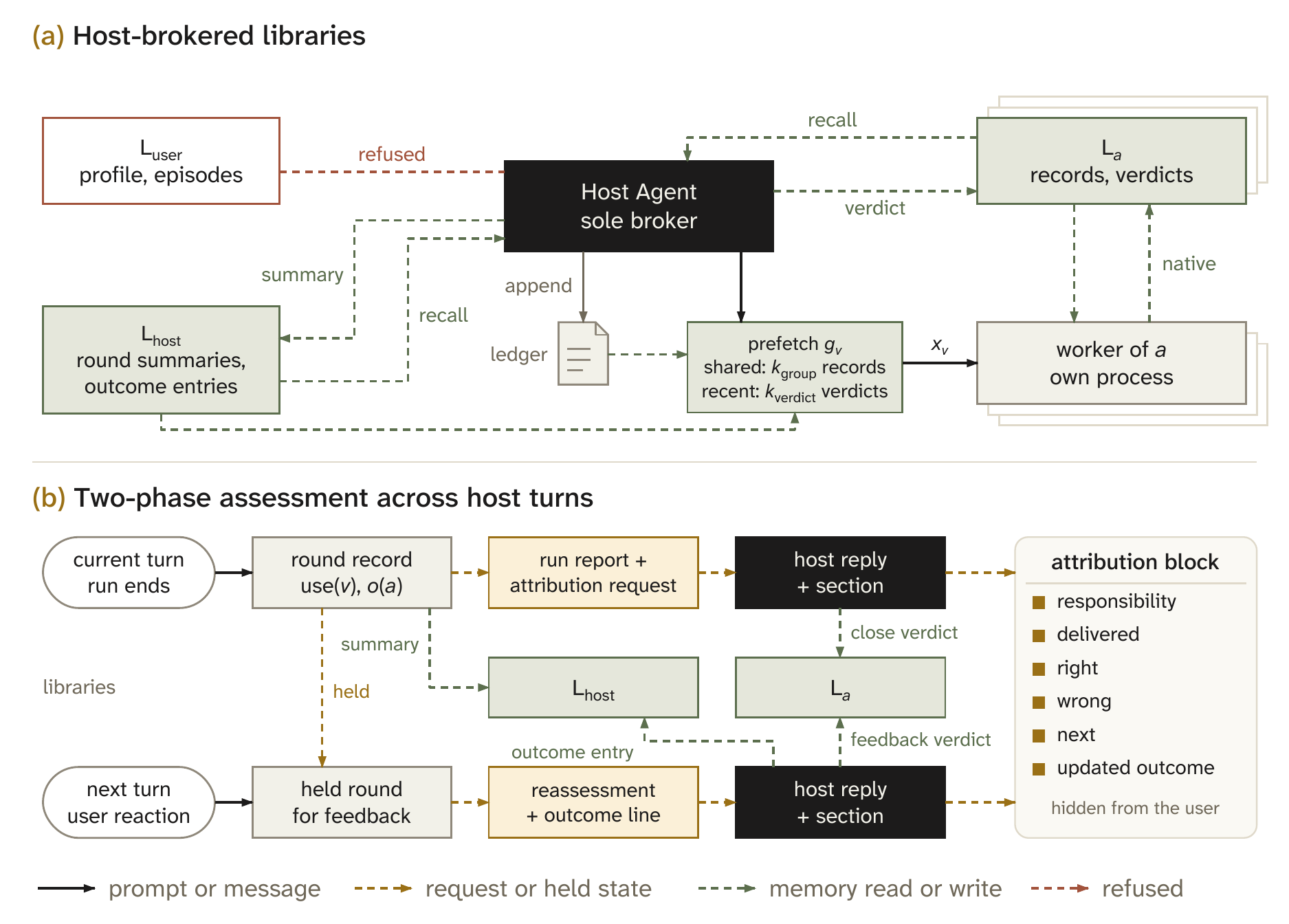}
    \caption{Shared group memory. (a) The host mediates access to
    owner-partitioned EverOS libraries. It retrieves experience before
    planning, deposits round summaries and outcomes in its shared library,
    and writes agent assessments to the corresponding agent libraries.
    Dispatch prefetch combines shared experience with recent assessments
    from a provenance ledger, with respective caps \(k_{\rm group}\)
    and \(k_{\rm verdict}\). Personal user memory is excluded, while
    workers retain their ordinary memory access. Offset outlines indicate
    one library and worker per mapped agent.
    (b) Assessment has two phases: a close verdict after execution and a
    feedback verdict on the next user turn. Attribution is extracted from
    the host's reply and omitted from the user-visible response and session
    log.}
    \label{fig:group-memory}
\end{figure}

A worker observes its own execution but lacks direct access to downstream
use of its output and subsequent user feedback. The host observes both
through run reports and later user messages. To incorporate this broader
feedback, shared group memory converts these observations into assessments that inform future orchestration and
worker invocations (\cref{fig:group-memory,alg:group-memory}). We define a
\emph{round} as a finished graph run or a single delegation, represented as
a one-node graph.
The completion judge evaluates individual nodes
(\cref{sec:orchestration-execution}), while group memory records the host's
assessment of each agent's contribution across the round.

The layer reuses the active EverOS backend and stores three record types:
a deterministic round summary, per-agent assessments called \emph{verdicts},
and an outcome entry after user feedback. Summaries and outcome entries
belong to the host's shared library, while verdicts belong to the assessed
agents' libraries. Bounded retrieval supplies this experience before host
planning and worker execution. Artifact transfer follows the graph and
cross-graph references described in \cref{sec:collaboration-memory}.

\paragraph{Memory Ownership and Access.}
A library is the collection of records under one EverOS owner identifier
(\cref{sec:memory-formation}). Let \(\mathsf L_{\mathrm{user}}\) contain
the user's profile and episodes. The host's shared library
\(\mathsf L_{\mathrm{host}}\) uses a separate owner derived from the
same user identifier. For each agent \(a\) in the configured subset
\(\mathcal A_{\mathrm{map}}\), let \(\mathsf L_a\) use the worker's
existing owner identifier \(\operatorname{own}(a)\). Depositing verdicts
under this owner makes them available through the worker's native recall.
All instances of an agent share its library, as do agents mapped to the
same owner. Reuse therefore operates at the owner level, without isolation
between instances that share an owner. Unmapped agents receive no group
memory prefetch or verdicts.

The host performs all group memory reads and writes under the policy in
\cref{tab:group-memory-access}. Workers receive selected records through
prefetch and retain their ordinary access to their own libraries. Group
memory excludes \(\mathsf L_{\mathrm{user}}\) and does not expose agent
libraries under other owners to workers. Verdict targets are resolved from
the configured owner map and restricted to agents dispatched in the round.
Generated assessment text cannot select a target owner.

\begin{table}[t]
\centering
\small
\caption{Access within the group memory layer. The host performs all
retrievals and deposits. Ordinary host and worker memory calls are outside
this policy.}
\label{tab:group-memory-access}
\begin{tabularx}{\linewidth}{@{}lllX@{}}
\toprule
\textbf{Actor} & \(\mathsf L_{\mathrm{host}}\) & \(\mathsf L_{\mathrm{user}}\)
& \textbf{Agent Libraries} \\
\midrule
Host & read, write & none & read and write every mapped library \\
Mapped worker & read via prefetch & none & its own verdicts via prefetch \\
Unmapped agent & none & none & none \\
\bottomrule
\end{tabularx}
\end{table}

\paragraph{Recall Before Orchestration.}
For an incoming host message \(x_{\mathrm{in}}\), define the queried
libraries as \(\mathcal Q=\{\mathsf L_{\mathrm{host}}\}\cup
\{\mathsf L_a:a\in\mathcal A_{\mathrm{map}}\}\). The host receives
\begin{equation}
    g_{\mathrm{orch}}=\operatorname{Fence}_{B_{\mathrm{orch}}}\!\left(
    \operatorname{Merge}\!\left(
    (\operatorname{Top}_{k_{\rm group}}(\mathsf L,x_{\mathrm{in}}))
      _{\mathsf L\in\mathcal Q}\right)\right).
    \label{eq:group-memory-recall}
\end{equation}
Here the positive integer \(k_{\rm group}\) is a retrieval cap, and
\(\operatorname{Top}_{k_{\rm group}}(\mathsf L,q)\) returns up to \(k_{\rm group}\)
records ranked for query \(q\) within \(\mathsf L\)
(\cref{sec:memory-retrieval}). \(\operatorname{Merge}\) groups results
by source library and removes duplicate lines, and
\(\operatorname{Fence}_{B_{\rm char}}\) renders them as untrusted context within a
character budget \(B_{\rm char}\), instantiated as \(B_{\mathrm{orch}}\) here and
\(B_{\mathrm{worker}}\) below. Without a subscript, \(\operatorname{Fence}\) applies no
character budget. Retrieval covers every mapped agent, including
those not yet dispatched, so prior experience is available when selecting
an executor. Agent libraries contribute both native worker records and
host verdicts. The block is appended to the incoming message and excluded
from the session log to prevent automatic replay on later turns.

\paragraph{Prefetch at Dispatch.}
For a node \(v\) assigned to a mapped agent \(\sigma(v)\), the host
augments the base input \(x_v^{\rm base}\), including the path block
from \cref{sec:collaboration-memory}:
\begin{equation}
    \begin{aligned}
    g_v&=\operatorname{Fence}_{B_{\mathrm{worker}}}\!\left(
      \operatorname{Merge}\!\left(
      (\operatorname{Top}_{k_{\rm group}}(\mathsf L_{\mathrm{host}},x_v^{\rm base}),
      \operatorname{Recent}_{k_{\rm verdict}}(\sigma(v)))\right)\right),\\
    x_v&=x_v^{\rm base}\mathbin{+\!\!+}g_v.
    \end{aligned}
    \label{eq:group-memory-prefetch}
\end{equation}
The operator \(\operatorname{Recent}_{k_{\rm verdict}}(a)\) selects the most recent
verdicts about agent \(a\), up to a positive integer cap \(k_{\rm verdict}\), from the
host's provenance ledger. Shared experience is selected by
query relevance, while verdicts are selected by recency. This gives recent
feedback a retrieval path independent of its similarity to the new task or
its ranking against the worker's native records. The block follows the
task so that the worker's memory records the assignment as its task
context. Single delegations use the same rule. If this layer is disabled
or the agent is unmapped, \(g_v\) is empty and \(x_v=x_v^{\rm base}\).

The final dispatched input \(x_v\) is the object corresponding to
\cref{eq:input-assembly}. Applying that theory to memory-enabled execution
requires an annotated expansion \(G_{\rm ann}=(V_{\rm ann},E_{\rm ann})\) of
\(G_{\rm rt}\). Semantic retrieval and other context-producing computation
become explicit operations with their own contracts, costs, and recorded
outputs. The theory is instantiated with \(G=G_{\rm ann}\), so
\(K_{\rm op}=|V_{\rm ann}|+|E_{\rm ann}|\), resource allocations, and local
error sums refer to the expanded plan. Authorized fixed context
is bound in \(d\), and later retrieval results enter through declared
messages. The final \(\Phi_v\) only projects or serializes these inputs.
This annotation preserves the theory's assembly restriction and accounts
for retrieval cost. Supplying a memory block alone does not establish these semantic guarantees.

\paragraph{Deterministic Round Outcomes.}
For round \(r\) with graph \(G_{\rm rt}=(V_{\rm rt},E_{\rm rt})\), let \(z(v)\) denote the final
status of node \(v\) and \(\operatorname{Succ}(v)=\{w\in V_{\rm rt}:v\in \operatorname{Dep}(w)\}\)
its successors within the round. Stopped runs are excluded from recording
because cancellation and skipping may reflect the stop decision. For other
finished rounds, define
\begin{equation}
    \operatorname{use}(v)=
    \begin{cases}
        \mathsf{terminal}, & \operatorname{Succ}(v)=\varnothing,\\
        \mathsf{consumed}, & \exists w\in\operatorname{Succ}(v):
            z(w)\in\{\mathsf{completed},\mathsf{failed}\},\\
        \mathsf{unused}, & \text{otherwise}.
    \end{cases}
    \label{eq:group-memory-use}
\end{equation}
Here \(\mathsf{consumed}\) indicates that a successor executed, even if
it failed, and \(\mathsf{terminal}\) indicates that no in-round successor
exists. These labels describe execution structure and do not establish
output quality or user acceptance.

Let \(\mathcal A_r=\{\sigma(v):v\in V_{\rm rt}\}\) be the participating agents
and \(V_a=\{v\in V_{\rm rt}:\sigma(v)=a\}\). For \(a\in\mathcal A_r\), the
initial outcome is the first matching case below:
\begin{equation}
    o(a)=
    \begin{cases}
        \mathsf{cancelled}, & \exists v\in V_a: z(v)=\mathsf{cancelled},\\
        \mathsf{failed}, & \exists v\in V_a: z(v)=\mathsf{failed},\\
        \mathsf{cancelled}, & \exists v\in V_a: z(v)\neq\mathsf{completed},\\
        \mathsf{adopted}, & \forall v\in V_a:
            \operatorname{use}(v)\neq\mathsf{unused},\\
        \mathsf{partially\_adopted}, & \exists v\in V_a:
            \operatorname{use}(v)\neq\mathsf{unused},\\
        \mathsf{rejected}, & \text{otherwise}.
    \end{cases}
    \label{eq:group-memory-outcome}
\end{equation}
The third case maps skipped nodes to \(\mathsf{cancelled}\). Adoption
labels apply only when all of the agent's nodes completed, with terminal
outputs counted as adopted. At this phase, \(\mathsf{rejected}\) denotes
unused outputs. The host immediately deposits a round summary containing
node assignments, statuses, dependencies, task excerpts, and agent outcomes
in \(\mathsf L_{\mathrm{host}}\). Task excerpts provide semantic context
for later retrieval.

\paragraph{Two-Phase Agent Assessment.}
A verdict associates the host's assessment with an agent, round, and phase.
For \(a\in\mathcal A_r\cap\mathcal A_{\mathrm{map}}\), write
\begin{equation}
    \mathsf{Verd}_{r,a}^{(p)}=
    \bigl(\mathrm{id}_H,a,r,p,o_{r,a}^{(p)},\mathrm{body}_{r,a}^{(p)}\bigr),
    \qquad p\in\{\mathrm{close},\mathrm{feedback}\},
    \qquad o_{r,a}^{(\mathrm{close})}=o(a),
    \label{eq:group-memory-verdict}
\end{equation}
where \(\mathrm{id}_H\) identifies the host as author and \(\mathrm{body}_{r,a}^{(p)}\) is the
assessment text. Each verdict is deposited in \(\mathsf L_a\), with its
author, subject, round, and outcome encoded in a leading header.

At round closure, the run report requests one attribution block per
distinct participating agent. The host appends these blocks to its existing
reply, requiring no additional model call. Each block addresses the agent
directly and describes its responsibility, delivered work, strengths,
failures and their causes, and instructions for future tasks. Assessments
must be grounded in observations, while node statuses come from the run
record. The layer preserves usable block text in the close verdict and
removes the attribution section before displaying or logging the reply.
Missing or unparseable blocks yield deterministic summaries of node status
and output use, with distinct headers identifying the fallback reason.

The layer also retains each closed round for one reassessment on the next user
turn. The host supplies updated attribution blocks and proposed outcome
labels \(\widehat o_{r,a}\), with missing or invalid labels represented by
\(\bot\). Let \(\mathsf{Labels}\) contain the five distinct labels in
\cref{eq:group-memory-outcome}. For each reassessed mapped agent,
\begin{equation}
    o_{r,a}^{(\mathrm{feedback})}=
    \begin{cases}
        \widehat o_{r,a}, & \widehat o_{r,a}\in\mathsf{Labels},\\
        o_{r,a}^{(\mathrm{close})}, & \text{otherwise}.
    \end{cases}
    \label{eq:group-memory-feedback}
\end{equation}
Only this phase permits model-generated text to revise the deterministic
outcome. The feedback verdict is appended alongside the close verdict under
the same round identifier. An outcome entry in
\(\mathsf L_{\mathrm{host}}\) records the reassessed labels and their
initial values when changed. The round is then released from pending
reassessment. Without a subsequent user turn, only the close verdict is
retained.

In the running example, suppose \ravenresearch is mapped. It receives one
close verdict covering both research nodes. If the user later identifies
an outdated benchmark version, the host can revise its outcome to
\(\mathsf{rejected}\) and
include a corrective instruction in the feedback verdict. Subsequent
prefetch can expose this verdict to another \ravenresearch instance even
when the task differs, subject to the recency and character budgets. The
host can also retrieve it when planning a related task.

\begin{ravenalgorithm}{Shared Group Memory Across Host Turns}
\label{alg:group-memory}
\textbf{Input:} host messages, dispatched rounds, owner map
\(\operatorname{own}\), libraries \(\mathcal Q\), and provenance ledger.\\
\textbf{Output:} host and worker context blocks, round summaries, agent
verdicts, and outcome entries.
\begin{enumerate}\setlength{\itemsep}{2pt}
    \item Before host planning, append \(g_{\mathrm{orch}}\) from
    \cref{eq:group-memory-recall} to the incoming message without logging
    the block.
    \item At dispatch to a mapped agent, send \(x_v\) from
    \cref{eq:group-memory-prefetch}.
    \item At completion of a non-stopped round, compute output-use labels
    and agent outcomes by
    \cref{eq:group-memory-use,eq:group-memory-outcome}, deposit the round
    summary in \(\mathsf L_{\mathrm{host}}\), and retain the round for
    reassessment.
    \item Request attribution in the host's reply to the run report.
    Extract and deposit \(\mathsf{Verd}_{r,a}^{(\mathrm{close})}\) for each
    participating mapped agent, using labeled fallbacks where needed.
    Strip attribution before displaying or logging the reply.
    \item On the next user turn, request reassessment, apply
    \cref{eq:group-memory-feedback}, append feedback verdicts and a shared
    outcome entry, and release the round.
    \item Enforce \cref{tab:group-memory-access} on every read and write,
    and record each deposit in the provenance ledger.
\end{enumerate}
\end{ravenalgorithm}

\paragraph{Provenance and Limitations.}
EverOS records identify an owner but have no separate author field, and
extraction can shorten or rewrite their text. Verdict headers therefore carry
attribution within the content. To preserve the original assessments,
a host-side append-only ledger retains
the authoritative deposit text, target owner, record kind, round, and
covered nodes. Each deposited record ends with a ledger reference, and the
ledger supplies verdicts for recency-based prefetch. The layer neither
expires records nor reviews workers' native writes. Verdicts remain
observational judgments by the host and may be incorrect. Group memory is
disabled by default and is not evaluated in this report. Its effect on
orchestration and task performance remains to be established.

\section{Agent Harness Self-Evolution}\label{sec:evolver}

Harness self-evolution seeks to improve the execution policy surrounding a
frozen language model. Building on HarnessBank \citep{luo2026harnessbank},
the method exposes the policy through four strategy interfaces, diagnoses
execution failures, constructs new or recombined harnesses, screens their
behavior, and retains evaluated candidates in a semantic gene bank.
\Cref{sec:evolver-modularity,sec:evolver-diagnosis,sec:evolver-bank,sec:evolver-screening}
specify the complete method, and \cref{sec:eval-evolution} summarizes the
published harness-evolution experiments. A Task Agent executes tasks, an Evolver Agent proposes
modifications, and a fixed Evaluator measures their outcomes. The resulting agent is available to the Host Agent for subsequent tasks. \Cref{fig:method-evolution} connects the
editable policy to the search procedure. We first define the modules
and admissible edits, then the diagnostic evidence, archive search,
and verification rules.

\begin{figure}[!htbp]
    \centering
    \includegraphics[width=\linewidth]{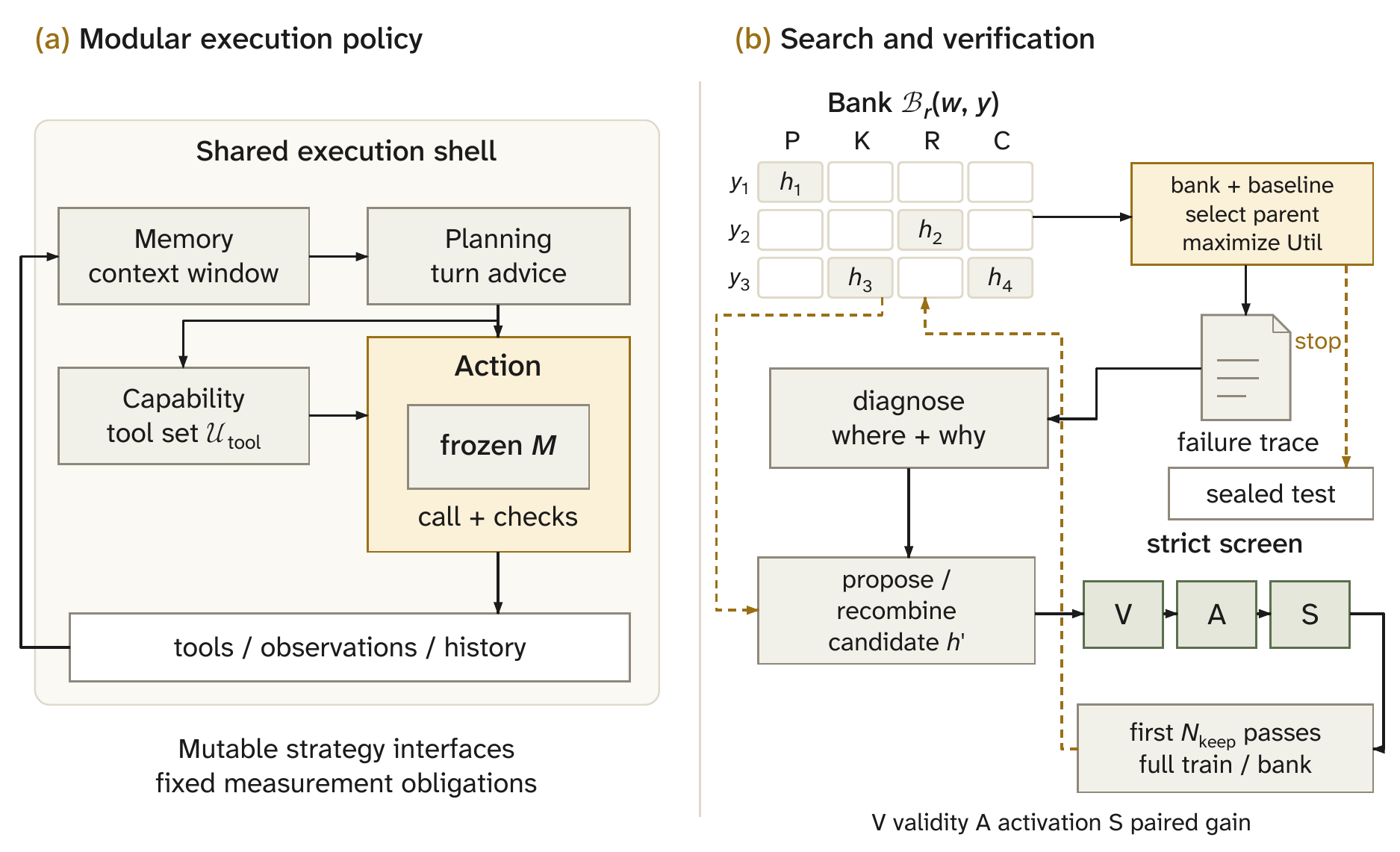}
    \caption{Harness self-evolution. (a) Four strategy interfaces expose
    decisions around a frozen model. The shared execution shell handles dispatch and
    accounting. (b) The best training-evaluated harness supplies failure
    traces for diagnosis and candidate generation. The bank stores one
    complete harness per edit-category/pathology cell. P/K/R/C denote
    prompt, knowledge, runtime, and configuration edits. The first
    \(N_{\rm keep}\) candidates passing validity, activation, and paired-gain
    screening receive full-training evaluation before cell competition,
    where \(N_{\rm keep}\) is the per-round survivor cap.
    The final training-selected harness is frozen before held-out testing.}
    \label{fig:method-evolution}
\end{figure}

\subsection{Modular Harness Design}\label{sec:evolver-modularity}

\paragraph{The Policy Surrounding a Model.}
For the model--harness pair in \cref{eq:harnessed-agent}, let \(M\)
denote the frozen Task Agent model and represent its configurable
harness as
\begin{equation}
    h=(\mathrm{id}_h,\mathsf{Mem}_h,\mathsf{Plan}_h,\mathsf{Cap}_h,
       \mathsf{Act}_h;\theta_h).
    \label{eq:harness-factorization}
\end{equation}
The identifier \(\mathrm{id}_h\) distinguishes candidates even when their
executable components coincide. It keys the candidate's descriptor, activation
specification, creation order, and designated evaluation records. The four
strategy components govern execution. \(\theta_h\) contains
their prompts, domain instructions, and configuration, rather than
the weights of \(M\). \textbf{Memory} assembles the working context
and decides what to remove or summarize under context pressure.
\textbf{Planning} prepares turn guidance and incorporates domain
advice. \textbf{Capability} selects the tool definitions exposed to
the model. \textbf{Action} obtains a usable model response and judges
actions or completion. These roles implement decisions within one
agent. The shared execution shell handles tool dispatch, iteration
accounting, persistence, events, and bounded retry control.

The interfaces connect through the context presented to the model.
Let \(x_{\rm in}\) be a request, \(\mathsf{Hist}_{\rm local}\) its available interaction history,
and \(B_{\rm ctx}>0\) the context allowance. The initial message
sequence is
\begin{equation}
    x^{(0)}=\mathsf{Plan}_h
       \bigl(x_{\rm in},\mathsf{Mem}_h(\mathsf{Hist}_{\rm local},x_{\rm in},B_{\rm ctx})\bigr).
    \label{eq:harness-initial-context}
\end{equation}
At model-call iteration \(j\ge0\), with current messages \(x^{(j)}\),
the next two decisions are abstracted as
\begin{equation}
    \mathcal U_{\rm tool}^{(j)}=\mathsf{Cap}_h(x^{(j)},j),
    \qquad
    u^{(j)}\sim\mathsf{Act}_h
        (\,\cdot\mid M,x^{(j)},\mathcal U_{\rm tool}^{(j)}).
    \label{eq:harness-iteration}
\end{equation}
Here \(\mathcal U_{\rm tool}^{(j)}\) is the available tool set and \(u^{(j)}\)
is a response containing tool requests or a proposed final answer.
The conditional distribution includes model sampling. The shell
dispatches tool requests and appends observations to the history to
construct the next context. A proposed final answer is subject to
the Action role's completion judgment. Under context pressure, Memory
proposes a reduced window and the shell controls the retry. These equations abstract the information flow across the strategy interfaces. The default Planning preparation preserves its
input messages, while domain-specific components can add iterative advice.

\paragraph{Domain Specialization and Mutable Boundaries.}
A specialized agent combines these strategies with domain tools,
skills, and instructions. Research can emphasize evidence sufficiency,
while coding can emphasize execution feedback and final verification.
The selected memory backend supplies context to \(\mathsf{Mem}_h\),
whereas the strategy itself controls the working window. This
distinction permits changing the use of memory without redefining
the storage system in \cref{sec:memory}.

Building on this modular design, evolution starts from a baseline harness
\(h_0\). Let \(\mathcal X_{\rm edit}\)
be the allowed edit locations and \(\mathcal K_{\rm fixed}\) the immutable
locations, with \(\mathcal X_{\rm edit}\cap\mathcal K_{\rm fixed}=\varnothing\).
The immutable kernel protects evaluation, bookkeeping, and required
execution interfaces. For a concrete patch \(\Delta^{\rm patch}\),
\(\operatorname{scope}(\Delta^{\rm patch})\) is the set of locations it changes
and \(\operatorname{Apply}(h,\Delta^{\rm patch})\) materializes a harness with
a fresh candidate identity. The identity is an analytical label and is
excluded from executable edit locations and protected contents.
Write \(\operatorname{Exec}(h)\) for the executable components of \(h\),
excluding its identity. The admissible search space is
\begin{equation}
\begin{split}
    \mathfrak H(\mathcal X_{\rm edit})=\{h:\;&\exists\Delta^{\rm patch},\\
       &\operatorname{Exec}(h)=\operatorname{Exec}(\operatorname{Apply}(h_0,\Delta^{\rm patch})),\\
       &\operatorname{scope}(\Delta^{\rm patch})\subseteq\mathcal X_{\rm edit},\quad
        h|_{\mathcal K_{\rm fixed}}=h_0|_{\mathcal K_{\rm fixed}},\\
       &\operatorname{InterfaceValid}(h)=1\}.
\end{split}
    \label{eq:harness-search-space}
\end{equation}
Restriction \(h|_{\mathcal K_{\rm fixed}}\) denotes the protected contents.
\(\operatorname{InterfaceValid}\) checks construction and interface
compatibility. We assume \(h_0\) satisfies these requirements.
A patch applied to an evolved parent must leave the resulting harness
in the same search space: its cumulative modifications still obey the
baseline's edit restrictions. Admissibility establishes executability,
while task-level improvement requires evaluation.

\subsection{Diagnosis-Driven Harness Evolution}
\label{sec:evolver-diagnosis}

\paragraph{Objective and Measurement Unit.}
Let \(\mathcal D_{\rm tr}\) be the finite training task set available
to evolution and \(\mathcal D_{\rm te}\) a disjoint, sealed test set.
A fixed protocol \(\Pi_{\rm eval}\) specifies the model and environment versions,
scorer, initial state, sampling policy, resource ceilings, and
infrastructure-retry rules. Each task receives \(K_{\rm att}\ge1\) scored
attempts, with \([K_{\rm att}]=\{1,\ldots,K_{\rm att}\}\). For task \(i\), attempt \(k\),
and harness \(h\), the Evaluator records score
\(s_{i,k}(h)\in[0,1]\), execution trajectory \(\tau_{i,k}(h)\),
and metadata \(\mathrm{meta}_{i,k}(h)\), including validity, cost, and activation
events. For any nonempty task set \(\mathcal D\), define
\begin{align}
    \bar s_i(h)&=\frac{1}{K_{\rm att}}\sum_{k=1}^{K_{\rm att}}s_{i,k}(h),
    &\operatorname{Util}(h;\mathcal D)&=\frac{1}{|\mathcal D|}
                         \sum_{i\in\mathcal D}\bar s_i(h),
        \label{eq:evolution-utility}\\
    \mathsf{Log}(h;\mathcal D)
      &=\bigl((s_{i,k}(h),\tau_{i,k}(h),\mathrm{meta}_{i,k}(h))\bigr)
                _{i\in\mathcal D,\,k\in[K_{\rm att}]}.
        \label{eq:evolution-ledger}
\end{align}
The log is an indexed collection, so identical outcomes do not
collapse into a single record. \(\bar s_i\) averages attempts within
a task and \(\operatorname{Util}\) weights tasks equally. For binary scores, this
estimates mean per-attempt success, not the probability of succeeding
at least once in \(K_{\rm att}\) attempts. The search maximizes training utility
over \(\mathfrak H(\mathcal X_{\rm edit})\) under a finite development budget.
The returned harness maximizes recorded utility among the baseline and
archived candidates, without a guarantee of global optimality.

Repeated evaluations produce distinct records. Let \(\mathcal D_r\)
denote the screening subset of \(\mathcal D_{\rm tr}\) used at round
\(r\ge0\). Unmarked quantities
refer to the designated full-training log (or the final test log
when \(\mathcal D=\mathcal D_{\rm te}\)). Superscript \(\mathrm{sc}\)
denotes a fresh screening evaluation, for example
\(s^{\rm sc}_{i,k}(h)\), \(\mathsf{Log}^{\rm sc}(h;\mathcal D_r)\),
and \(\operatorname{Util}^{\rm sc}(h;\mathcal D_r)\), defined by the same formulas.
The screening and confirmation scores for an overlapping task need
not coincide. Only the full-training record enters archive ranking.

The operation \(\operatorname{Eval}_{\Pi_{\rm eval}}(h,\mathcal D)\) produces
this log under the fixed protocol. Agent failures remain scored
outcomes and cannot be removed from its \(|\mathcal D|K_{\rm att}\) denominator.
Infrastructure failures follow the declared repair-and-retry rule.
An evaluation that remains invalid under the protocol is excluded from
selection. Score comparisons use only complete, valid logs.
The procedure below retains the recorded full-training log of
each harness and reuses its restriction to a screening subset. This
requires a stable evaluation environment. If the environment changes,
compared candidates and incumbents require a common reevaluation
protocol. Utilities measured under different environments are not ranked together.

\paragraph{From a Failure Trace to an Intervention.}
Using these evaluation records, the Evolver Agent converts the parent's log into a testable
intervention hypothesis. A diagnosis identifies the observed symptom,
its supporting trace events, the edit location, and the behavior
expected to change. For example, an empty final response after
reasoning-budget exhaustion motivates a recovery intervention,
while an untested code change motivates a finalization check. The diagnosis must connect the observed failure to a mechanism that the
Evolver Agent can modify.

Each proposal contains an executable patch, a semantic descriptor,
and a declared activation specification. Fixed before candidate evaluation, the activation specification identifies
the recorded event used to determine whether the targeted mechanism executed. The Evaluator
records activation and scores independently of the Evolver Agent's
explanation. The proposer may use a different model from the frozen
Task Agent. Neither proposal generation nor patch application updates
the Task Agent's weights.

\paragraph{Development and Deployment Cost.}
Development includes diagnosis, proposal generation, screening, and
confirmation. Deployment cost measures only execution of the selected
harness. Both candidate and parent operate under the same externally
fixed resource ceilings in \(\Pi_{\rm eval}\), although configuration edits may change
how that allowance is used. The utility in
\cref{eq:evolution-utility} contains no implicit cost penalty, so
token use and latency must be measured separately. The reference
experiments use \(K_{\rm att}=3\). Screening size \(n\), proposal cap \(J\),
survivor cap \(N_{\rm keep}\le J\), round cap \(R_{\rm round}\), and stagnation patience
\(P\) are positive integer run settings, with \(2\le n\le|\mathcal D_{\rm tr}|\).

\subsection{Harness Gene Bank and Recombination}
\label{sec:evolver-bank}

\paragraph{Semantic Archive.}
To support reuse across evolution rounds, the Harness Gene Bank retains evaluated candidates associated with
different failure modes, including candidates below the current best utility.
Let \(\mathcal W_{\rm edit}=\)
\(\{\mathsf{prompt},\mathsf{knowledge},\mathsf{runtime},\mathsf{config}\}\)
and let \(\mathcal Y_{\rm fail}\) contain admissible pathology labels.
At round \(r\ge0\), the bank is
\begin{equation}
    \mathcal B_r:
      \mathcal W_{\rm edit}\times\mathcal Y_{\rm fail}
      \longrightarrow\mathfrak H(\mathcal X_{\rm edit})\cup\{\bot\},
    \qquad \mathcal B_0(\mathsf{cell})=\bot.
    \label{eq:evolution-bank}
\end{equation}
A cell \(\mathsf{cell}=(w,y)\) records the edit category and the diagnosed failure
that motivates the modification. The observed vocabulary of \(y\) can
grow, and newly introduced cells start empty, holding \(\bot\). These edit categories are
distinct from the four strategy roles in
\cref{eq:harness-factorization}. For example, modifying a Memory
strategy can be a runtime edit.

Each identified candidate carries a fixed descriptor \(\operatorname{desc}(h)=(w,y)\).
The edit category is checked against the patch, while the pathology
is inferred from the diagnosis. For a composite intervention, its
primary cell and donor lineages are recorded before evaluation.
Labels are not changed in response to scores. Thus, \(\operatorname{desc}\) records the intervention's assigned category
and diagnosis. It is not uniquely determined by the executable content. The bank stores the complete
harness and its evaluated log. One elite per cell implements
quality-diversity search over semantic niches
\citep{mouret2015illuminating}. Pathology labels remain hypotheses,
not established causes.

\paragraph{Parent Selection and Candidate Construction.}
Define the set of stored harnesses
\(\operatorname{Im}^{+}(\mathcal B_r)
=\{\mathcal B_r(\mathsf{cell}):\mathcal B_r(\mathsf{cell})\ne\bot\}\).
The parent \(p_r\) is
\begin{equation}
    p_r=\operatorname{Best}
       \bigl(\{h_0\}\cup\operatorname{Im}^{+}(\mathcal B_r)\bigr),
    \qquad
    \operatorname{Best}(S)\in\arg\max_{h\in S}
       \operatorname{Util}(h;\mathcal D_{\rm tr}).
    \label{eq:evolution-parent}
\end{equation}
For every nonempty finite \(S\), \(\operatorname{Best}\) resolves
equal utility by a fixed creation order, with \(h_0\) first.
All candidates in this comparison have valid cached training logs.
The empty initial bank therefore selects \(h_0\).

Given the parent, its diagnostic evidence, and the bank, proposal
generation returns an ordered list
\begin{equation}
    \mathsf{Proposals}_r=\operatorname{Evolve}
       \bigl(p_r,\mathsf{Log}(p_r;\mathcal D_{\rm tr}),\mathcal B_r;J\bigr)
       =(\mathsf{prop}_1,\ldots,\mathsf{prop}_{J_r}),\qquad 0\le J_r\le J,
    \label{eq:evolution-proposals}
\end{equation}
where \(\mathsf{prop}_j=(\Delta^{\rm patch}_j,\mathsf{cell}_j,\beta_j)\) contains a patch, cell, and
activation specification. The materialized candidate is
\(h'_j=\operatorname{Apply}(p_r,\Delta^{\rm patch}_j)\), with
\(\operatorname{desc}(h'_j)=\mathsf{cell}_j\). Invalid patches are rejected before task rollouts.
Proposals either synthesize an intervention from the diagnosis or recombine
changes from archived donor lineages with the parent.
Donor provenance is retained and overlapping edits must be resolved
before materialization. A recovery rule and a final-verification rule,
for example, can act at different stages. The combined harness requires its own evaluation because the components
may interact and their gains need not be additive. Both proposal types follow the same screening procedure.

\paragraph{Competitive Archive Update.}
Screening considers proposals in their recorded order and retains
the first \(N_{\rm keep}\) that pass. Each survivor is evaluated on all training
tasks. Let \(\mathsf{Confirmed}_r\) contain the survivors whose full-training
logs are protocol-valid. The competitors for cell \(\mathsf{cell}\) are
\begin{equation}
    \mathsf{Comp}_r(\mathsf{cell})=
       \bigl(\{\mathcal B_r(\mathsf{cell})\}\setminus\{\bot\}\bigr)
       \cup\{h\in\mathsf{Confirmed}_r:\operatorname{desc}(h)=\mathsf{cell}\}.
    \label{eq:evolution-cell-competitors}
\end{equation}
The update is
\begin{equation}
    \mathcal B_{r+1}(\mathsf{cell})=
    \begin{cases}
        \operatorname{Best}_{\mathsf{cell}}(\mathsf{Comp}_r(\mathsf{cell})),
                    &\mathsf{Comp}_r(\mathsf{cell})\ne\varnothing,\\
        \bot,&\mathsf{Comp}_r(\mathsf{cell})=\varnothing.
    \end{cases}
    \label{eq:evolution-bank-update}
\end{equation}
Here \(\operatorname{Best}_{\mathsf{cell}}\) maximizes full-training utility, retains
the incumbent on a tie, and otherwise retains the earliest candidate
in proposal order among tied maxima. Equivalently, process confirmed
candidates sequentially and replace only an empty cell or a strictly
weaker incumbent. This resolves competition when several candidates
target the same cell.

Archive admission and parent selection use different comparisons.
A candidate can outperform its parent on the screening subset yet
fall below it on the full training set. It may still fill a new cell
or replace a weaker cell incumbent, but becomes a parent only if
\cref{eq:evolution-parent} selects it. With fixed cached utilities,
occupied-cell utility and the best available utility cannot decrease.
However, neither property guarantees improvement on unseen tasks.
HarnessDev reports this gap empirically: harness gains obtained by
evolution can be unstable, transfer only partially to held-out tasks, and
depend strongly on the model that executes the harness
\citep{wu2026harnessdev}.

\subsection{Gated Harness Screening and Selection}
\label{sec:evolver-screening}

\paragraph{Validity and Activation.}
At round \(r\), sample \(\mathcal D_r\subseteq\mathcal D_{\rm tr}\)
uniformly without replacement, with \(|\mathcal D_r|=n\), and use it
for every candidate-parent comparison in that round. Record the
sample and proposal order. The parent's measurements are the
corresponding entries of its full-training log. Each candidate
receives \(K_{\rm att}\) attempts per sampled task under \(\Pi_{\rm eval}\).
Let \(\operatorname{ValidLog}(\mathsf{Log})\in\{0,1\}\) indicate a complete protocol-valid
log. The gates are
\begin{align}
    g_{\rm valid}(h';\mathcal D_r)
       &=\operatorname{ValidLog}\bigl(\mathsf{Log}^{\rm sc}(h';\mathcal D_r)\bigr),\\
    g_{\rm act}(h';\mathcal D_r)
       &=\mathbb I\!\left[
           \sum_{i\in\mathcal D_r}\sum_{k=1}^{K_{\rm att}}\mathrm{active}_{i,k}(h')>0
          \right].
    \label{eq:evolution-operational-gates}
\end{align}
The activation indicator is
\(\mathrm{active}_{i,k}(h')=\beta
(\tau^{\rm sc}_{i,k}(h'),\mathrm{meta}^{\rm sc}_{i,k}(h'))\in\{0,1\}\),
where \(\beta\) is the declared activation predicate of the proposal that
produced \(h'\).
It detects an instrumented intervention in
the candidate's increment over \(h_0\), such as execution of a
recovery branch or actual injection of a new instruction. The predicate must exclude unrelated log events. Recombined mechanisms retain
their provenance and separate event records. The existential gate
establishes that an instrumented intervention ran at least once.
It neither requires every donor mechanism to execute nor establishes the
causal contribution of the newest edit. The subsequent utility
comparison includes \emph{all} sampled tasks, including those without
activation.

\paragraph{Paired Improvement Screening.}
Validity and activation alone do not establish improvement.
For a protocol-valid candidate \(h'\) and parent \(p_r\), define one
paired difference per task:
\begin{align}
    \gamma_i&=\bar s_i^{\rm sc}(h')-\bar s_i(p_r),
        &\bar\gamma_r&=\frac{1}{n}
                         \sum_{i\in\mathcal D_r}\gamma_i,
        \label{eq:evolution-paired-gain}\\
    \widehat\sigma_\gamma^2
       &=\frac{1}{n-1}\sum_{i\in\mathcal D_r}
                         (\gamma_i-\bar\gamma_r)^2,
        &\widehat{\operatorname{se}}_r
       &=\frac{\widehat\sigma_\gamma}{\sqrt n}.
        \label{eq:evolution-paired-variance}
\end{align}
Here \(\bar\gamma_r=\operatorname{Util}^{\rm sc}(h';\mathcal D_r)-\operatorname{Util}(p_r;\mathcal D_r)\)
is mean gain, \(\widehat\sigma_\gamma^2\) is the sample variance of
task-level gains, and \(\widehat{\operatorname{se}}_r\) is their
estimated standard error. The notation suppresses dependence on the candidate for readability. Averaging the \(K_{\rm att}\) attempts first keeps the task as
the measurement unit, so the sample size is \(n\), not \(nK_{\rm att}\).
Pairing matches task identity and does not assume that attempt
indices share identical random draws.

Define the paired improvement statistic and its operational boundary
cases by
\begin{equation}
    t_{\rm pair}=\begin{cases}
        \bar\gamma_r/\widehat{\operatorname{se}}_r,
            &\widehat{\operatorname{se}}_r>0,\\
        +\infty,&\widehat{\operatorname{se}}_r=0,
                              \ \bar\gamma_r>0,\\
        -\infty,&\widehat{\operatorname{se}}_r=0,
                              \ \bar\gamma_r<0,\\
        0,&\widehat{\operatorname{se}}_r=0,
                              \ \bar\gamma_r=0.
    \end{cases}
    \label{eq:evolution-paired-statistic}
\end{equation}
The nondegenerate expression has the form of a paired Student
\(t\)-statistic, since its denominator uses an estimated variance.
The screening rule uses the fixed standard-normal cutoff \(\lambda_{\rm screen}=1.96\)
without interpreting it as an exact
finite-sample two-sided 5\% test. A Student \(t\) interpretation
requires assumptions on the task differences and a critical value
with \(n-1\) degrees of freedom. Zero-variance
cases follow the explicit convention above and are flagged in the
log. No calibrated \(p\)-value is assigned to an infinite statistic.

For a declared \(\lambda_{\rm screen}>0\), the paired-gain screen is
\begin{equation}
    g_{\rm sig}(h',p_r;\mathcal D_r)
       =\mathbb I[\bar\gamma_r>0]\,
        \mathbb I[t_{\rm pair}\ge\lambda_{\rm screen}].
    \label{eq:evolution-significance}
\end{equation}
The composite decision rejects invalid or inactive candidates before
assessing paired gain:
\begin{equation}
    g_{\rm screen}(h',p_r;\mathcal D_r)=
    \begin{cases}
        g_{\rm sig}(h',p_r;\mathcal D_r),
          &g_{\rm valid}=g_{\rm act}=1\ \text{and }n\ge2,\\
        0,&\text{otherwise}.
    \end{cases}
    \label{eq:evolution-screening}
\end{equation}
Both operational gates in this expression are evaluated on
\((h',\mathcal D_r)\). The parent log must also be valid.
An invalid log never reaches the arithmetic in
\cref{eq:evolution-paired-gain,eq:evolution-paired-variance}.
\Cref{alg:evolution-screen} makes the rejection order explicit.
In the algorithms, \(\mathsf{Log}^{\rm sc}_{h'}\) abbreviates the fresh
screening record \(\mathsf{Log}^{\rm sc}(h';\mathcal D)\) for the supplied
subset, and \(\mathsf{Log}_{h}\) abbreviates the cached record
\(\mathsf{Log}(h;\mathcal D_{\rm tr})\). The supplied proposal's activation
predicate is \(\beta\). The screening algorithm suppresses the round
subscript on its gain statistic because it handles one comparison at a time.
These equations define the candidate screening procedure.
\Cref{app:harness-screening-policy} describes the experiment source and
the conventions used to complete the algorithm specification.

\begin{ravenalgorithm}{Candidate Screening}
\label{alg:evolution-screen}
\raggedright
\textbf{Input:} proposal \(\mathsf{prop}=(\Delta^{\rm patch},\mathsf{cell},\beta)\), parent \(p\)
with valid training log, subset \(\mathcal D\), protocol \(\Pi_{\rm eval}\),
admissible space \(\mathfrak H(\mathcal X_{\rm edit})\), threshold \(\lambda_{\rm screen}\).\\
\textbf{Output:} candidate \(h'\) and screening log, or
rejection \(\bot\) with a recorded reason.
\begin{tabularx}{\linewidth}{@{}r@{\quad}X@{}}
1 & \textbf{if} \(|\mathcal D|<2\), \textbf{return} \(\bot\)
    (insufficient tasks).\\
2 & Materialize \(h'\leftarrow\operatorname{Apply}(p,\Delta^{\rm patch})\) and
    attach \(\operatorname{desc}(h')=\mathsf{cell}\), predicate \(\beta\), and lineage.\\
3 & \textbf{if} application fails, \(h'\notin\mathfrak H(\mathcal X_{\rm edit})\),
    or its descriptor/activation contract is invalid,
    \textbf{return} \(\bot\) (inadmissible proposal).\\
4 & \(\mathsf{Log}^{\rm sc}_{h'}\leftarrow\operatorname{Eval}_{\Pi_{\rm eval}}(h',\mathcal D)\),
    including the declared bounded infrastructure retries.\\
5 & \textbf{if} \(\operatorname{ValidLog}(\mathsf{Log}^{\rm sc}_{h'})=0\), \textbf{return} \(\bot\)
    (invalid measurement) without computing a gain.\\
6 & \textbf{if} \(\sum_{i\in\mathcal D}\sum_{k=1}^{K_{\rm att}}
    \beta(\tau^{\rm sc}_{i,k}(h'),\mathrm{meta}^{\rm sc}_{i,k}(h'))=0\), \textbf{return} \(\bot\)
    (inactive intervention).\\
7 & Compute \(\gamma_i,\bar\gamma,\widehat\sigma_\gamma^2\)
    against the parent's restricted log using
    \cref{eq:evolution-paired-gain,eq:evolution-paired-variance}.\\
8 & Compute \(t_{\rm pair}\) by \cref{eq:evolution-paired-statistic} and
    record whether the variance is zero.\\
9 & \textbf{if} \(\bar\gamma\le0\) or
    \(t_{\rm pair}<\lambda_{\rm screen}\), \textbf{return} \(\bot\)
    (insufficient paired gain).\\
10 & \textbf{return} \((h',\mathsf{Log}^{\rm sc}_{h'})\).\\
\end{tabularx}
\end{ravenalgorithm}

\paragraph{Complete Search and Final Selection.}
Combining the screening and archive rules, \cref{alg:evolution} holds the parent and archive fixed while a round's
proposals are screened. Only complete full-training evaluations can
update the bank. If no proposal passes, the bank remains unchanged
and the stagnation counter advances. With \(R_{\rm round}\) rounds, at most
\(J\) proposals and \(N_{\rm keep}\) confirmations per round, the procedure
uses at most
\(K_{\rm att}[(1+R_{\rm round}N_{\rm keep})|\mathcal D_{\rm tr}|+R_{\rm round}Jn]\) scored task attempts,
excluding infrastructure retries and proposer computation.
An additional token or time ceiling can stop development early.
An unfinished evaluation is never used for selection.
In the algorithm, \(\mathsf{Survivors}_r\) is the ordered list of screening survivors
and \(\mathrm{stale}\) counts consecutive rounds without a cell update.

\begin{ravenalgorithm}{Semantic Gene-Bank Search}
\label{alg:evolution}
\textbf{Input:} valid \(h_0\), admissible space \(\mathfrak H(\mathcal X_{\rm edit})\),
training set \(\mathcal D_{\rm tr}\), protocol \(\Pi_{\rm eval}\) with \(K_{\rm att}\)
attempts, settings \((n,J,N_{\rm keep},R_{\rm round},P,\lambda_{\rm screen})\).\\
\textbf{Output:} training-selected harness \(\widehat h\) and its
cached evaluation record. Test tasks are not inputs to the search.
\begin{tabularx}{\linewidth}{@{}r@{\quad}X@{}}
1 & \(\mathsf{Log}_{h_0}\leftarrow\operatorname{Eval}_{\Pi_{\rm eval}}(h_0,\mathcal D_{\rm tr})\).
    \textbf{If} \(\operatorname{ValidLog}(\mathsf{Log}_{h_0})=0\), terminate with an evaluation error.\\
2 & Cache \(\mathsf{Log}_{h_0},\operatorname{Util}(h_0;\mathcal D_{\rm tr})\).
    Initialize \(\mathcal B_0\leftarrow\bot\) in every cell and set
    \(r\leftarrow0,\ \mathrm{stale}\leftarrow0\).\\
3 & \textbf{while} \(r<R_{\rm round}\) and \(\mathrm{stale}<P\):\\
4 & \quad \(p_r\leftarrow\operatorname{Best}
    (\{h_0\}\cup\operatorname{Im}^{+}(\mathcal B_r))\).\\
5 & \quad Generate ordered \(\mathsf{Proposals}_r\) by \cref{eq:evolution-proposals} and
    sample \(\mathcal D_r\) of size \(n\) without replacement.\\
6 & \quad \(\mathsf{Survivors}_r\leftarrow[\,]\), \(\mathsf{Confirmed}_r\leftarrow\varnothing\).\\
7 & \quad \textbf{for each} \(\mathsf{prop}_j\) in \(\mathsf{Proposals}_r\), in order:\\
8 & \qquad \textbf{if} \(|\mathsf{Survivors}_r|=N_{\rm keep}\), \textbf{break}.\\
9 & \qquad Run \cref{alg:evolution-screen} with
    \((\mathsf{prop}_j,p_r,\mathcal D_r,\Pi_{\rm eval},\mathfrak H(\mathcal X_{\rm edit}),\lambda_{\rm screen})\).\\
10 & \qquad \textbf{if} accepted, append the returned candidate to \(\mathsf{Survivors}_r\).\\
11 & \quad \textbf{for each} \(h'\) in \(\mathsf{Survivors}_r\), in order:\\
12 & \qquad \(\mathsf{Log}_{h'}\leftarrow
    \operatorname{Eval}_{\Pi_{\rm eval}}(h',\mathcal D_{\rm tr})\).\\
13 & \qquad \textbf{if} \(\operatorname{ValidLog}(\mathsf{Log}_{h'})=1\): cache its log and
    utility, and set \(\mathsf{Confirmed}_r\leftarrow\mathsf{Confirmed}_r\cup\{h'\}\).\\
14 & \quad Update \(\mathcal B_{r+1}\) by
    \cref{eq:evolution-cell-competitors,eq:evolution-bank-update}.\\
15 & \quad \(\mathrm{stale}\leftarrow0\) if any cell changes,
    otherwise \(\mathrm{stale}\leftarrow\mathrm{stale}+1\).\\
16 & \quad \(r\leftarrow r+1\).\\
17 & \(\widehat h\leftarrow\operatorname{Best}
    (\{h_0\}\cup\operatorname{Im}^{+}(\mathcal B_r))\).
    \textbf{Return} \(\widehat h\) and its cached record.\\
\end{tabularx}
\end{ravenalgorithm}

After selection, freeze \(\widehat h\) and evaluate it and \(h_0\)
under the same held-out protocol. The reported difference is
\(\operatorname{Util}(\widehat h;\mathcal D_{\rm te})-\operatorname{Util}(h_0;\mathcal D_{\rm te})\).
Test outcomes do not revise proposals, thresholds, archive cells, or
the selected harness. Repeated adaptive screening on training tasks guides the search but does
not provide a family-wise significance guarantee.
Activation evidence, archive competition, and the held-out comparison
address separate questions: whether the intervention executed, how the
candidate ranked on training tasks, and how the selected harness performed
on unseen tasks.

\section{Memory System}\label{sec:memory}

An agent's interaction history contains information that can remain useful
beyond the active context, including user constraints, evidence supporting
earlier decisions, and procedures learned during execution.
To preserve this information across turns and sessions, Raven uses
EverOS \citep{hu2026evermemos,evermind2026everos}. Its memory lifecycle separates episodic
trace formation, semantic consolidation, and query-dependent recollection,
with distinct user and agent memory tracks and asynchronous persistence.
Raven connects user memory to context
assembly and agent experience to Skill Forge. We first describe how
memories acquire structure, then explain their retrieval and use in
Raven. \Cref{fig:method-memory} summarizes the formation and consolidation
process, while \cref{fig:memory-retrieval} shows the retrieval path used
for Raven's user memory.
Implementation defaults are collected in \cref{app:method-defaults}.

\subsection{Episodic Trace Formation}
\label{sec:memory-formation}

A useful memory must preserve enough context to interpret an event
without replaying the entire conversation. Let
\(\mathsf{Hist}=(\mathsf{msg}_1,\ldots,\mathsf{msg}_{T_{\rm msg}})\) be an ordered interaction stream of
\(T_{\rm msg}\) messages. Each message records its content, sender, role,
timestamp, and any associated tool interaction. A boundary detector
groups contiguous messages into segments
\(\mathsf{seg}_i=(\mathsf{msg}_{\mathrm{end}_{i-1}+1},\ldots,\mathsf{msg}_{\mathrm{end}_i})\), where \(i\) indexes
segments, \(\mathrm{end}_i\) is the last message index of segment \(i\), and
\(\mathrm{end}_0=0\). An LLM examines a sliding context
window and closes a segment when it detects a semantic boundary such as
a topic change. Grouping related exchanges preserves the context needed to resolve
references and explain decisions. Messages after the last
accepted boundary remain buffered until a later boundary or an explicit
flush.

We describe the information associated with segment \(i\) as a
\emph{MemCell}:
\begin{equation}
    \mathsf{Cell}_i=(\mathsf{ep}_i,\mathcal F_i,\mathcal P_i,\mathrm{meta}_i).
    \label{eq:memory-cell}
\end{equation}
Here \(\mathsf{ep}_i\) is an episode narrative, \(\mathcal F_i\) is a set of
atomic facts, \(\mathcal P_i\) is a set of forward-looking statements,
and \(\mathrm{meta}_i\) is grounding metadata. The narrative rewrites the
interaction in the third person, resolves coreferences, and removes
conversational repetition while retaining the event and its context.
Atomic facts split that narrative into statements that can be matched
individually to a later query. For example, a discussion about setting up
a project can yield a narrative explaining the environment decision and
a fact stating the selected Python version. The narrative explains the decision, while the fact provides a precise
retrieval target.

Foresight represents information whose usefulness depends on time.
Write an element of \(\mathcal P_i\) as \((p,t_{p,0},t_{p,1})\), where \(p\)
is a plan or inferred temporary state, and \([t_{p,0},t_{p,1}]\) is its
estimated validity interval. At time \(t_{\rm now}\), the time-valid subset is
\begin{equation}
    \mathcal P_i(t_{\rm now})=
    \{(p,t_{p,0},t_{p,1})\in\mathcal P_i: t_{p,0}\le t_{\rm now}\le t_{p,1}\}.
    \label{eq:memory-validity}
\end{equation}
An upcoming deadline, for instance, has a different temporal status from
a durable preference. The validity interval allows expired predictions to be
filtered out. EverOS can derive foresight from the episode or the shared
source segment.

The metadata \(\mathrm{meta}_i\) connects the extracted representations to their
origin. It includes time and source identifiers, together with session
and ownership information. EverOS separates application and project
scope, then partitions records by owner type and owner identifier. User
records belong to a user, while cases and skills belong to an agent.
A session identifier connects records from one interaction within these
partitions. Thus, a session describes where an experience occurred,
while ownership determines the collection in which it can be recalled.
Raven's ordinary recall names exactly one user or agent identifier and
uses the backend's default application and project scope.

In EverOS, \cref{eq:memory-cell} describes a logical association among
representations that are stored separately.
A shared boundary stage creates the source segment and its identifier.
The user pipeline synthesizes and writes the episode, after which
background strategies extract facts and update the higher-level
representations. Facts retain a parent link to the episode. In the
default agent mode, the same source segment also enters the agent
pipeline, preserving tool calls and their results for case extraction.
This shared origin connects what the user discussed with what the agent
attempted. Raven submits completed turns and flushes the remaining
buffer, so its turn boundaries also influence memory granularity.

\begin{figure}[t]
    \centering
    \includegraphics[width=\linewidth]{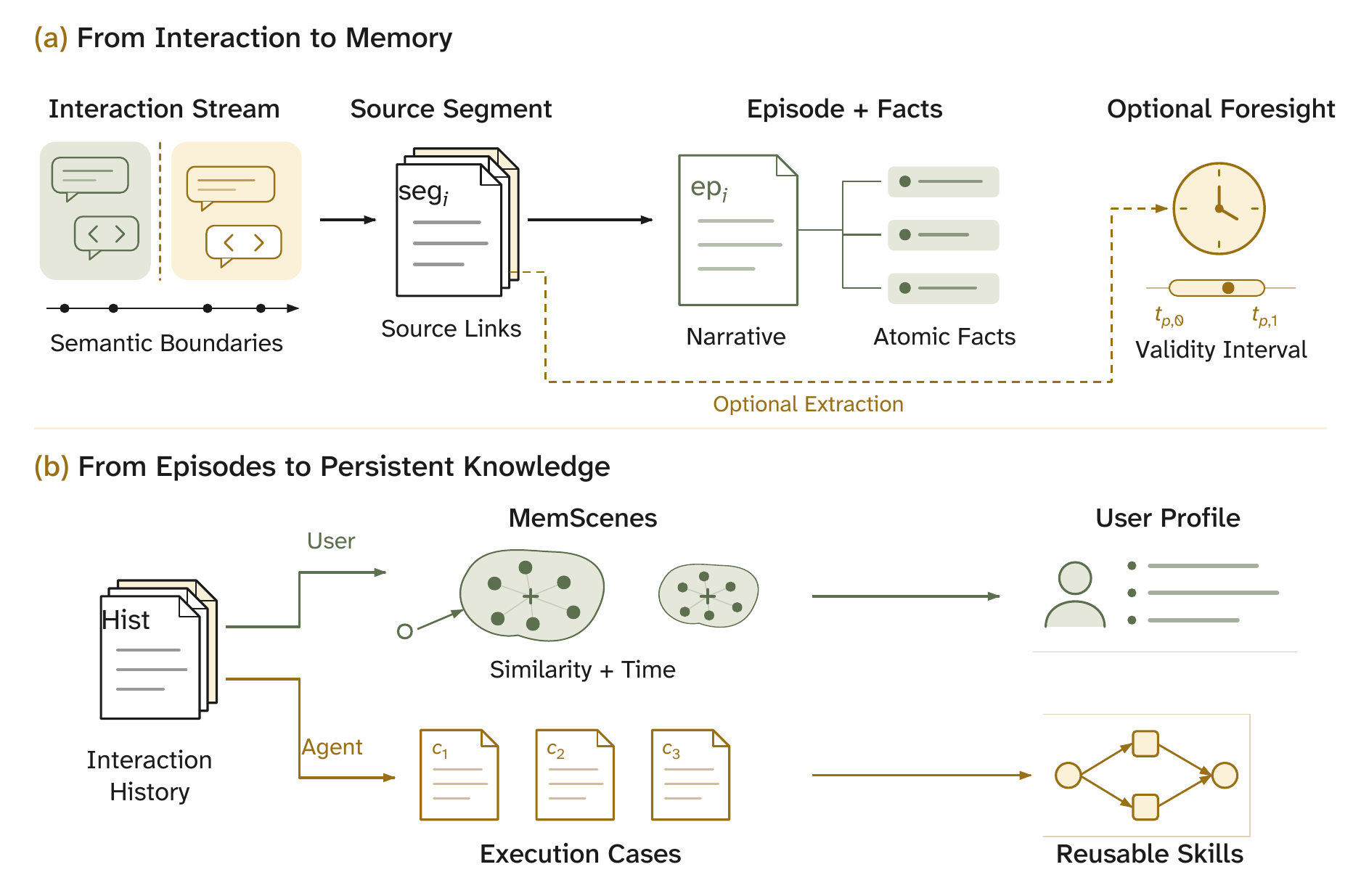}
    \caption{Memory formation and consolidation in EverOS.
    (a) Semantic boundaries group dialogue and tool interactions into a
    source segment \(\mathsf{seg}_i\). Its episode \(\mathsf{ep}_i\) anchors atomic
    facts. The dashed branch derives optional foresight with validity
    interval \([t_{p,0},t_{p,1}]\).
    (b) Across interaction history \(\mathsf{Hist}\), episodes, shown as
    points, form MemScenes around centroids marked by crosses.
    These scenes select evidence for profile updates.
    The agent track consolidates execution cases \(c_1,c_2,c_3\) into
    reusable procedures, depicted by a small dependency graph.
    The diagram shows representation and data dependencies.
    Persistence and indexing proceed separately.}
    \label{fig:method-memory}
\end{figure}

\subsection{Semantic Consolidation}
\label{sec:memory-consolidation}

Individual episodes describe local events. To identify recurring topics
and persistent attributes across interactions, consolidation groups related
events into a \emph{MemScene}. EverOS implements the user-side grouping as online
clustering of episode embeddings. Let \(\mathbf{z}_i=\operatorname{Enc}(\mathsf{ep}_i)\in\mathbb R^{d_{\rm emb}}\)
be the embedding of episode \(i\), where \(\operatorname{Enc}\) is the embedding
model and \(d_{\rm emb}\) its output dimension. A scene \(\mathsf{scene}_j\) stores its
member episode identifiers, centroid \(\boldsymbol{\mu}_j\in\mathbb R^{d_{\rm emb}}\), member
count \(n_j\), and latest member timestamp \(t_{j,{\rm last}}\). These quantities
summarize the evidence already assigned to that scene.

Assignment combines semantic similarity with temporal proximity. For
episode timestamp \(t_i\), let \(\mathcal J_{\rm valid}\) index existing
scenes with nonempty centroids of the embedding dimension. Define the
eligible scene indices and the
similarity score as
\begin{equation}
    \begin{aligned}
    \mathcal J_i&=\{j\in\mathcal J_{\rm valid}:
                       |t_i-t_{j,{\rm last}}|\le\Delta_{\rm time}\},
    \\
    \operatorname{sim}(\mathbf{z}_i,\boldsymbol{\mu}_j)&=\frac{\mathbf{z}_i^\top\boldsymbol{\mu}_j}
      {(\|\mathbf{z}_i\|_2+\epsilon_{\rm num})(\|\boldsymbol{\mu}_j\|_2+\epsilon_{\rm num})}.
    \end{aligned}
    \label{eq:memory-scene-similarity}
\end{equation}
Here \(\Delta_{\rm time}\) is the allowed time gap and \(\epsilon_{\rm num}=10^{-9}\)
prevents division by zero. Scenes with no centroid are excluded. When
\(\mathcal J_i\) is nonempty, the highest-scoring eligible scene is
\(j^*\in\arg\max_{j\in\mathcal J_i}\operatorname{sim}(\mathbf{z}_i,\boldsymbol{\mu}_j)\), with ties resolved by
the first scene encountered in the backend's candidate order. Episode \(i\) joins
that scene if \(\operatorname{sim}(\mathbf{z}_i,\boldsymbol{\mu}_{j^*})\ge\theta_{\rm sim}\), where \(\theta_{\rm sim}\) is the
similarity threshold. Otherwise, a new singleton scene is created with centroid \(\mathbf{z}_i\),
count one, and timestamp \(t_i\). For an accepted assignment, the updates are
\begin{equation}
    \boldsymbol{\mu}_{j^*}^{+}=\frac{n_{j^*}\boldsymbol{\mu}_{j^*}+\mathbf{z}_i}{n_{j^*}+1},
    \qquad n_{j^*}^{+}=n_{j^*}+1,
    \qquad t_{j^*,{\rm last}}^{+}=\max(t_{j^*,{\rm last}},t_i).
    \label{eq:memory-scene-update}
\end{equation}
The superscript \(+\) denotes the state after insertion. The member
identifier is appended as well. The time window constrains scene assignment without limiting episode
retention or searchability. Consequently, a recurring topic can occupy several scenes
when its occurrences are separated in time.

Profile evolution uses the resulting groups to select evidence for a
compact account of the user. Let \(\mathsf{Prof}_u\) denote the current profile
of user \(u\), with explicit attributes and inferred traits. Explicit
attributes include stated preferences and goals, while inferred traits
summarize patterns across interactions. EverOS uses recently updated
scenes to select the underlying source segments for a profile update. If
\(t_{\mathsf{Prof}_u}\) is the previous profile timestamp and \(j_{\rm cur}\)
is the scene updated by the current insertion, this selection is
\begin{equation}
    \mathcal J_u^{+}=\{j:t_{j,{\rm last}}>t_{\mathsf{Prof}_u}\}\cup\{j_{\rm cur}\},
    \qquad
    \mathsf{Prof}_u^{+}=\operatorname{UpdateProfile}
        (\mathsf{Prof}_u,\operatorname{Segments}(\mathcal J_u^{+})).
    \label{eq:memory-profile-update}
\end{equation}
All scenes in this expression belong to user \(u\) in the selected
application and project. The operator \(\operatorname{Segments}(\mathcal J)\)
collects the available source segments linked to scenes indexed by
\(\mathcal J\), ordered by time. \(\operatorname{UpdateProfile}\)
is an LLM extraction step conditioned on those segments and the previous
profile. With no prior profile, the extractor initializes one from the
available source segments. Including \(j_{\rm cur}\) keeps the current cluster eligible
even when an older interaction is ingested after a newer profile update.
Without embeddings, EverOS can select recent source segments by
timestamp and update the profile without geometric clustering.

Beyond profile updates, EverOS provides optional episode reflection for consolidation at a
longer timescale. Reflection selects a cluster, synthesizes a merged
episode, extracts facts from that merged narrative, and marks the
superseded episodes and facts as deprecated. Online assignment preserves
individual events, whereas reflection can replace fragmented retrieval
units with a consolidated account. This operation is separate from the
ordinary cluster and profile updates described above.

\subsection{Reconstructive Retrieval}
\label{sec:memory-retrieval}

EverOS supports several retrieval methods. Raven's user-memory call uses the
\textsc{hybrid}
method, which searches episodes first and then lets their atomic facts
compete for a bounded result set. Retrieval remains within the selected user partition, following the
search and response flow in \cref{fig:memory-retrieval}.

\begin{figure}[t]
    \centering
    \includegraphics[width=\linewidth]{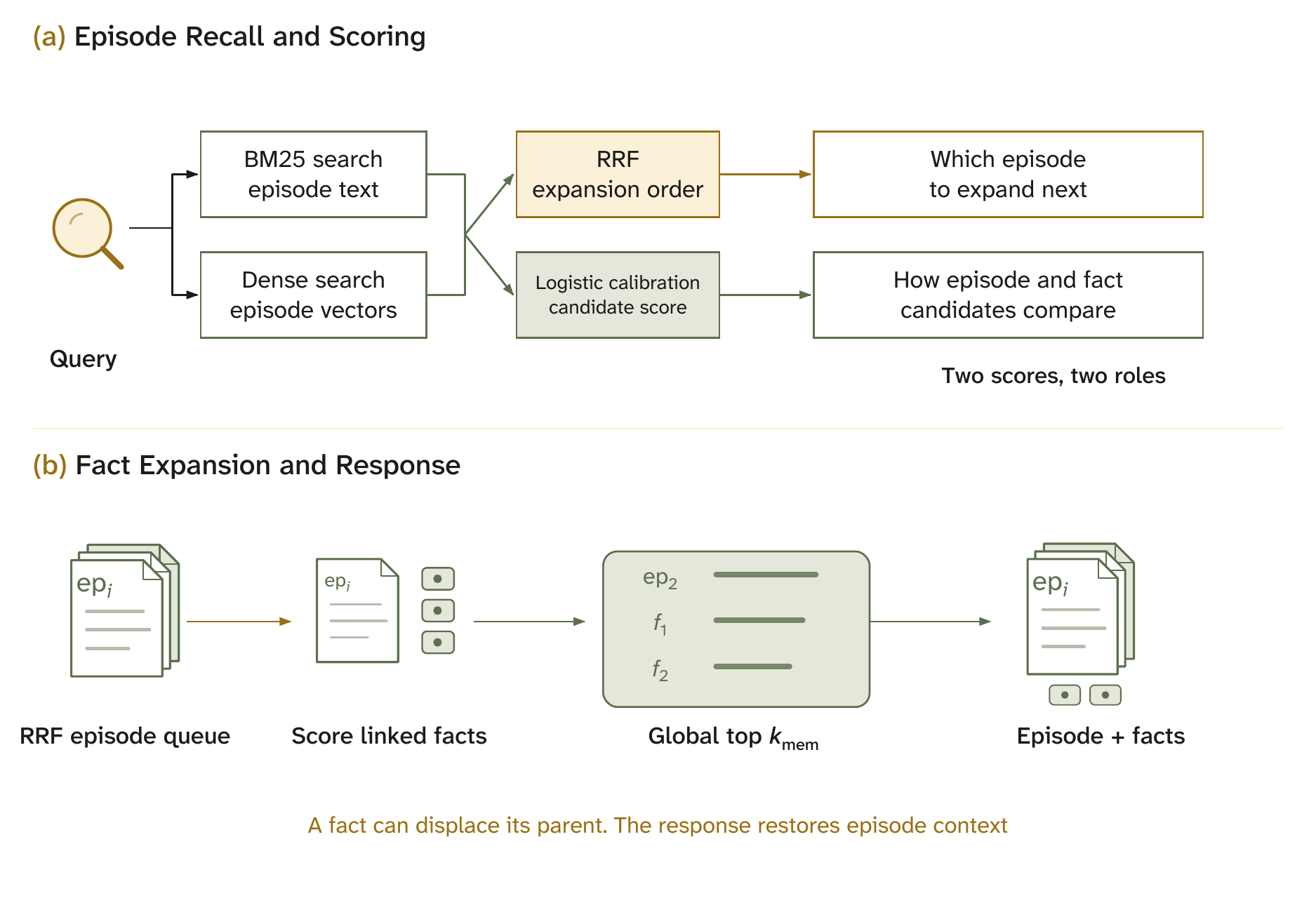}
    \caption{Raven's user-memory retrieval through EverOS \textsc{hybrid}.
    (a) BM25 and dense search retrieve episodes. Reciprocal rank fusion
    (RRF) determines their expansion order. A separate calibrated score
    ranks them for final selection.
    (b) The query scores facts belonging to each expanded episode.
    Episodes and facts compete within a bounded candidate set. A selected
    fact is returned under its parent episode, which restores its context.
    The positive integer \(k_{\rm mem}\) is the backend candidate-set cap.}
    \label{fig:memory-retrieval}
\end{figure}

\paragraph{Episode Recall and Expansion Order.}
Let the positive integer \(k_{\rm mem}\) bound the backend's episode-and-fact
candidate set. For query \(q\), EverOS retrieves episodes through BM25 search over the
episode text and dense search over episode embeddings. Lexical search matches query terms, while dense search can also match
paraphrases. Let
\(L^{\rm hit}_{\ell}(q)\) be the ranked episode list from channel
\(\ell\in\{\mathrm{lex},\mathrm{dense}\}\), with one-based ranks.
RRF gives episode \(\mathsf{ep}_i\) the expansion priority
\begin{equation}
    \operatorname{RRF}_i(q)=
    \sum_{\substack{\ell\in\{\mathrm{lex},\mathrm{dense}\}\\\mathsf{ep}_i\in L^{\rm hit}_\ell(q)}}
    \frac{1}{c_{\rm RRF}^{\rm mem}+\operatorname{rank}_{L^{\rm hit}_\ell(q)}(\mathsf{ep}_i)},
    \label{eq:memory-rrf}
\end{equation}
where \(c_{\rm RRF}^{\rm mem}>0\) controls rank damping. A channel contributes only
when it returns the episode. RRF determines expansion order, while
calibrated scores determine final relevance.

Final selection uses a separate calibrated relevance score. Let
\(\mathrm{dense}_i(q)\) be episode \(\mathsf{ep}_i\)'s dense similarity and \(\mathrm{lex}_i(q)\)
its BM25 score, with zero for a channel that did not recall it.
The backend's logistic calibration has the form
\begin{equation}
    \operatorname{Rel}_i(q)=\operatorname{sigmoid}\bigl(w_{\rm dense}\,\mathrm{dense}_i(q)+w_{\rm lex}\,\mathrm{lex}_i(q)+b_{\rm off}\bigr),
    \qquad \operatorname{sigmoid}(x)=\frac{1}{1+\exp(-x)}.
    \label{eq:memory-hybrid-score}
\end{equation}
Here \(w_{\rm dense}\), \(w_{\rm lex}\), and \(b_{\rm off}\) are fixed backend
coefficients. Calibration puts episode and fact candidates on a common
ranking scale.

\paragraph{Fact Expansion and Response.}
For each recalled episode, EverOS fetches linked atomic facts by parent
identifier and scores them against the query embedding. The initial
BM25 search is over episodes, so a fact \(f\in\mathcal F_i\) uses its
parent's lexical score \(\mathrm{lex}_i(q)\), together with its own dense
similarity \(\mathrm{dense}_f(q)\):
\begin{equation}
    \operatorname{Rel}_f(q)=\alpha_{\rm mix}\,
      \operatorname{sigmoid}\bigl(w_{\rm dense}\,\mathrm{dense}_f(q)+w_{\rm lex}\,\mathrm{lex}_i(q)+b_{\rm off}\bigr)
      +(1-\alpha_{\rm mix})\operatorname{Rel}_i(q).
    \label{eq:memory-fact-score}
\end{equation}
The mixing weight satisfies \(\alpha_{\rm mix}\in[0,1]\).
The default setting is \(\alpha_{\rm mix}=1\), allowing a precise fact
match to outrank its parent episode. The backend initializes a set
of at most \(k_{\rm mem}\) candidates and expands episodes in RRF order.
The best facts from an expanded episode may enter if space remains
or their calibrated score exceeds the lowest retained score.
When one enters, the parent episode is removed from this competition
to avoid spending two slots on the same evidence. Expansion stops
after all candidates are considered or the selected identifiers stay
unchanged for a bounded number of steps.

The response restores episode context after this competition.
Selected facts are grouped beneath their parent episode. If that episode
was removed from the candidate set, the backend recovers it from the
episode recall pool and assigns it the highest score among its retained
facts.
The API therefore returns episode objects with nested
\texttt{atomic\_facts}, possibly fewer than \(k_{\rm mem}\) distinct episodes.
Raven renders episode summaries, falling back to the episode body.

\subsection{User Context Assembly}
\label{sec:memory-archive}

To use the retrieved information during execution, Raven combines recalled
semantic memory with a host-maintained Markdown archive in the workspace.
The archive contains profile sections and
episode notes. Its consolidator appends an episode note when history is
archived and revises the profile when enough new notes share a tag.
For an ordinary turn, the host selects the two profile sections with
the greatest lexical overlap with the incoming message, together with
all Notes sections. This gives the host a compact source
of persistent user context alongside query-dependent backend recall.

Let the positive integer \(k_{\rm recall}\) be the adapter's final result
limit, distinct in role from the backend cap \(k_{\rm mem}\). EverOS user
recall requests episodes and the user's profile. The backend
retrieves episodes by the \textsc{hybrid} method above and fetches the profile
separately. Raven renders the profile as key-value lines within a character
budget. It then places the profile in the same result list as episodes with an
adapter-assigned score of \(1.0\) and retains at most \(k_{\rm recall}\) items.
This score is a placement convention, not a relevance estimate or
confidence in the profile. The adapter's result limit includes any
retained profile, while the server's episode limit excludes it.

Let \(\operatorname{Archive}(q)\) be the selected host archive text and \(\operatorname{Recall}(q)\) the
rendered user recall. Raven assembles the memory segment as
\begin{equation}
    \operatorname{MemContext}(q)=\operatorname{Archive}(q)\mathbin{+\!\!+}\operatorname{Fence}(\operatorname{Recall}(q)),
    \label{eq:memory-context}
\end{equation}
where \(\mathbin{+\!\!+}\) denotes ordered text concatenation and
\(\operatorname{Fence}\) marks recalled content as untrusted contextual
data. This segment can contain both host and EverOS profile text.
If either layer is unavailable, it contributes an empty segment and the
turn continues. When the backend lacks
an embedding model, the adapter requests keyword retrieval.

\subsection{Agent Experience and Cross-Session Reuse}
\label{sec:memory-session}

For procedural reuse, the agent memory track extracts procedures from execution evidence.
Given the source segment produced during formation, a case extractor
records the task intent, attempted approach, reusable insight, and an
estimated quality score. Related cases are grouped by task intent and
used to create or revise skills with links to their supporting cases.
User profiles describe the user's context, whereas cases and skills
capture how an agent approached a class of tasks. Agent recall supplies these records
to Skill Forge, whose admission, clustering, and skill-update rules are
specified in \cref{sec:skillforge-evolution}. They do not enter the
user memory segment in \cref{eq:memory-context}.

Cross-session reuse requires memory content to survive independently of
the active process and its search index. EverOS stores episodes, facts,
profiles, cases, and skills as Markdown records. SQLite maintains
operational state, source-segment payloads, and cluster membership,
while a derived index supports lexical and vector retrieval. File changes
are propagated to the index asynchronously. The backend uses
LanceDB by default, supports Milvus as an alternative, and maintains
approximate nearest-neighbor indexes for sufficiently large vector
collections. These indexing choices change retrieval execution without
changing the memory representations above. The Markdown records can be
used to rebuild the search index. Raven serializes writes within a
session and retries failed stores. Queue limits and retry settings are
recorded in \cref{app:method-defaults}.

For multi-agent execution, Raven also reads the records associated with
a particular worker session. This is a session-filtered read of user
episodes and agent cases.
When a worker uses its own memory backend, the host reads that backend.
When memory is formed from a captured worker trace, the host submits
the trace with the worker's identities and session identifier. The
resulting memory artifacts can be made available to downstream graph
nodes through the paths described in \cref{sec:collaboration-memory}.
Task completion releases execution dependencies while memory recording
continues asynchronously.

\section{Skill Forge}\label{sec:skillforge}

Skill Forge retrieves reusable procedures for an agent and updates them
using execution experience. A skill describes a task class, instructions
for performing it, and any required resources.
The method has three stages: build a curated catalog, select procedures
for the current task, and revise experience-derived skills for later
use. The catalog and reference retrieval pipeline build on our prior work,
SkillCorpus \citep{wang2026skillcorpus}. SkillCorpus introduced the
catalog construction in \cref{sec:skillforge-hub}, including corpus
deduplication, quality assessment, and release admission. It also
released the reference stack of query rewriting, recall, reranking, and
LLM selection on which \cref{sec:skillforge-retrieval} builds, and
\cref{sec:eval-skills} summarizes its published experiments. Raven
extends this work with an online router that fuses the SkillHub
catalog with local and EverOS skills (\cref{sec:skillforge-retrieval}).
Stored turns enter EverOS, whose agent pipeline converts execution
cases into procedures (\cref{sec:skillforge-evolution}).
\Cref{fig:method-skills} shows how corpus admission, task-specific
selection, and experience-driven revision connect.

\begin{figure}[!htbp]
    \centering
    \includegraphics[width=\linewidth]{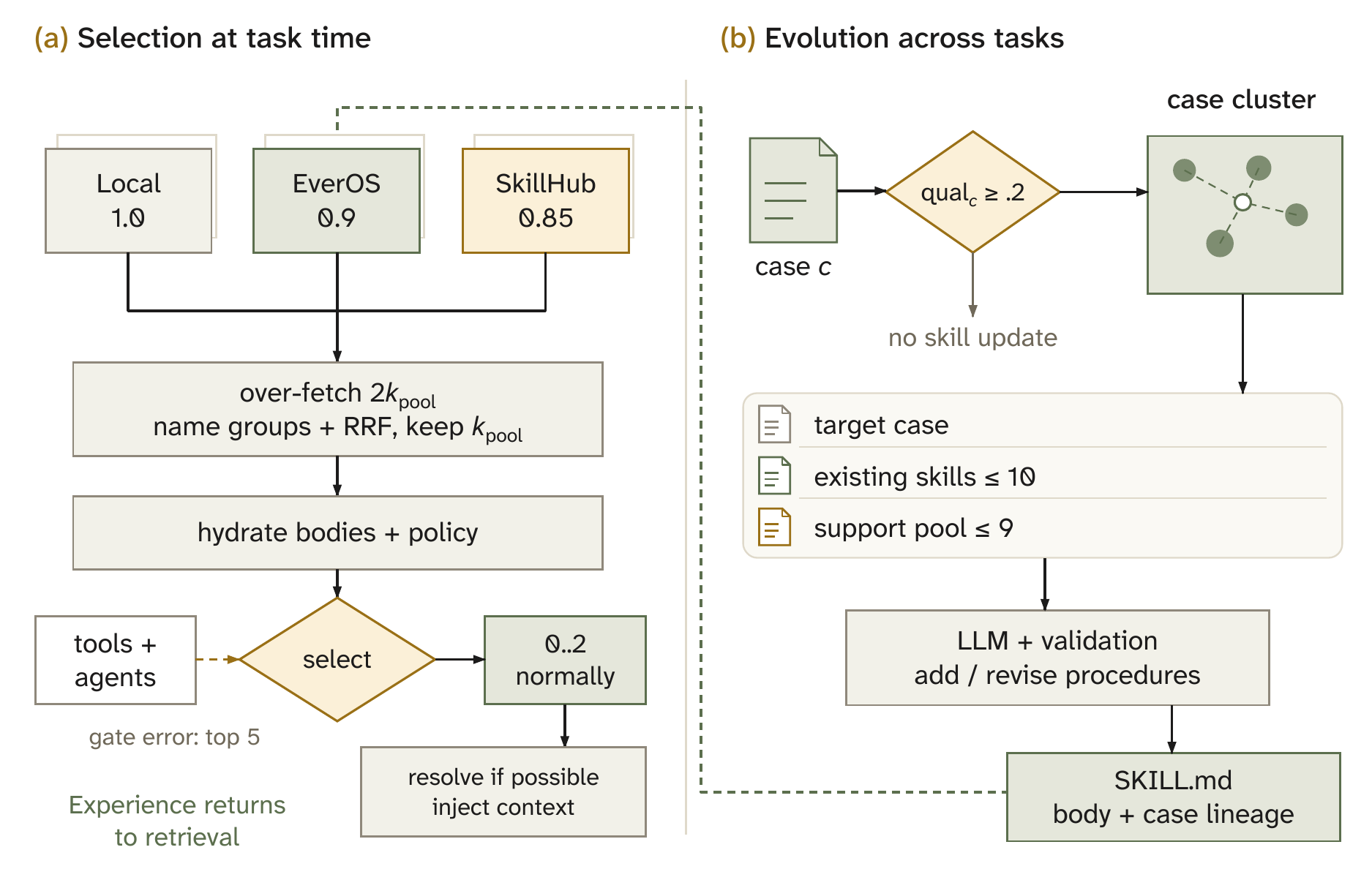}
    \caption{Skill selection and experience-driven updates.
    (a) Ranked source hits are grouped by name and fused before body
    retrieval, runtime policy checks, and LLM selection. The normal
    gate selects a bounded list, while a gate failure uses a separate
    ranked fallback. Resource resolution is attempted after selection.
    (b) A case enters an intent cluster and conditions skill operations
    together with existing skills and a bounded set of supporting cases.
    Persisted revisions become candidates for later
    EverOS retrieval after indexing. The positive integer \(k_{\rm pool}\)
    is the fused skill-pool cap. Numeric labels show the implementation
    defaults in \cref{app:method-defaults}.}
    \label{fig:method-skills}
\end{figure}

\subsection{Curated Skill Catalog}
\label{sec:skillforge-hub}

SkillHub is the live skill catalog that Raven queries. Its corpus and
retrieval design build on SkillCorpus \citep{wang2026skillcorpus},
whose curation pipeline this subsection summarizes.

\paragraph{Skill Representation.}
We denote a skill record by a sans-serif \(\mathsf{s}\), distinguishing it from
the analytical state \(s\) in \cref{sec:theory}, and represent it as
\begin{equation}
    \mathsf{s}=(\mathrm{id}_{\mathsf{s}},\mathrm{name}_{\mathsf{s}},\mathrm{desc}_{\mathsf{s}},\mathrm{body}_{\mathsf{s}},\mathcal R^{\rm res}_{\mathsf{s}},\mathrm{meta}_{\mathsf{s}}),
    \label{eq:skill-record}
\end{equation}
where \(\mathrm{id}_{\mathsf{s}}\) is its source-native identifier,
\(\mathrm{name}_{\mathsf{s}}\) its display name, \(\mathrm{desc}_{\mathsf{s}}\) its applicability
and trigger description, \(\mathrm{body}_{\mathsf{s}}\) its procedural body, \(\mathcal R^{\rm res}_{\mathsf{s}}\)
its bundled resources, and \(\mathrm{meta}_{\mathsf{s}}\) its metadata. Metadata can include
origin, curation scores, owner, and supporting-case identifiers.
In the \texttt{SKILL.md} representation, name and description occupy
frontmatter, the procedure is Markdown, and optional directories hold
scripts, references, or assets. Separating these fields permits compact
discovery while retaining the instructions used at execution time.
A name is not a globally unique identifier.

\paragraph{Corpus Construction and Deduplication.}
SkillCorpus aggregates public skills through a source registry and
applies six curation stages. The pipeline parses each entry, filters malformed or length-ineligible
entries, removes duplicates, assesses quality and assigns issue flags,
applies release admission rules, and attaches source priors, retrieval
embeddings, and index entries. The reference release contains 96,401 active skills. This count describes
a fixed corpus snapshot and may differ from the live SkillHub catalog.

Deduplication first collapses content- and name-fingerprint matches.
The semantic stage embeds the remaining entries in a 1,024-dimensional space and
compares pairs by cosine similarity. Similarity above \(0.995\)
triggers automatic merging, while pairs in \((0.90,0.995]\) receive LLM
adjudication. Removing duplicate procedures reduces redundant retrieval
and potential near-duplicate collisions in retriever training. This
corpus-level procedure differs from the online router's exact-name
collapse in \cref{sec:skillforge-retrieval}.

\paragraph{Quality Facets and Release Admission.}
After deduplication, the judge assigns integer utility, robustness, and safety subscores
from \(\{0,\ldots,10\}\). Dividing them by ten yields
\(\mathrm{util}_{\mathsf{s}},\mathrm{rob}_{\mathsf{s}},\mathrm{safe}_{\mathsf{s}}\in[0,1]\). Utility evaluates the description's scope
and triggers. Robustness evaluates the body and its consistency with
the description. Safety evaluates risks in following the instructions.
Each of the nineteen issue flags imposes a cap on its assigned quality
facet. Soft flags apply through these caps, without a separate score deduction.

For a parsed, deduplicated, judged record, let \(\mathcal F^{\rm flag}_{\mathsf{s}}\) be
its triggered flags, \(\mathrm{mal}_{\mathsf{s}}\in\{0,1\}\) a malware-pattern match, and
\(g_{\rm lic}(\mathsf{s})\in\{0,1\}\) acceptance under the release's source
license policy. The hard-flag set \(\mathcal F^{\rm flag}_{\rm hard}\) contains
\texttt{prompt\_injection}, \texttt{cmd\_injection},
\texttt{unsafe\_exec}, \texttt{auth\_bypass}, and
\texttt{csam\_risk}. Release admission is
\begin{equation}
    g_{\rm rel}(\mathsf{s})=g_{\rm lic}(\mathsf{s})\,
        \mathbb I[\mathrm{mal}_{\mathsf{s}}=0]\,
        \mathbb I[\mathcal F^{\rm flag}_{\mathsf{s}}\cap\mathcal F^{\rm flag}_{\rm hard}=\varnothing]\,
        \mathbb I[\mathrm{safe}_{\mathsf{s}}\ge0.3],
    \label{eq:skill-admission}
\end{equation}
Only records that pass
structural validation and deduplication reach this decision.

For admitted skills, safety attenuation and content quality are
\begin{align}
    \operatorname{atten}(x)&=
       \begin{cases}
           0.5+0.5(x-0.3)/0.4,&0.3\le x\le0.7,\\
           1,&0.7<x\le1,
       \end{cases}
       \label{eq:skill-safety-attenuation}\\
    \mathrm{qual}_{\rm content}(\mathsf{s})&=
       \operatorname{atten}(\mathrm{safe}_{\mathsf{s}})\,[0.50\,\mathrm{util}_{\mathsf{s}}+0.35\,\mathrm{rob}_{\mathsf{s}}+0.15\,\mathrm{safe}_{\mathsf{s}}].
       \label{eq:skill-content-quality}
\end{align}
Thus \(\operatorname{atten}(0.3)=0.5\) and \(\operatorname{atten}(0.7)=1\). It is not evaluated
for a safety-rejected record. Let \(\operatorname{src}(\mathsf{s})\) identify
the corpus source and \(p_{\rm src}(\operatorname{src}(\mathsf{s}))\) be its
prior in \cref{app:method-skill-prior}. Define the structural bonus
\(\mathrm{bonus}_{\mathsf{s}}=0.05\mathbb I[\texttt{scripts/}\text{ present}]
+0.02\mathbb I[\texttt{references/}\text{ present}]\).
The admitted skill's composite quality is
\begin{equation}
    \mathrm{qual}_{\rm cur}(\mathsf{s})=\operatorname{clip}_{[0,1]}
       \bigl(0.85\,\mathrm{qual}_{\rm content}(\mathsf{s})
          +0.15\,p_{\rm src}(\operatorname{src}(\mathsf{s}))+\mathrm{bonus}_{\mathsf{s}}\bigr),
    \label{eq:skill-quality}
\end{equation}
where \(\operatorname{clip}_{[a,b]}(z)=\min(b,\max(a,z))\).
A prior or structural bonus cannot override rejection by
\cref{eq:skill-admission}. The prior uses source-level content-quality
statistics, not the final composite score, avoiding a circular
definition. The resulting scores reflect curation rules and are not calibrated
probabilities of downstream success. Each released skill also receives
one of sixteen task-domain labels for organization.

\paragraph{Catalog Retrieval Artifacts.}
To search the resulting catalog, SkillCorpus provides two fine-tuned models, using
\modelname{Qwen3-Embedding-0.6B} for retrieval and
\modelname{Qwen3-Reranker-0.6B} for reranking. They use a retrieval field
distilled to 3,000 characters from each body. The reference stack
recalls and reranks candidates, then presents full bodies to an LLM
selector returning zero to two skills. Raven's integration combines several sources and supplies bounded body
excerpts to its selection gate.
Consequently, the corpus, its packaged retrieval models, and Raven's
online selector are distinct components of an evaluation.

\subsection{Task-Aware Skill Retrieval and Selection}
\label{sec:skillforge-retrieval}

\paragraph{Query, Sources, and Budgets.}
Let \(x_{\rm in}\) be the current request, \(\mathcal U_{\rm tool}\) the receiving
agent's available tools, and \(\mathcal A_{\rm spec}\subseteq\mathcal A\)
its available specialist roster. The latter need not contain every executor
in the theoretical pool, including the host's nondelegating policy.
An optional rewriter returns \((\mathrm{need},q)\), with
\(\mathrm{need}\in\{0,1\}\) and retrieval query \(q\). If the rewriter returns a valid skip decision, the router supplies no skill
segment. Otherwise, \(q\) retains the
task class, domain, and required capabilities while removing irrelevant
identifiers. A rewriter failure keeps retrieval enabled with \(x_{\rm in}\)
as the query. The gate subsequently sees this retrieval query, so
query rewriting affects both retrieval and selection.

The positive integer \(k_{\rm pool}\) caps the fused candidate pool.
The retrieval sources are
\[
    \mathcal J_{\rm src}=\{\mathrm{local},\mathrm{memory},\mathrm{hub}\}.
\]
Configured sources are queried concurrently for up to \(2k_{\rm pool}\) hits each, where \(\mathrm{hub}\)
denotes SkillHub. Local retrieval uses
the indexed local catalog. The memory adapter recalls skills within
its configured agent scope. SkillHub retrieval initially returns metadata.
The router can pass session history, but an adapter may ignore it.
A failed source contributes an empty list while other sources remain
usable. \Cref{app:method-defaults} collects the budgets for gated and
ungated retrieval, fusion weights, and extraction settings. Candidate budgets limit the number of entries, not the token count of the
final skill bodies.

\paragraph{Name-Group Fusion and Representative Selection.}
To combine the source rankings, let \(L_j^{\rm hit}=(\mathsf{h}_{j1},\ldots,\mathsf{h}_{jm_j})\) be source \(j\)'s
ranked hit list. A hit contains a qualified identifier
\(\operatorname{qid}(\mathsf{h})=\texttt{source/native\_id}\), display name \(\operatorname{name}(\mathsf{h})\),
body, metadata, and source-native score \(\operatorname{score}_{\rm native}(\mathsf{h})\). The display
name is the \(\mathrm{name}_{\mathsf{s}}\) field of the underlying skill record.
Positions \(r=1,\ldots,m_j\) are one-based ranks. For every name
\(\mathsf n\) present in the lists, the weighted reciprocal-rank score is
\begin{equation}
    \operatorname{RRF}(\mathsf n)=\sum_{j\in\mathcal J_{\rm src}}\sum_{r=1}^{m_j}
       \mathbb I[\operatorname{name}(\mathsf{h}_{jr})=\mathsf n]\,
       \frac{w_j}{c_{\rm RRF}^{\rm skill}+r}.
    \label{eq:skill-fusion}
\end{equation}
Weights \(w_j\ge0\) are source preferences and
\(c_{\rm RRF}^{\rm skill}>0\) controls rank damping. The formula sums hit occurrences, including
repeated names within one source if supplied. It groups exact name matches without testing semantic equivalence.

Let \(\mathsf{Hits}_{\mathsf n}=\{\mathsf{h}_{jr}:\operatorname{name}(\mathsf{h}_{jr})=\mathsf n\}\) be the
observed hits carrying a name. The representative whose content will be rendered is selected separately:
\begin{equation}
    \operatorname{rep}(\mathsf n)\in
       \arg\max_{\mathsf{h}\in\mathsf{Hits}_{\mathsf n}}\operatorname{score}_{\rm native}(\mathsf{h}),
    \qquad
    \mathsf{Pool}_0=\operatorname{Take}_{k_{\rm pool}}
       \bigl([\operatorname{rep}(\mathsf n)]_{\mathsf n\text{ ordered by }\operatorname{RRF}(\mathsf n)\downarrow}\bigr).
    \label{eq:skill-representative}
\end{equation}
Here \(\operatorname{Take}_{k_{\rm pool}}\) retains the first \(k_{\rm pool}\) entries of
an ordered list. Ties in representative scores retain the first encountered hit. Ties in
fused scores follow the order in which names first appear, determined by
registered-source order and within-source rank. Thus RRF
avoids comparing native score scales when ordering names, while
representative selection still compares those scores. Rescaling a
source's scores can change the injected version without changing its
name's RRF position. The representative's qualified identifier and
contributing-source metadata preserve this distinction.

\paragraph{Body Retrieval, Policy Checks, and Model Selection.}
After fusion, the runtime removes blocked representatives from \(\mathsf{Pool}_0\), then
fetches missing SkillHub bodies and applies the configured detail policy.
A failed detail fetch or policy refusal removes that candidate.
Filtering occurs after the top-\(k_{\rm pool}\) fusion, without refilling from
lower-ranked names. Let \(\mathsf{Pool}\) be the remaining ordered pool, with
\(|\mathsf{Pool}|\le k_{\rm pool}\). Catalog release admission (\cref{eq:skill-admission})
and runtime admission are
separate checks. Their default safety thresholds differ, as recorded
in \cref{app:method-defaults}.

To select procedures from this filtered pool, the LLM gate receives \(q\), qualified identifiers, and descriptions and
normalized body excerpts under separate character limits.
Available tool and specialist rosters are supplied
when obtainable. The prompt asks the gate to plan the task, reject
procedures requiring unavailable capabilities, and avoid procedures
whose work is already covered by a specialist. Procedures shaping the
current agent's planning or verification can still be useful. The gate assesses compatibility from the supplied excerpts and rosters.
However, this judgment does not guarantee that all requirements for executing the
selected skill are satisfied.

Let \(\mathsf{IDs}\) be the ordered identifier list from a successfully parsed
gate response, \(k_{\rm sel}\ge1\) the normal selection cap, and \(k_{\rm fb}\ge1\)
the failure-fallback cap, both integers. Define
\(\operatorname{Lookup}(\mathsf{IDs},\mathsf{Pool})\) to replace each known identifier by
its pool entry, ignore unknown identifiers, and preserve list order.
The selected list is
\begin{equation}
    S_{\rm sel}=\begin{cases}
       \operatorname{Take}_{k_{\rm sel}}(\operatorname{Lookup}(\mathsf{IDs},\mathsf{Pool})),
           &\text{gate response parses successfully},\\
       \operatorname{Take}_{k_{\rm fb}}(\mathsf{Pool}),
           &\text{gate call or response parsing fails},\\
       \mathsf{Pool},&\text{gate is disabled}.
    \end{cases}
    \label{eq:skill-selection-decision}
\end{equation}
The response fields are specified in \cref{app:method-contracts}.
A successful empty identifier list remains empty without triggering fallback.
With the default caps, fallback can select more entries than a normal response. The current
parser preserves repeated known identifiers, so the cap bounds list
entries, not necessarily distinct procedures. Gate output membership
is checked against \(\mathsf{Pool}\). The accompanying plan is explanatory text,
not an executable task graph.

\paragraph{Resource Resolution and Context Injection.}
After selection, local references are resolved against the skill's
directory. Selected SkillHub resources are downloaded only when the
installation policy permits, using the previously fetched metadata.
A skipped or failed installation retains the available body with
unresolved resource references. Memory-derived skills have no bundled
files in this path. The runtime renders selected bodies, qualified
identifiers, and any resolved resource locations into the agent's
context. Selection, resource availability, and actual execution are
therefore different events. Recording an injected identifier establishes selection, but successful
execution requires separate evidence.

\begin{ravenalgorithm}{Task-Aware Skill Selection}
\label{alg:skill-selection}
\textbf{Input:} request \(x_{\rm in}\), session history, capabilities
\((\mathcal U_{\rm tool},\mathcal A_{\rm spec})\), configured sources and clients,
rewriter/gate switches, budgets \((k_{\rm pool},k_{\rm sel},k_{\rm fb})\), runtime policy.\\
\textbf{Output:} skill context and selected identifiers with provenance.
\par\smallskip
\begin{tabularx}{\linewidth}{@{}r@{\quad}X@{}}
1 & \textbf{if} no router is configured, \textbf{return} an empty segment.\\
2 & Initialize \((\mathrm{need},q)\leftarrow(1,x_{\rm in})\).
    If enabled and the request is nonblank, run the rewriter.
    Retain this initialization on failure.\\
3 & \textbf{if} \(\mathrm{need}=0\), \textbf{return} an empty segment.
    If the rewritten query is empty, use \(x_{\rm in}\).\\
4 & Query sources concurrently for \(2k_{\rm pool}\) hits each.
    Assign \(L_j^{\rm hit}\leftarrow[\,]\) for each failed source.\\
5 & Compute name scores and representatives by
    \cref{eq:skill-fusion,eq:skill-representative}, obtaining \(\mathsf{Pool}_0\).\\
6 & Remove blocked entries, retrieve missing SkillHub bodies, and apply
    detail policy, dropping failed or refused entries. Call the result \(\mathsf{Pool}\).\\
7 & \textbf{if} \(\mathsf{Pool}=[\,]\), \textbf{return} an empty segment.\\
8 & If enabled, run the gate on \((q,\mathsf{Pool},\mathcal U_{\rm tool},\mathcal A_{\rm spec})\) and
    distinguish a parsed identifier list from call/parse failure.\\
9 & Compute \(S_{\rm sel}\) by \cref{eq:skill-selection-decision}.
    No further truncation to \(k_{\rm sel}\) is applied to the failure fallback.\\
10 & For selected entries, attempt resource resolution under policy.
    Preserve available bodies when SkillHub installation is skipped or fails.\\
11 & Render the selected entries and available paths, then return the
    context with qualified identifiers and source provenance.\\
\end{tabularx}
\end{ravenalgorithm}

\subsection{Experience-Driven Skill Self-Evolution}
\label{sec:skillforge-evolution}

\paragraph{Update Pipeline.}
To turn execution experience into reusable procedures, EverOS's agent
pipeline performs the updates below when Raven submits
interaction segments for memory formation
(\cref{sec:memory-formation}). The procedure uses the backend's default
thresholds. During retrieval, \cref{eq:skill-fusion} groups local and
EverOS skills with identical names into one candidate.

\paragraph{Execution Cases and Ownership.}
Inside the EverOS pipeline, an extractor produces zero or one
substantive case from an agent-mode memory cell:
\begin{equation}
    c=(\mathrm{id}_c,t_c,\mathrm{intent}_c,
       \mathrm{approach}_c,\mathrm{insight}_c,\mathrm{qual}_c).
    \label{eq:skill-case}
\end{equation}
Here \(t_c\) is the case timestamp, intent describes the task class,
approach describes the execution strategy, insight states the reusable
lesson, and \(\mathrm{qual}_c\in[0,1]\) is estimated case quality. It is distinct
from catalog quality \(\mathrm{qual}_{\rm cur}(\mathsf{s})\) and skill confidence.
The same case is persisted separately for each participating assistant
owner. Let \(\operatorname{own}(a)\) denote the owner partition of agent \(a\)
within its application and project scope. This replication makes a shared execution
available to its participants. It does not constitute independent
successes or causal credit for each participant.

\paragraph{Incremental Task-Intent Clustering.}
A case below the quality threshold \(\theta_{\rm skip}\), or a deployment without
embedding capability, produces no skill update. Otherwise, embed its
intent as \(\mathbf{z}_c\in\mathbb R^{d_{\rm emb}}\). An existing cluster \(C\) has
centroid \(\boldsymbol{\mu}_C\in\mathbb R^{d_{\rm emb}}\), positive count \(n_C\), members,
timestamp, and previews.
The clustering operator recalls a bounded list of centroids by
\begin{equation}
    \operatorname{sim}(\mathbf{z}_c,\boldsymbol{\mu}_C)=
       \frac{\mathbf{z}_c^\top\boldsymbol{\mu}_C}
       {(\|\mathbf{z}_c\|_2+\epsilon_{\rm num})(\|\boldsymbol{\mu}_C\|_2+\epsilon_{\rm num})},
       \qquad\epsilon_{\rm num}=10^{-9}.
    \label{eq:skill-cluster-similarity}
\end{equation}
Empty centroids are skipped. With no eligible cluster, a singleton
is created. If the highest similarity meets the direct-merge threshold, the case joins
the corresponding cluster. Otherwise the LLM receives the recalled
previews and returns a cluster index or a new-cluster decision.
A valid existing index is accepted. An out-of-range index creates a
singleton, while a parse error propagates for retry. No temporal
window is applied. Merging a singleton into \(C\) gives
\begin{equation}
    \boldsymbol{\mu}_C^+=\frac{n_C\boldsymbol{\mu}_C+\mathbf{z}_c}{n_C+1},\qquad n_C^+=n_C+1.
    \label{eq:skill-cluster-update}
\end{equation}
The update appends the member, retains the latest timestamp, and limits
previews to five.

Within \(\operatorname{own}(a)\), clustering serializes its read--assign--write
sequence. The resulting cluster identifier, target case content, and embedding
are carried by an update event. The skill extractor can therefore
reconstruct its target without waiting for the case's search-index row.
Skill writes use a separate lock for the same owner partition, since
updates to different clusters can still target the same skill name.

\paragraph{Selecting and Formatting the Update Context.}
With the cluster assigned, let \(\mathsf{Catalog}_C\) be its existing Markdown skills.
Select an ordered list
\(\mathsf{Skills}_C=(\mathsf{s}_0,\ldots,\mathsf{s}_{m-1})\), with
\(m\le k_{\rm skill}\), where \(k_{\rm skill}\) is the existing-skill cap. If \(|\mathsf{Catalog}_C|\le k_{\rm skill}\), use the Markdown
listing directly. Otherwise the index ranks relevant skills by the
case vector. Accept only hits with matching Markdown records and
backfill missing slots in Markdown order. The files determine
existence and content, while the index supplies an ordering.

For a skill \(\mathsf{s}\), let \(\Lambda(\mathsf{s})\) be its ordered supporting-case
identifier list. The candidate support identifiers and collected case pool are
\begin{align}
    \mathsf{SupIDs}_C&=\left(\bigcup_{\mathsf{s}\in \mathsf{Skills}_C}\Lambda(\mathsf{s})\right)
                   \setminus\{\mathrm{id}_c\},\\
    \mathsf{Cases}_C&=\operatorname{Take}_{k_{\rm case}}
       \bigl(\operatorname{Sort}_{(\mathrm{qual}_{c'},t_{c'})\downarrow}
          \operatorname{Fetch}_{\operatorname{own}(a)}(\mathsf{SupIDs}_C)\bigr).
    \label{eq:skill-support-pool}
\end{align}
Here \(k_{\rm case}\) is the supporting-case pool cap. The union treats lineage lists
as sets of identifiers. \(\operatorname{Fetch}_{\operatorname{own}(a)}\) returns available indexed
cases in the same owner partition, once per identifier. Records absent from the index
are excluded from the current support pool. Sorting prioritizes quality, then
recency. The target is always excluded from its own support pool.

Support formatting applies a separate limit to each skill.
For each \(\mathsf{s}\in \mathsf{Skills}_C\), the formatter takes the last matching
identifiers in \(\Lambda(\mathsf{s})\) whose records occur in \(\mathsf{Cases}_C\),
up to the per-skill support cap and preserving lineage order.
It supplies their summaries alongside that skill's name, confidence,
description, and body. Description, body, and supporting-approach prefixes
have separate token budgets, with omission markers appended on truncation.
The new target is supplied separately. These token budgets differ from the
online gate's character limits.

\paragraph{Generating Explicit Update Operations.}
The target quality determines the prompt mode:
\begin{equation}
    \operatorname{mode}(c)=
    \begin{cases}
        \mathsf{skip},&\mathrm{qual}_c<\theta_{\rm skip},\\
        \mathsf{failure},&\theta_{\rm skip}\le \mathrm{qual}_c<0.5,\\
        \mathsf{success},&0.5\le \mathrm{qual}_c\le1.
    \end{cases}
    \label{eq:skill-prompt-mode}
\end{equation}
For a non-skipped case, one structured LLM call uses this context to produce an ordered
operation list
\begin{equation}
    \mathsf{Ops}=\operatorname{SkillExtract}_{\operatorname{mode}(c)}
          \bigl(\operatorname{Format}(c,\mathsf{Skills}_C,\mathsf{Cases}_C)\bigr)
      =(\mathsf{op}_1,\ldots,\mathsf{op}_{|\mathsf{Ops}|}).
    \label{eq:skill-extraction-context}
\end{equation}
A well-formed operation \(\mathsf{op}_l=(\mathrm{act}_l,i_l,\Delta_l^{\rm fields})\) has action
\(\mathrm{act}_l\in\{\mathsf{add},\mathsf{update},\mathsf{none}\}\),
optional target index \(i_l\), and proposed field values
\(\Delta_l^{\rm fields}\). Update indices are zero-based positions in the fixed
input list \(\mathsf{Skills}_C\), not database identifiers or indices into newly
added records. Malformed list items or unknown actions are skipped.
An empty operation list means no change. Malformed top-level
responses are errors handled by the strategy retry path.

\paragraph{Validation, Field Changes, and Lineage.}
An addition requires a body with at least five nonempty lines and
fifty characters after trimming, and at least one of name or
description. For an update, the target must satisfy
\(0\le i_l<m\) and must not have been addressed by an earlier
update in the same batch. The index is reserved before validating
replacement content. A supplied nonempty replacement body must pass
the same length checks, while omitted or empty text fields preserve their
previous values. An update is accepted only if it changes a text
field or supplies a parseable confidence value. Supplying the existing confidence value can therefore still append new
evidence.

Let \(\chi(\mathsf{s})\in[0,1]\) denote confidence. A valid proposed value
is clipped to \([0,1]\). A malformed value retains prior confidence
on update and defaults to \(0.5\) on addition. An accepted normal
update produces a revised record with lineage
\begin{equation}
    \Lambda(\mathsf{s}^+)=\operatorname{Unique}
          \bigl(\Lambda(\mathsf{s})\mathbin{+\!\!+}[\mathrm{id}_c]\bigr),
    \label{eq:skill-lineage-update}
\end{equation}
where \(\mathbin{+\!\!+}\) concatenates lists and \(\operatorname{Unique}\)
preserves first occurrence. An addition starts with
\([\mathrm{id}_c]\). These checks produce an emitted record list
\(\mathsf{Emitted}=\operatorname{ApplyOps}(\mathsf{Ops},\mathsf{Skills}_C,c)\). Only emitted records are written, leaving the owner's other skills unchanged.

A confidence below \(0.1\) does not remove a skill. EverOS writes that
record as an ordinary skill, and later retrieval can return it.
The pipeline does not invoke \texttt{retire\_confidence}. Optional maturity scoring is
skipped by default: new skills receive its default value \(1.0\), and
retained maturity metadata is not an independent assessment of task
utility.

\paragraph{Persistence and Subsequent Reuse.}
Emitted records are written as owner-scoped Markdown skills in the
EverOS store, with cluster membership and case lineage. These files
are separate from the workspace catalog. Updates preserve identity
within the extraction batch to reconcile renames: write the new
sanitized name, then remove an obsolete directory unless another
emitted record claimed that old name. Non-emitted skills remain
unchanged. The owner lock serializes competing strategy executions.
It does not make multiple file writes an atomic transaction. Cascade
indexing later exposes the revisions to the memory source in
\cref{alg:skill-selection}.

\begin{ravenalgorithm}{Incremental Case-to-Skill Evolution}
\label{alg:skill-evolution}
\textbf{Input:} new case \(c\), owner \(\operatorname{own}(a)\), owned cluster
state, Markdown skill store, case index, configured model services.\\
\textbf{Output:} emitted skill records persisted with case lineage,
possibly an empty update. Failures follow the strategy retry policy.
\par\smallskip
\begin{tabularx}{\linewidth}{@{}r@{\quad}X@{}}
1 & \textbf{if} embedding is unavailable or \(\mathrm{qual}_c<\theta_{\rm skip}\),
    \textbf{return} no skill update.\\
2 & Under the clustering lock for \(\operatorname{own}(a)\), embed the intent,
    recall centroids, and assign or create a cluster as described above.\\
3 & Persist the cluster, emit its identifier together with \(c\) and \(\mathbf{z}_c\), and
    release the clustering lock.\\
4 & Under the skill-update lock for \(\operatorname{own}(a)\), verify that the
    cluster exists. A missing cluster is a retryable error.\\
5 & Enumerate cluster skills from Markdown and select \(\mathsf{Skills}_C\)
    with the \(k_{\rm skill}\) budget and index-ranking/backfill rule.\\
6 & Collect \(\mathsf{Cases}_C\) by \cref{eq:skill-support-pool} and format the
    bounded per-skill support and text fields.\\
7 & Request \(\mathsf{Ops}\) using \cref{eq:skill-prompt-mode,eq:skill-extraction-context}.
    Malformed top-level output raises an error before skill writes.\\
8 & Initialize reserved targets \(\mathsf{UsedTargets}\leftarrow\varnothing\)
    and emitted list \(\mathsf{Emitted}\leftarrow[\,]\).\\
9 & \textbf{for each} \(\mathsf{op}_l\in\mathsf{Ops}\), in order:\\
10 & \quad \textbf{if} the action is \(\mathsf{update}\): parse \(i_l\).
    Skip if invalid, outside \([0,m)\), or already in \(\mathsf{UsedTargets}\).
    Otherwise, set \(\mathsf{UsedTargets}\leftarrow \mathsf{UsedTargets}\cup\{i_l\}\).\\
11 & \quad Validate an add/update operation's body and fields.
    Skip rejected operations, \(\mathsf{none}\), and unknown actions.\\
12 & \quad Construct the accepted record, apply
    \cref{eq:skill-lineage-update}, and append it to \(\mathsf{Emitted}\).
    A confidence below \(0.1\) is still emitted and remains retrievable.\\
13 & Write emitted records in order under sanitized owner-scoped names.
    Retain the written-name map for rename reconciliation.\\
14 & Reconcile obsolete names after writes and release the update lock
    on completion or error. Index changes asynchronously.\\
\end{tabularx}
\end{ravenalgorithm}

Across these stages, case quality, catalog quality, source rank, skill confidence,
and measured task outcome are different quantities. Neither a
structurally accepted update nor a high-confidence skill establishes
an improvement. \Cref{sec:eval-skills} reports the effect of the fixed
catalog on task performance with the model and harness held fixed.

\section{Evaluation}\label{sec:evaluation}

We evaluate Raven's effectiveness in organizing specialist work, improving
execution with a fixed model, and reusing procedural knowledge. Planning
experiments test specialist selection and dependency prediction, while
domain evaluations measure task outcomes and resource use.
Harness-evolution and skill-reuse experiments then examine how
adaptation and reusable procedures improve agent performance. Together,
these evaluations connect Raven's architectural choices to measurable
gains in planning, execution, and reuse.

\paragraph{Comparison Protocol.}
The planning comparison fixes the backbone across systems to assess their
orchestration capabilities. Harness-evolution experiments keep the task
model frozen and separate adaptation tasks from held-out testing. Skill
experiments hold the backbone and harness fixed within each comparison,
isolating the effect of adding retrieved procedures. Specialist results
report task quality alongside runtime, token use, and monetary cost where
available. Benchmark-specific settings describe the task split, tool
access, attempt count, scoring procedure, and resource budget. Adaptation
data and development compute are distinguished from final evaluation.

\paragraph{Evaluation Structure.}
We first evaluate multi-agent orchestration and harness self-evolution,
then assess the four specialists on research, software engineering, visual
design, and sustained operation. The final experiments measure the benefit
of retrieving and executing skills from a fixed catalog.
\Cref{app:specialist-eval} gives scoring conventions and protocol details
for these comparisons.

\subsection{Multi-Agent Orchestration}\label{sec:multi-agent-evaluation}

Multi-agent planning determines which specialists receive a request, what
information they receive, and which subtasks may execute concurrently.
End-to-end task accuracy reflects both planning decisions and downstream
execution quality, making planning errors difficult to isolate. We therefore
evaluate the host's proposed graph before dispatching workers. This protocol
measures specialist selection and dependency prediction independently of
specialist execution.

\subsubsection{The Multi-Agent Orchestration Benchmark}\label{subsec:mao-benchmark}

We introduce the \emph{Multi-Agent Orchestration Benchmark} (MAOB), a suite of
$140$ requests modeled on occupational tasks, each paired with a reference directed
acyclic graph (DAG) over a fixed roster of four specialist sub-agents. The
benchmark measures agreement between the host's proposed graph and a reviewed
reference graph for each request.

\paragraph{Specialist Roster.}\label{par:mao-domains}
MAOB defines four specialist domains:

\begin{itemize}
  \item \textbf{research}: searching the web and producing a source-supported
        analysis.
  \item \textbf{coding}: writing, executing, and debugging code on the host.
  \item \textbf{content}: converting source material into a deliverable,
        such as a presentation, report, or briefing.
  \item \textbf{oncall}: executing and monitoring a long-running or remote
        job through completion.
\end{itemize}

The fixed roster permits exhaustive coverage of domain combinations.
Four domains yield exactly $11$ subsets containing at least two domains,
and MAOB includes all $11$ (Table~\ref{tab:mao-coverage}). This coverage is exhaustive for domain combinations within the declared
roster, but does not extend to all occupational scenarios or tool
configurations.

\paragraph{Reference-Guided Construction.}\label{par:mao-construction}
MAOB specifies the reference graph before generating the request. This provides
an explicit construction target and reduces reliance on a model's post hoc
interpretation of a request. However, graph-first generation still requires review:
the resulting request must justify the selected specialists and dependencies.
The construction procedure has six stages:

\begin{enumerate}
  \item \textbf{Pattern extraction.} We survey public occupational task
        collections, GDPval and JobBench
        \citep{patwardhan2025gdpval,li2026jobbench}, and extract the task structure as the
        triple $\langle$occupation, deliverable, reference-material
        type$\rangle$. No task text is reused, reducing direct overlap with
        the source collections while retaining their occupational framing.
  \item \textbf{Template library.} Each extracted pattern becomes a template
        specifying an occupation, scenario, deliverable, and compatible graphs.
        A template may support several graphs, allowing the same occupational
        scenario to represent different dependency structures.
  \item \textbf{Reference graph first.} For each template instance we fix the
        reference DAG, including its node set and partial order, before generating
        the request. The authored graph provides a construction
        target whose consistency with the resulting request is checked during
        review. Reference graphs are authored with
        \modelname{Claude Opus 5}, chosen from a different model family
        than the systems and backbones under test to reduce shared-family
        labeling effects.
  \item \textbf{Backward generation.} The request text is then generated
        from the graph with GLM-5.2, which also produces task metadata and
        corpus support files. The generator is instructed to make the dependencies inferable from the
request without stating the graph explicitly.
  \item\label{item:leakage} \textbf{Leakage filtering.} Every generated request
        passes an automatic filter before it enters the pool. A request that
        enumerates steps
        (\emph{first\ldots{}then\ldots{}finally}), or that names a domain outright
        (\emph{research}, \emph{write code}, \emph{make a deck}), is rejected and
        regenerated. The filter shares its vocabulary with the generation prompt,
        providing consistent exclusion criteria. Lexical filtering alone cannot
        rule out paraphrased planning cues, which remain a review concern.
  \item \textbf{Review and admission.} Surviving samples are reviewed against
        their reference graph. The review records node-set agreement, ordering
        ambiguity, and the attribution of request elements to graph nodes.
\end{enumerate}

Step~\ref{item:leakage} limits explicit planning cues that would allow a system
to reproduce a listed sequence without inferring dependencies. A reference
chain such as
$\textsf{research}\rightarrow\textsf{coding}\rightarrow\textsf{content}$
is admissible only when the request motivates each specialist and the
information passed between them. Omitting an explicit sequence from the
request is therefore necessary for this construction protocol but does not
by itself establish the validity or uniqueness of the reference graph.

\paragraph{Quality Control.}\label{par:mao-qc}
Every generated sample must pass automatic admission checks before evaluation.
A failed check stops the construction pipeline. In addition to the leakage
filter in step~\ref{item:leakage}, the checks reject three classes of defect:

\begin{itemize}
  \item \textbf{Inconsistent specification.} Contradictory quantities or
        requirements make a request unsatisfiable as written and can confound
        the assessment of planning ability.
  \item \textbf{Role-boundary violation.} A reference graph that assigns work
        outside a specialist's remit, such as asking the content agent to
        retrieve or to compute, contradicts the roster the system under test is
        given.
  \item \textbf{Field inconsistency.} Task metadata must agree with the reference
        graph. A mismatch can exclude a sample from its intended metric group.
\end{itemize}

A further pass removes reference nodes that the request does not justify.
This correction is applied during dataset construction. Masking a reference
node only at scoring time could otherwise penalize a prediction made under
the original task specification. After removal, edges are reconstructed from
the transitive closure to preserve dependencies among the remaining nodes.

Samples that pass these checks undergo expert review of the reference
node set, partial order, and attribution of request elements to nodes.
Structured review records preserve the basis for each admission decision.

\paragraph{Dataset Statistics.}\label{par:mao-stats}
The released benchmark comprises $140$ tasks drawn from $137$ distinct
occupations across $140$ templates, with one template per task. Reference graphs
contain $2.72$ nodes and $1.84$ edges on average, with $257$ reference edges in
total. Every request is self-contained text without attachments, so the
evaluation does not require file-handling capability.

\begin{table}[t]
\centering
\caption{Coverage of domain combinations. Every subset containing at least
two domains from the four-domain roster is represented.}
\label{tab:mao-coverage}
\begin{tabular}{lcr}
\toprule
\textbf{Domains per task} & \textbf{Subsets covered} & \textbf{Tasks} \\
\midrule
2 & 6\,/\,6 & 66 \\
3 & 4\,/\,4 & 47 \\
4 & 1\,/\,1 & 27 \\
\midrule
\textbf{Total} & \textbf{11\,/\,11} & \textbf{140} \\
\bottomrule
\end{tabular}
\end{table}

Structurally, $98$ tasks are purely serial and $42$ admit parallel execution.
The latter include independent successors of a shared predecessor and
independent entry nodes that contribute to a common deliverable. Including
both structures requires the planner to distinguish sequential dependencies
from opportunities for concurrency.

Domain frequency is close to even across the three deliverable-producing
domains: \textsf{content} appears in $106$ tasks, \textsf{research} in $103$,
and \textsf{coding} in $102$. The \textsf{oncall} domain appears in $70$
tasks that require sustained execution or monitoring. By
subset size, $66$ tasks span two domains, $47$ span three, and $27$ span all
four (\cref{fig:mao-stats}).

\begin{figure}[htbp]
\centering
\includegraphics[width=0.95\linewidth]{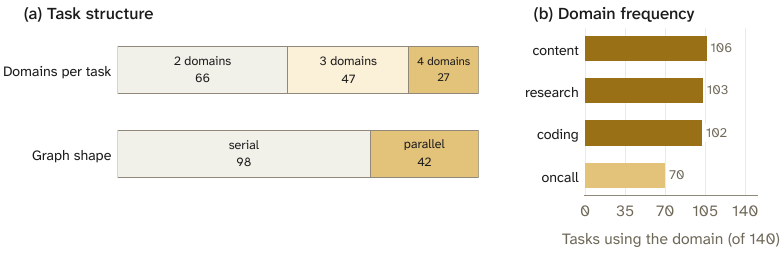}
\caption{MAOB composition. (a) Number of domains per task and the share of
tasks whose reference graph contains parallel branches. (b) Number of
tasks that use each specialist domain.}
\label{fig:mao-stats}
\end{figure}

\subsubsection{Metrics}\label{subsec:mao-metrics}

Let $\widehat G^{\rm raw}$ and $G^{\star,\rm raw}$ denote the submitted and
reference graphs. A domain map assigns each invocation to its specialist
label. The projection $\operatorname{Fold}$ identifies nodes with the
same label, removes edges within a domain, and deduplicates the remaining
directed edges. Write
\[
  \widehat G=(\widehat V,\widehat E)=\operatorname{Fold}(\widehat G^{\rm raw}),\qquad
  G^\star=(V^\star,E^\star)=\operatorname{Fold}(G^{\star,\rm raw}).
\]
All metrics use these domain-level graphs to assess specialist selection
and inter-domain dependencies. Invocation names and subdivisions within
a domain do not affect the scores. The formulas first describe order-comparable,
acyclic predictions. Projection conflicts, cycles, admissible reference
variants, and aggregation are specified in \cref{app:maob-conventions}.

For finite sets $X_1,X_2$, define the total set-agreement function
\[
 \operatorname{F1}(X_1,X_2)=
 \begin{cases}
  1,&X_1=X_2=\varnothing,\\
  2|X_1\cap X_2|/(|X_1|+|X_2|),&\text{otherwise}.
 \end{cases}
\]

\paragraph{Node F1.}
Specialist-selection agreement is
\[
  \mathrm{NodeF1}=\operatorname{F1}(\widehat V,V^\star).
\]
This penalizes omitted reference domains and additional predicted domains.
Every MAOB reference has at least two domains, so its ordinary denominator
is positive even for an empty prediction.

\paragraph{Edge F1.}
Let $\operatorname{TR}(G)$ denote a DAG's transitive reduction, which
removes edges implied by longer directed paths while preserving
reachability. Let $\widehat E_{\rm red}$ be the edge set of
$\operatorname{TR}(\widehat G)$, and let $\mathfrak E^\star_{\rm red}$ contain
the reduced edge sets allowed by the selected reference's order
annotations. Without alternative annotations this family contains only
the edge set of $\operatorname{TR}(G^\star)$. The score is
\[
  \mathrm{EdgeF1}=\max_{E'\in\mathfrak E^\star_{\rm red}}
                     \operatorname{F1}(\widehat E_{\rm red},E').
\]
Thus a redundant transitive edge is not penalized. If both compared edge
sets are empty, their F1 is one. The metric compares the normalized
dependency structure.

\paragraph{Partial Order Accuracy (POA).}
Let $V_\cap=\widehat V\cap V^\star$ and define
\[
 \mathcal P_{\rm pair}=\{(u,v):u,v\in V_\cap,\ u\prec_{\rm dom}v\},
\]
where $\prec_{\rm dom}$ is the fixed roster order research, coding,
content, oncall, so each unordered pair is counted once. The relation
$\operatorname{ord}_G(u,v)$ is \emph{before}, \emph{after}, or
\emph{unordered}, determined by reachability in $G$.
Let $\mathcal O^\star_{uv}$ be the reference's accepted relations for
that oriented pair. It is the singleton
$\{\operatorname{ord}_{G^\star}(u,v)\}$ unless an annotation explicitly
allows alternatives. Then
\[
 \mathrm{POA}=\begin{cases}
 \displaystyle\frac{1}{|\mathcal P_{\rm pair}|}
   \sum_{(u,v)\in\mathcal P_{\rm pair}}
    \mathbb I[\operatorname{ord}_{\widehat G}(u,v)\in\mathcal O^\star_{uv}],
     &|V_\cap|\ge2,\\[5pt]
 1,&|V_\cap|=1,\\
 0,&|V_\cap|=0.
 \end{cases}
\]
A single shared node provides no ordering pair to compare and receives a
score of one. A prediction with no shared node receives zero. POA must therefore be interpreted together with
Node F1. Different node or edge sets can still preserve the relative
order of their shared nodes.

\paragraph{Exact Match Rate.}
For an acyclic, order-comparable prediction, the per-record Exact Match
indicator is one precisely when $\widehat V=V^\star$ and every shared pair's
predicted relation belongs to $\mathcal O^\star_{uv}$. The reported rate
averages this indicator over records with a defined value. Without
alternative annotations this means equal node sets and equal reachability,
or equivalently equal transitive reductions. It does not require identical
raw edge lists. An alternative valid plan can differ from the reference.

The four metrics are reported separately. They distinguish specialist
selection, reduced dependencies, accepted precedence, and whole-plan
agreement at the domain level. They are not combined into a single score.

\subsubsection{Experimental Setup}\label{subsec:mao-setup}

\paragraph{Systems Compared.}
We evaluate three orchestrating agent systems on identical inputs:
Raven and two baseline harnesses, Claude Code and Hermes Agent.
Each system plans through its native agent loop.

\paragraph{Backbones.}
To control for backbone choice, each system is evaluated with two backbones:
\modelname{Qwen3.8-27B} and \modelname{DeepSeek-V4-Flash-0731}. A result is reported only
when all three systems have been run on the same backbone, over the same $140$
tasks, in the same protocol.

\paragraph{Protocol.}
Planning is measured in isolation. Each system is given the request text and
submits its graph through a tool call to its native orchestration interface.
The submitted graph is captured and scored without dispatching workers.
The measured behavior and associated evaluation cost therefore cover planning
alone.

\paragraph{Elicitation Parity.}
All three systems receive identical request text and the same delegation
instruction, taken verbatim from a shared constant. Domain descriptions are
provided only through each system's orchestration-tool schema, without an
additional system-level preamble. This controls the task and domain information
provided to the models while retaining each system's native interface.

\subsubsection{Results}\label{subsec:mao-results}

Figure~\ref{fig:mao-results} reports the four metrics for
all three systems under both backbones.

\begin{figure}[htbp]
\centering
\includegraphics[width=0.95\linewidth]{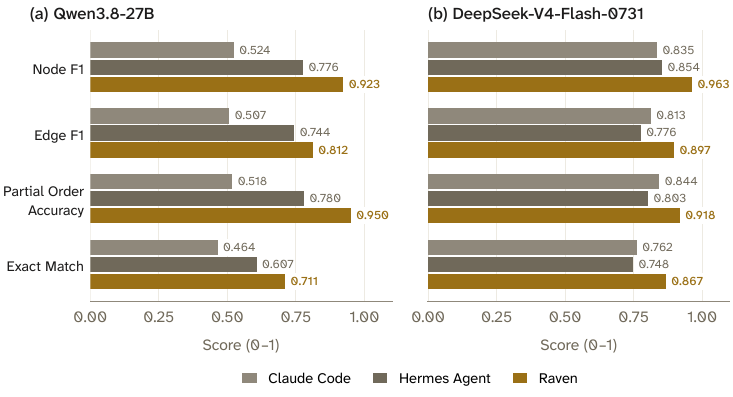}
\caption{Multi-Agent Orchestration Benchmark results with backbone
(a) \modelname{Qwen3.8-27B} and (b) \modelname{DeepSeek-V4-Flash-0731}.}
\label{fig:mao-results}
\end{figure}

\paragraph{Consistent Gains Across Metrics and Backbones.}
Raven outperforms both baseline systems on all four metrics with each
backbone, demonstrating effective specialist selection and dependency
planning across the two model families. With \modelname{Qwen3.8-27B},
Raven reaches $0.923$ Node F1 against $0.776$ for the strongest baseline,
a margin of $14.7$ points on a percentage scale, and obtains $0.950$ POA.
With \modelname{DeepSeek-V4-Flash-0731}, Raven reaches $0.963$ Node F1
and $0.897$ Edge F1, compared with $0.854$ and $0.813$ for the strongest
baseline on each metric. The gains cover both the choice of specialists
and the relationships that organize their work.

\paragraph{Exact Graph Agreement.}
The improvements also extend to complete plans, for which both the
specialist set and all precedence relations must match an accepted
reference. On \modelname{Qwen3.8-27B}, Raven achieves an Exact Match rate
of $0.711$, compared with $0.607$ for the strongest baseline, a
$10.4$-percentage-point improvement. On
\modelname{DeepSeek-V4-Flash-0731}, the corresponding rates are $0.867$
and $0.762$, a $10.5$-percentage-point improvement. These gains show that
Raven more often produces a complete dependency plan consistent with the
reference, complementing its advantages on individual nodes and edges.

\paragraph{Comparison Under Matched Backbones.}
Raven retains its leading position even as the relative ranking of the
baselines changes. Hermes Agent leads Claude Code on all four metrics
with \modelname{Qwen3.8-27B} but trails it on POA with
\modelname{DeepSeek-V4-Flash-0731}. All systems receive the same request
text, delegation instruction, and domain descriptions
(\S\ref{subsec:mao-setup}). Raven's consistent advantage with matched
backbones supports the effectiveness of its orchestration design across
both tested models.

\paragraph{Improved Dependency Prediction.}
Raven improves Edge F1 over the strongest baseline by $+0.068$ with
\modelname{Qwen3.8-27B} and $+0.084$ with
\modelname{DeepSeek-V4-Flash-0731}. Even with these gains, Edge F1 remains lower than Node F1
for every system under both backbones, indicating that dependency
prediction is a more difficult part of this benchmark. Raven's gains on
this metric show that its advantage extends to identifying which
specialists must exchange information and which subtasks depend on
earlier results.

\subsection{Harness Self-Evolution}\label{sec:eval-evolution}

We next examine harness self-evolution, which improves an agent by adapting the execution policy
around a frozen backbone. The published HarnessBank experiments
\citep{luo2026harnessbank} evaluate the diagnosis, search, and screening
procedure on tasks withheld from evolution.

\paragraph{Setup.}
The seven benchmarks are Terminal-Bench-2 \citep{merrill2026terminalbench},
five domains of EvoAgentBench \citep{gao2026evoagentbench} (LiveCodeBench,
Omni-MATH, BrowseComp+, GDPval, and SWE-bench Verified), and AppWorld
\citep{trivedi2024appworld}.
Figure~\ref{fig:harnessbank-source} presents held-out
Pass@1 results with a frozen \modelname{Qwen3.6-27B} backbone, averaged over
three attempts per task. The evolved harness is selected using the
training split and then assessed on the test split.

\begin{figure}[htbp]
\centering
\includegraphics[width=0.95\linewidth]{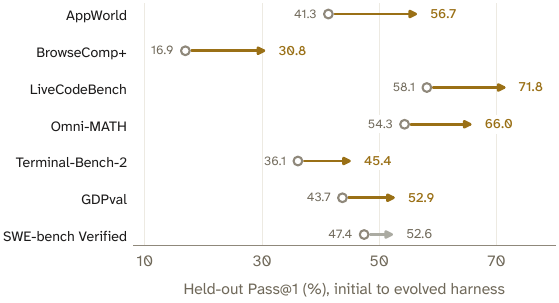}
\caption{Held-out Pass@1 reported in HarnessBank \citep{luo2026harnessbank}
for the initial and evolved harnesses with a frozen \modelname{Qwen3.6-27B} backbone.
Gray denotes SWE-bench Verified.
\Cref{tab:harnessbank-source} lists the values.}
\label{fig:harnessbank-source}
\end{figure}

\paragraph{Cross-Domain Gains with a Frozen Backbone.}
The evolved harness improves held-out Pass@1 on all seven benchmarks.
The largest gains are 15.4 percentage points on AppWorld, 13.9 on
BrowseComp+, and 13.7 on LiveCodeBench. Improvements also extend to
terminal operation, mathematical reasoning, professional tasks, and
repository-level issue resolution. Because the task model remains
frozen and the test tasks are withheld from evolution, these results
demonstrate that harness adaptation can improve execution across several
domains and transfer beyond the tasks used to guide the search. Six
comparisons pass the source's paired-gain criterion. SWE-bench Verified
also improves, although its 26-task test split does not pass that
criterion.

\paragraph{Mechanism Analysis.}
To examine the mechanisms behind these gains, we consider the reported
interventions, which address diagnosed execution failures, including
reasoning-budget overruns and premature finalization. The source links
these edits to observable events in the execution traces, providing a
behavioral explanation for how changes around a fixed model improve task
completion. Screening combines evidence that an intervention executed
with a paired comparison of task outcomes. This connects candidate
selection to both operational behavior and measured utility, while
held-out evaluation tests the benefit of the selected harness on new
tasks.

\subsection[Raven-Research]{Raven-Research: Evidence-Grounded Deep Research}
\label{sec:eval-research}
\label{sec:agents-research}

\ravenresearch answers questions using evidence from the live web. Its
harness addresses recurring failure modes in deep research: repeated
searches without reading retrieved pages, late restarts of an
investigation, turns that end without an answer, and unsupported claims.
To address these failures, its single plugin,
\texttt{research-flow}, replaces the built-in search, fetch, and
clarification tools. The search tool caches repeated queries and stops
once consecutive queries surface nothing new. The fetch tool extracts
the passages that answer an explicit information request and falls back
to alternative addresses for pages with little content. Checks at fixed phases of the agent loop can block or redirect execution.
They require page reads after extended search sequences, prevent late
restarts, and recover turns that end without an answer.
Before delivery, an independent reviewer, a separate model call without
the agent's context, checks the decisive claims of the draft against the
gathered evidence. The three session modes use different stopping rules under the same
resource budget. Each report gives a
full URL for every number, date, and quotation, and a research trail
computed from the tool-call record flags any citation whose page the
agent never opened.

\paragraph{Setup.}
We evaluate this workflow on DeepResearch Mixed, which combines questions from BrowseComp
\citep{wei2025browsecomp}, FRAMES \citep{krishna2025frames}, the
text-only exact-match setting of Humanity's Last Exam
\citep{phan2026hle}, and xBench-DeepSearch
\citep{chen2025xbench,xbench2025deepsearch}. An LLM judge grades each
short final answer for semantic equivalence with the reference, and
accuracy is pooled over all questions. We compare Raven-Research with
MiroFlow \citep{su2026miroflow} and DeepSeek-Harness
\citep{deepseek2026harness} on three shared backbones, \modelname{Qwen3.6-35B-A3B},
\modelname{Qwen3.5-397B-A17B}, and \modelname{DeepSeek-V4-Flash}, with the same
search and page-retrieval tools, and we include three commercial deep
research services as references
\citep{miromind2026mirothinker17,perplexity2025sonardr,google2026geminidr}.
The runs took place between August 10 and August 18, with each system
answering each question once.
\Cref{app:eval-research} gives the benchmark composition, grading, cost
accounting, and budgets.

\begin{table}[htbp]
\centering
\small
\caption{DeepResearch Mixed results. Accuracy (\%) is pooled over all
questions. Input tokens (including cache reads), output
tokens, and cost are means per question. A dash marks a quantity
that was not measured. Bold marks the best accuracy within each
shared-backbone group.}
\label{tab:research-results}
\begin{tabular}{lrrrr}
\toprule
\textbf{System} & \textbf{Accuracy} & \textbf{Input (M)} &
\textbf{Output (k)} & \textbf{Cost (USD)} \\
\midrule
\multicolumn{5}{l}{\emph{\modelname{Qwen3.6-35B-A3B}}} \\
MiroFlow           & 47.7 & 1.493 & 11.3 & -- \\
DeepSeek-Harness   & 49.7 & 4.587 & 16.2 & -- \\
Raven-Research     & \textbf{56.3} & 1.706 & 30.0 & -- \\
\midrule
\multicolumn{5}{l}{\emph{\modelname{Qwen3.5-397B-A17B}}} \\
MiroFlow           & 48.0 & 1.799 & 15.3 & -- \\
DeepSeek-Harness   & 56.0 & 0.679 & 6.6 & -- \\
Raven-Research     & \textbf{59.3} & 1.837 & 7.6 & -- \\
\midrule
\multicolumn{5}{l}{\emph{\modelname{DeepSeek-V4-Flash}}} \\
MiroFlow           & 67.2 & 3.737 & 39.3 & 0.0374 \\
DeepSeek-Harness   & 68.9 & 2.975 & 31.8 & 0.0211 \\
Raven-Research     & \textbf{76.5} & 2.380 & 41.6 & 0.0242 \\
\midrule
\multicolumn{5}{l}{\emph{Commercial API Reference Systems}} \\
MiroThinker-1.7 & 50.3 & \(\approx 1.247\) & \(\approx 18.5\) & 5.62 \\
Perplexity Sonar Deep Research & 52.0 & -- & 11.7 & 0.55 \\
Gemini Deep Research (2604) & 63.2 & 1.178 & 97.0 & \(\approx 1.67\) \\
\bottomrule
\end{tabular}
\end{table}

\begin{figure}[htbp]
\centering
\includegraphics[width=0.95\linewidth]{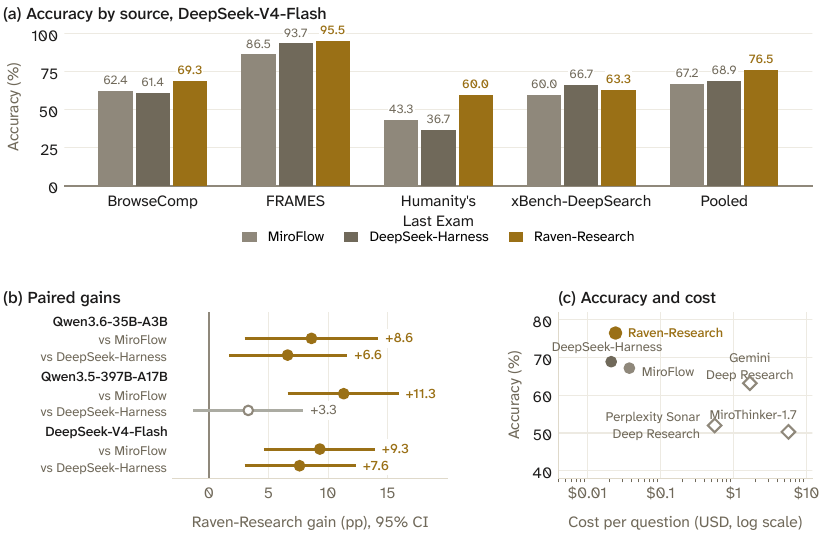}
\caption{Raven-Research results on DeepResearch Mixed. (a) Accuracy per
source under \modelname{DeepSeek-V4-Flash}. (b) Paired accuracy gain of Raven-Research over
each comparison on the same backbone, with stratified bootstrap 95\%
intervals over all questions. The hollow marker denotes the comparison
with DeepSeek-Harness on \modelname{Qwen3.5-397B-A17B}, with exact McNemar \(p=0.21\).
(c) Accuracy against mean cost per question,
with filled markers for the \modelname{DeepSeek-V4-Flash} harnesses and hollow markers for
the commercial services, whose costs follow each provider's own
accounting.}
\label{fig:research-results}
\end{figure}

\paragraph{Consistent Accuracy Gains Across Backbones.}
Raven-Research achieves the highest pooled accuracy with all three shared
backbones, exceeding the strongest baseline by 6.6, 3.3, and 7.6
percentage points, respectively (Table~\ref{tab:research-results}).
The consistent gains across the Qwen and DeepSeek models demonstrate the
effectiveness of the research harness across different backbones. With
\modelname{DeepSeek-V4-Flash}, Raven-Research reaches 69.3\% on
BrowseComp and 60.0\% on Humanity's Last Exam, compared with at most
62.4\% and 43.3\% for the other two harnesses
(Figure~\ref{fig:research-results}a). The larger margins on these two
benchmarks account for an important part of the overall advantage.
All three harnesses exceed 86\% on FRAMES, while
DeepSeek-Harness leads Raven-Research by 3.4 percentage points on
xBench-DeepSearch. Raven-Research nevertheless obtains the strongest
aggregate performance across the combined task set.

\paragraph{Paired Statistical Comparisons.}
Paired comparisons reinforce the overall accuracy advantage. Exact
McNemar tests over all questions give \(p<0.02\) for five of the six
same-backbone comparisons (Figure~\ref{fig:research-results}b). The
remaining comparison, against DeepSeek-Harness with
\modelname{Qwen3.5-397B-A17B}, shows a positive 3.3-percentage-point
difference with \(p=0.21\). Thus, the numerical gains are consistent
across all six comparisons, with statistical support in five.

\paragraph{Higher Accuracy at Comparable Inference Cost.}
With \modelname{DeepSeek-V4-Flash}, Raven-Research combines the highest
accuracy with a mean cost of 0.0242 USD per question, between the costs
of the two locally evaluated baselines. It uses fewer input tokens and
more output tokens than both, achieving better answers without incurring
the highest inference cost. Raven-Research also outperforms the three
commercial reference services, whose accuracies range from 50.3\% to
63.2\% (Figure~\ref{fig:research-results}c). However, those services use their own
models and search stacks, with provider-specific cost accounting
(\cref{app:eval-research}). Accordingly, the shared-backbone comparison provides the
clearest evidence that Raven's research workflow offers an effective
balance between answer quality and inference cost.

\subsection[Raven-Code]{Raven-Code: Repository-Aware Software Engineering}
\label{sec:eval-code}
\label{sec:agents-code}

\ravencode writes, runs, and debugs code in the working directory assigned
by the Host Agent. Its harness addresses errors arising from incomplete
repository context, including edits to stale file contents, changes that
violate project conventions, and completion claims unsupported by tests.
To address these errors, its plugin, \texttt{code-flow}, adds the repository's own instruction
files, such as \texttt{AGENTS.md} and \texttt{CLAUDE.md}, to every model
call and warns the agent when other sessions share the directory. The
tool interface follows the conventions of common coding assistants and
adds safeguards. An edit is refused unless the session has read the
current version of the file, every Python write returns a syntax check,
and the shell returns partial output when a command times out. A set of coding
instructions, selected by model family, asks the agent to reproduce a
failure before editing, make minimal root-cause fixes, and prefer the project's
own tests as evidence. Each task ends with a harness manifest derived from repository state,
including the base commit, uncommitted files, and the diff. This provides
an execution record independent of the model's account of its work.

\paragraph{Setup.}
We evaluate issue resolution on SWE-bench Pro
\citep{deng2025swebenchpro} (731 tasks) and SWE-bench Verified
\citep{openai2024sweverified,jimenez2024swebench} (500 tasks),
repository-level development on the 80-task code subset of WorkBuddy
Bench \citep{cai2026workbuddybench}, and whole-repository stack
migration on SWE-Refactor \citep{hong2026swerefactorbench} (20 tasks).
For systems evaluated locally, we use AgentEval, an internal framework
that runs each attempt in a disposable container and applies the
benchmark's official grading code. Patches are generated without network access on
SWE-bench Pro and with network access on SWE-bench Verified. Hollow
markers in \cref{fig:code-results} are public leaderboard values
\citep{valsai2026swebench,workbuddy2026leaderboard,einsia2026srbleaderboard}.
\Cref{app:eval-code} gives the versions, budgets, and accounting.

\begin{figure}[htbp]
\centering
\includegraphics[width=0.95\linewidth]{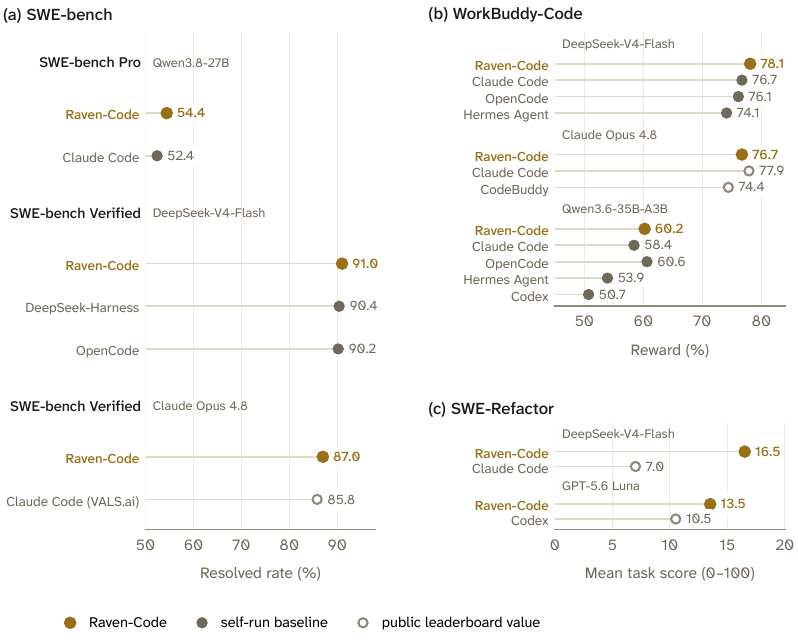}
\caption{Coding results. Each row is one system, grouped by benchmark and
backbone. SWE-bench reports the resolved rate, WorkBuddy-Code the reward,
and SWE-Refactor the mean task score. Hollow markers are public
leaderboard values.}
\label{fig:code-results}
\end{figure}

\paragraph{Strong Performance on Complex Repository Tasks.}
Raven-Code achieves the highest score in six of the eight settings in
Figure~\ref{fig:code-results}, with its largest margins on
whole-repository migration. On SWE-Refactor, it exceeds the compared
leaderboard entries by 9.5 points with \modelname{DeepSeek-V4-Flash}
and 3.0 points with \modelname{GPT-5.6 Luna}. Each task requires migrating
an entire repository to a new stack over several hours. The gains
demonstrate Raven-Code's effectiveness on sustained development tasks
that require coordinated changes across files. This result is consistent
with the harness's emphasis on repository instructions, current file
contents, and execution-based verification.

\paragraph{Improved Issue Resolution.}
Raven-Code also improves resolution of existing repository issues. On
SWE-bench Pro, it resolves 15 more of the 731 tasks than Claude Code with
the same \modelname{Qwen3.8-27B} backbone. On SWE-bench Verified, its
advantages of 0.6 points with \modelname{DeepSeek-V4-Flash} and 1.2
points with \modelname{Claude Opus 4.8} correspond to 3 and 6 additional solved
tasks out of 500. These results extend the gains from repository
migration to targeted bug fixing. The Verified runs permit network
access during generation (\cref{app:eval-code}).

\paragraph{Competitive Repository Development.}
On WorkBuddy-Code, Raven-Code outperforms all three locally evaluated
baselines with \modelname{DeepSeek-V4-Flash}, by margins of 1.4 to 4.0
points. It remains close to the strongest alternatives with the other
backbones: the Claude Code leaderboard entry leads by 1.2 points with
\modelname{Claude Opus 4.8}, and OpenCode leads by 0.4 points with
\modelname{Qwen3.6-35B-A3B}, where the three strongest systems are within
2.2 points. Raven-Code therefore remains competitive across the tested
development settings, alongside its larger advantages on migration.

\paragraph{High Accuracy on Database Analysis.}
Beyond software modification, we assess database analysis using
DataAgentBench \citep{ma2026dab}. This benchmark asks an agent to answer 54
natural-language questions over 12 datasets stored in PostgreSQL,
MongoDB, SQLite, and DuckDB. Following the leaderboard protocol with five
trials per query, Raven-Code with \modelname{Claude Opus 5} as its main model reaches a
Pass@1 of 0.8762 and a Pass@5 of 0.9097. On August 24, 2026, this was the
highest Pass@1 among the entries in Figure~\ref{fig:data-results}, with
the other entries taken from the public leaderboard on that date. This
extends the evidence for Raven-Code beyond software modification to
answering analytical questions over heterogeneous database systems.

\begin{figure}[htbp]
\centering
\includegraphics[width=0.95\linewidth]{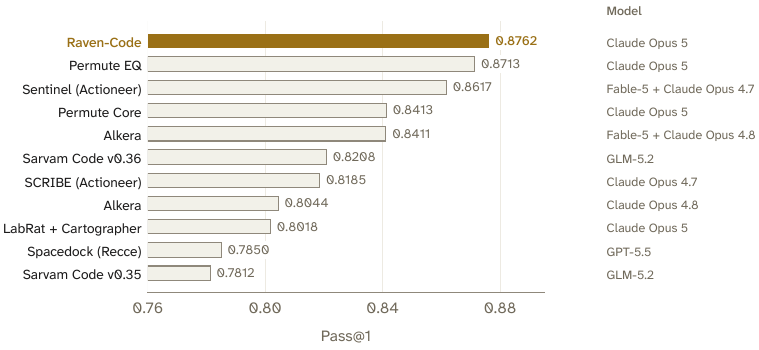}
\caption{DataAgentBench Pass@1 on August 24, 2026, with the horizontal
axis starting at 0.76. The comparison entries are the public leaderboard
values on that date, and the right column gives their models.}
\label{fig:data-results}
\end{figure}

\subsection[Raven-Design]{Raven-Design: Render-and-Inspect Visual Design}
\label{sec:eval-design}
\label{sec:agents-design}

Visual defects such as overflowing text, obscured chart labels, and poor
layout may become apparent only after rendering. \ravendesign therefore
incorporates rendering and inspection into the production of slide decks,
charts, diagrams, and web interfaces. A single entry in the Host Agent's registry provides access to two
implementations. The design engine selects
the relevant packaged domain skills for each request, maintains a persistent Task State containing goals and requirements, and renders
HTML, SVG, and Office outputs into pages and previews that the model
inspects before delivery. Requests based on a user-provided template are routed internally to
Raven-PPT. This deck engine specifies each slide's main claim in an
outline and builds the deck using python-pptx. Each build is checked for
68 types of issue, nine of which block publication. Any slide not yet
inspected by the agent is sent to a second reviewer with an empty context.

\paragraph{Slide Generation Setup.}
PresentBench \citep{chen2026presentbench} contains 238
slide-generation tasks from five domains, each with an instruction and
background materials. The agent produces an editable PPTX file, and the
official evaluator scores it from 0 to 100 over five equally weighted
dimensions: presentation fundamentals, visual design and layout,
content completeness, content correctness, and content fidelity. Reported
means cover successfully graded decks, and task coverage varies across
configurations.
\Cref{app:eval-design} gives the protocol and coverage of each
configuration.

\begin{figure}[htbp]
\centering
\includegraphics[width=0.95\linewidth]{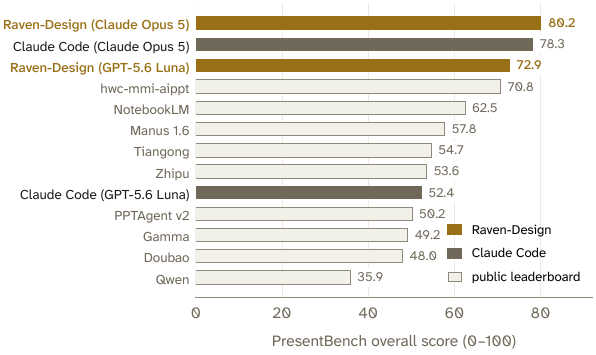}
\caption{PresentBench overall scores (0--100). Gold bars are Raven-Design and
dark bars Claude Code, with the backbone in parentheses. Light bars are
public leaderboard entries evaluated by the benchmark authors, whose model
configurations are not given.}
\label{fig:present-results}
\end{figure}

\paragraph{Improved Slide Generation.}
Raven-Design produces higher-scoring decks than Claude Code with both
tested backbones. It scores 80.2 with \modelname{Claude Opus 5} and 72.9 with
\modelname{GPT-5.6 Luna}, compared with 78.3 and 52.4 for Claude Code
(Figure~\ref{fig:present-results}). The gains of 1.9 and 20.5 points
show the benefit of a specialized slide-generation harness, with a
particularly large improvement for \modelname{GPT-5.6 Luna}. Both
Raven-Design scores also exceed every public leaderboard entry shown in
the figure. Those entries provide product-level reference scores, with
unspecified models and evaluation on all 238 tasks. The results support
the effectiveness of structuring slide generation around explicit layouts
and artifact-level evaluation.

\paragraph{Visual Artifact Setup.}
Beyond slide decks, the evaluation includes SVG generation and
data-visualization dashboard tasks from ArtifactsBench
\citep{zhang2025artifactsbench}, while GDPval
\citep{patwardhan2025gdpval} contributes professional deliverables
such as spreadsheets, reports, and slide decks. An internal grader
renders each deliverable, replays interactions where the task requires
them, and scores the benchmark's checklist or rubric from 0 to 100.

\begin{figure}[htbp]
\centering
\includegraphics[width=0.95\linewidth]{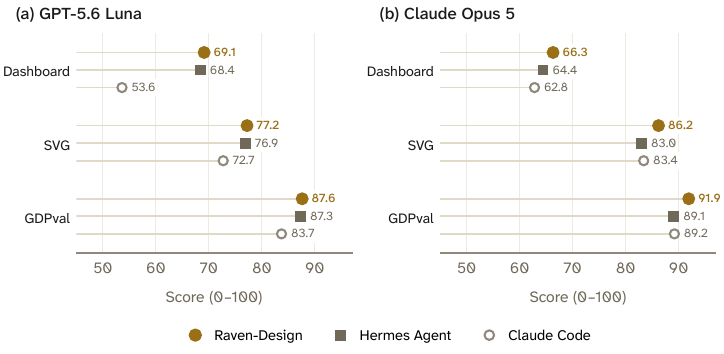}
\caption{Visual-design scores (0--100) under (a) \modelname{GPT-5.6 Luna} and (b) \modelname{Claude Opus 5}.
Scores come from an internal grader with \modelname{GPT-5.6 Luna} as the judge and are
not comparable with the official leaderboards.}
\label{fig:design-results}
\end{figure}

\paragraph{Consistent Gains Across Visual Tasks.}
Raven-Design achieves the highest score in all six benchmark and backbone
settings (Figure~\ref{fig:design-results}), extending its advantages to
SVG generation, interactive dashboards, and professional deliverables.
It leads the strongest alternative, Hermes Agent in four settings and Claude
Code in two, by 0.3 to 2.8 points. The largest gap against an individual
baseline is 15.5 points over Claude Code on Dashboard tasks with
\modelname{GPT-5.6 Luna}. Although the margins over the strongest
alternative are modest, their consistency across task types and both
backbones supports the broader utility of Raven-Design's visual workflow.

\paragraph{Case Study Setup.}
To complement the benchmark scores with examples of generated deliverables,
we examine seven runs from the
showcase in the Raven repository
(\texttt{docs/showcase.md}): four slide decks, two posters, and one
interactive website (\cref{fig:design-cases}). In each run, Raven planned a task graph from the user's request, ending
with a Raven-Design node that consumed outputs from other specialists (\cref{tab:design-cases}).
These selected cases provide qualitative examples of completed
collaborative workflows, without human preference ratings.

\begin{figure}[t]
\centering
\begin{minipage}{0.95\linewidth}
\centering
\begin{subfigure}[t]{0.2325\linewidth}
\centering
\includegraphics[height=1.88\linewidth]{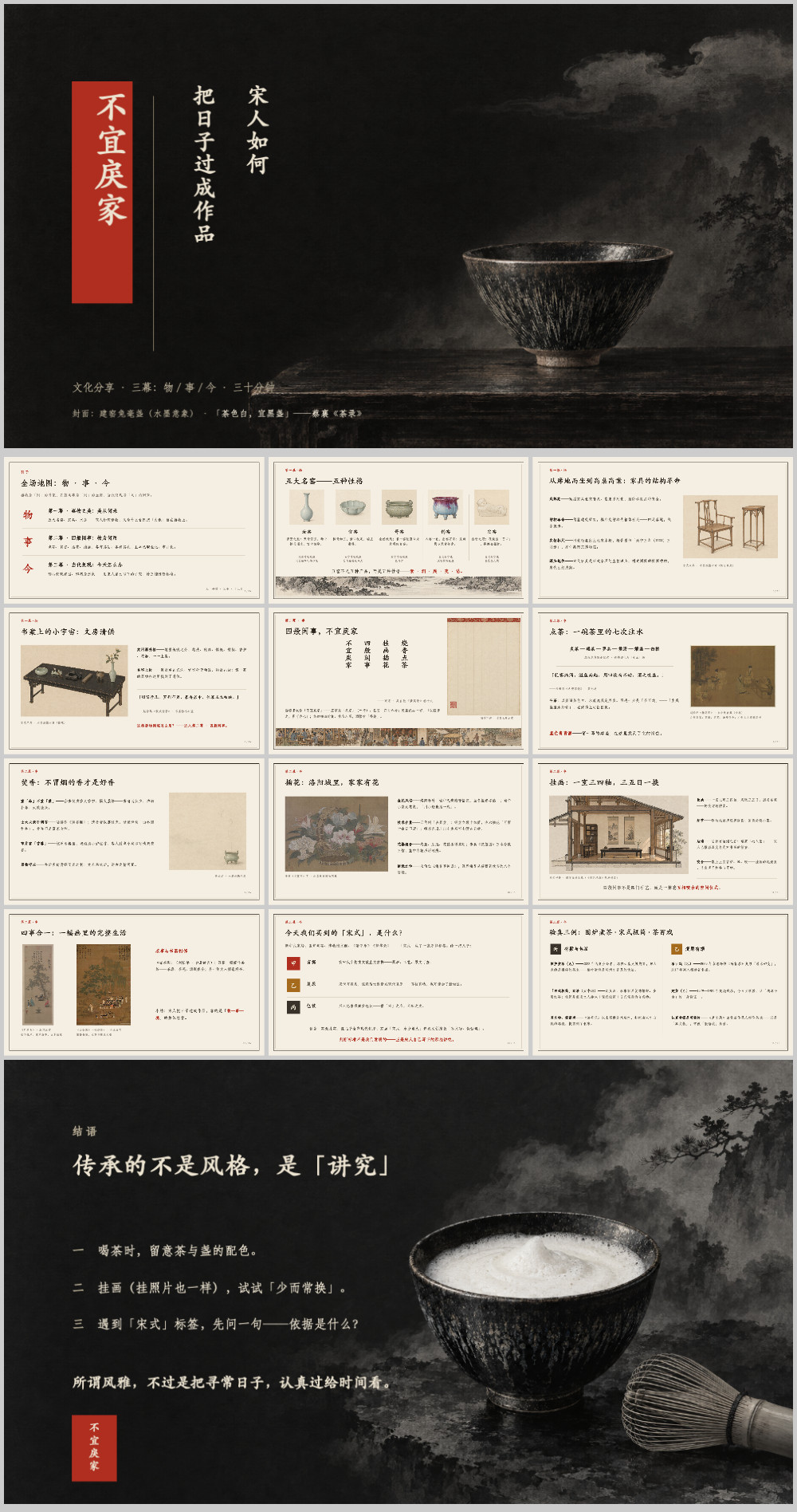}
\caption{Song Dynasty}
\label{fig:design-case-song}
\end{subfigure}\hfill
\begin{subfigure}[t]{0.2325\linewidth}
\centering
\includegraphics[height=1.88\linewidth]{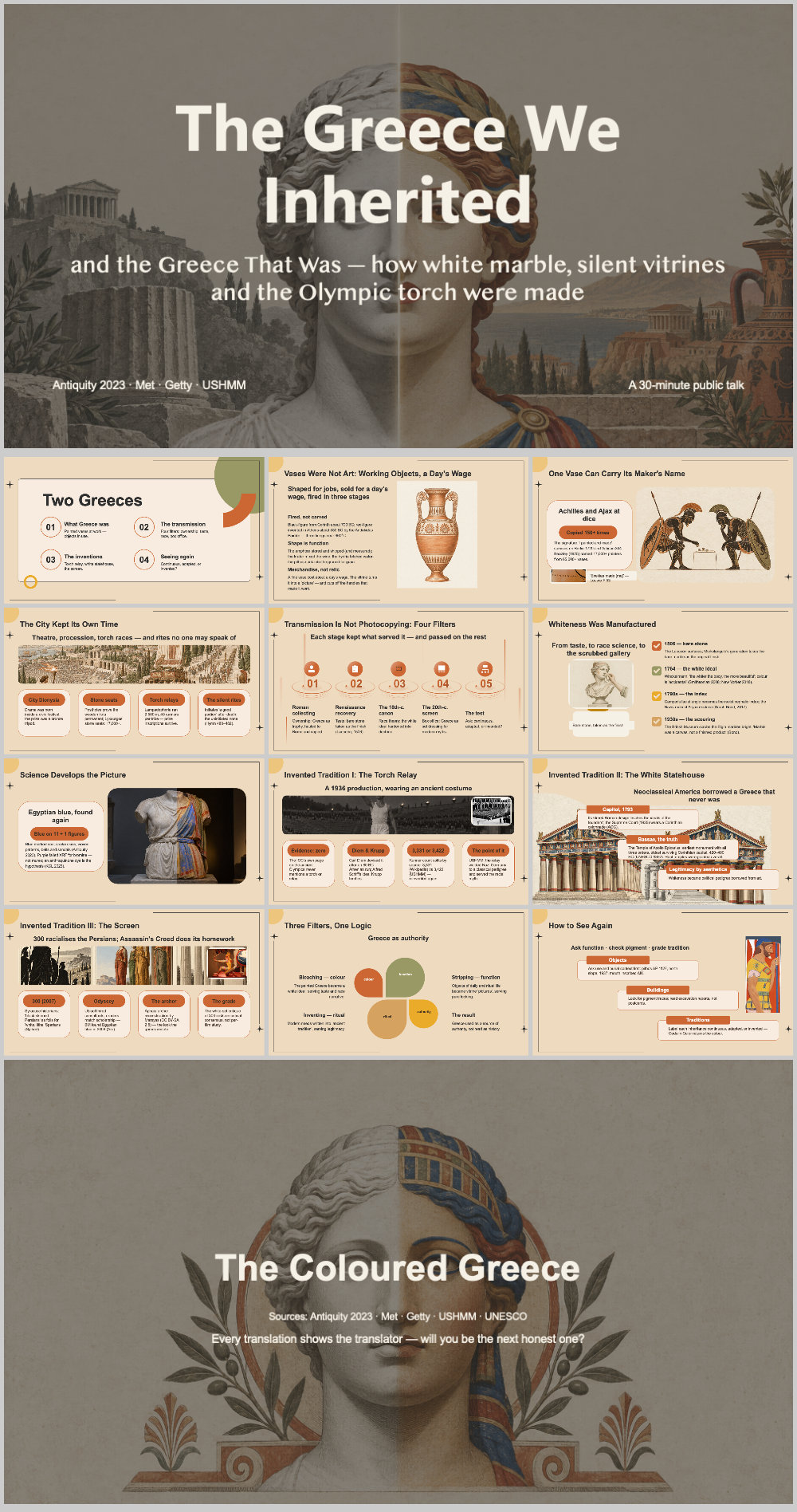}
\caption{Greek Polychromy}
\label{fig:design-case-greek}
\end{subfigure}\hfill
\begin{subfigure}[t]{0.2325\linewidth}
\centering
\includegraphics[height=1.88\linewidth]{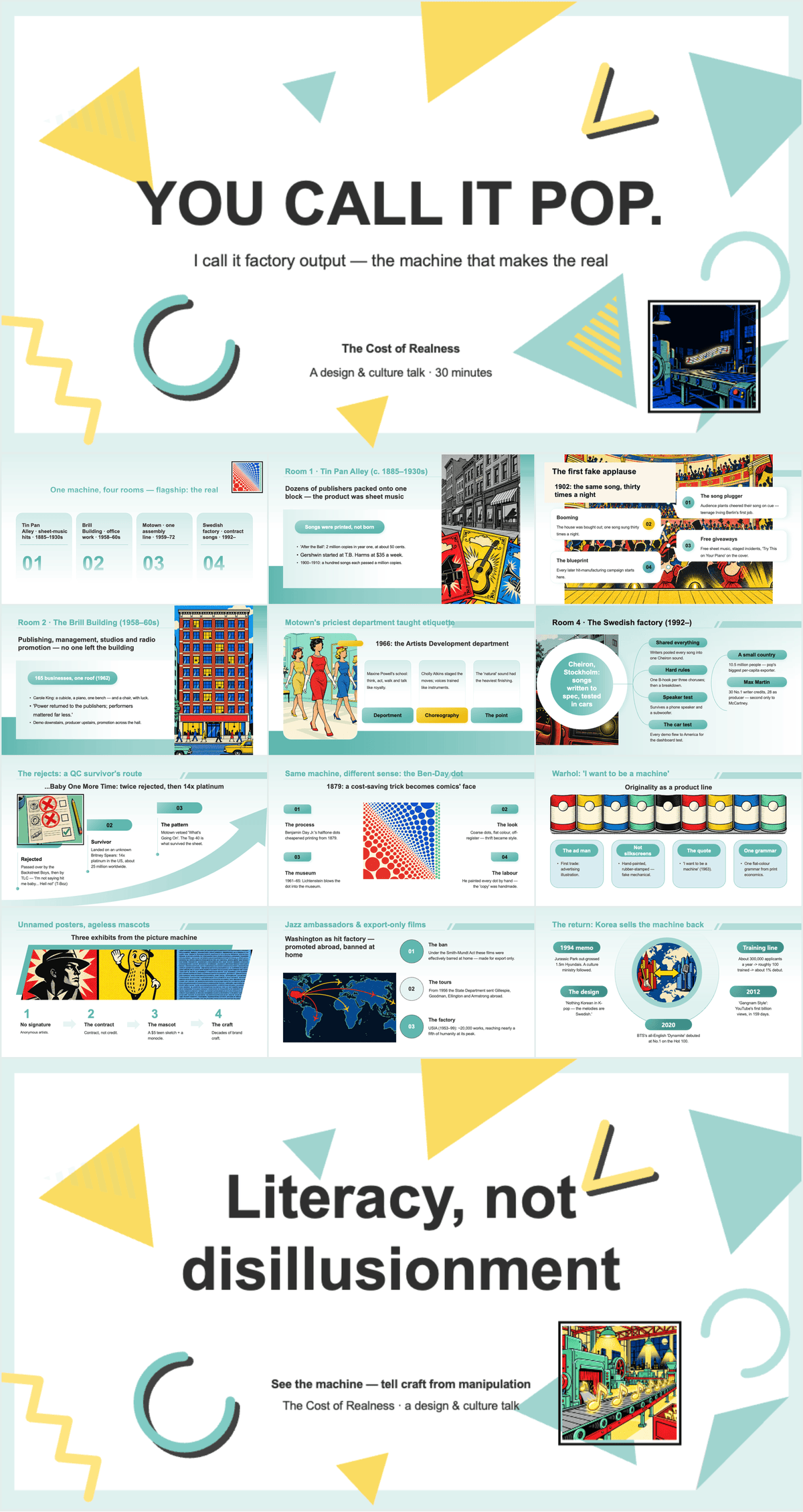}
\caption{Pop Music}
\label{fig:design-case-pop}
\end{subfigure}\hfill
\begin{subfigure}[t]{0.2325\linewidth}
\centering
\includegraphics[height=1.88\linewidth]{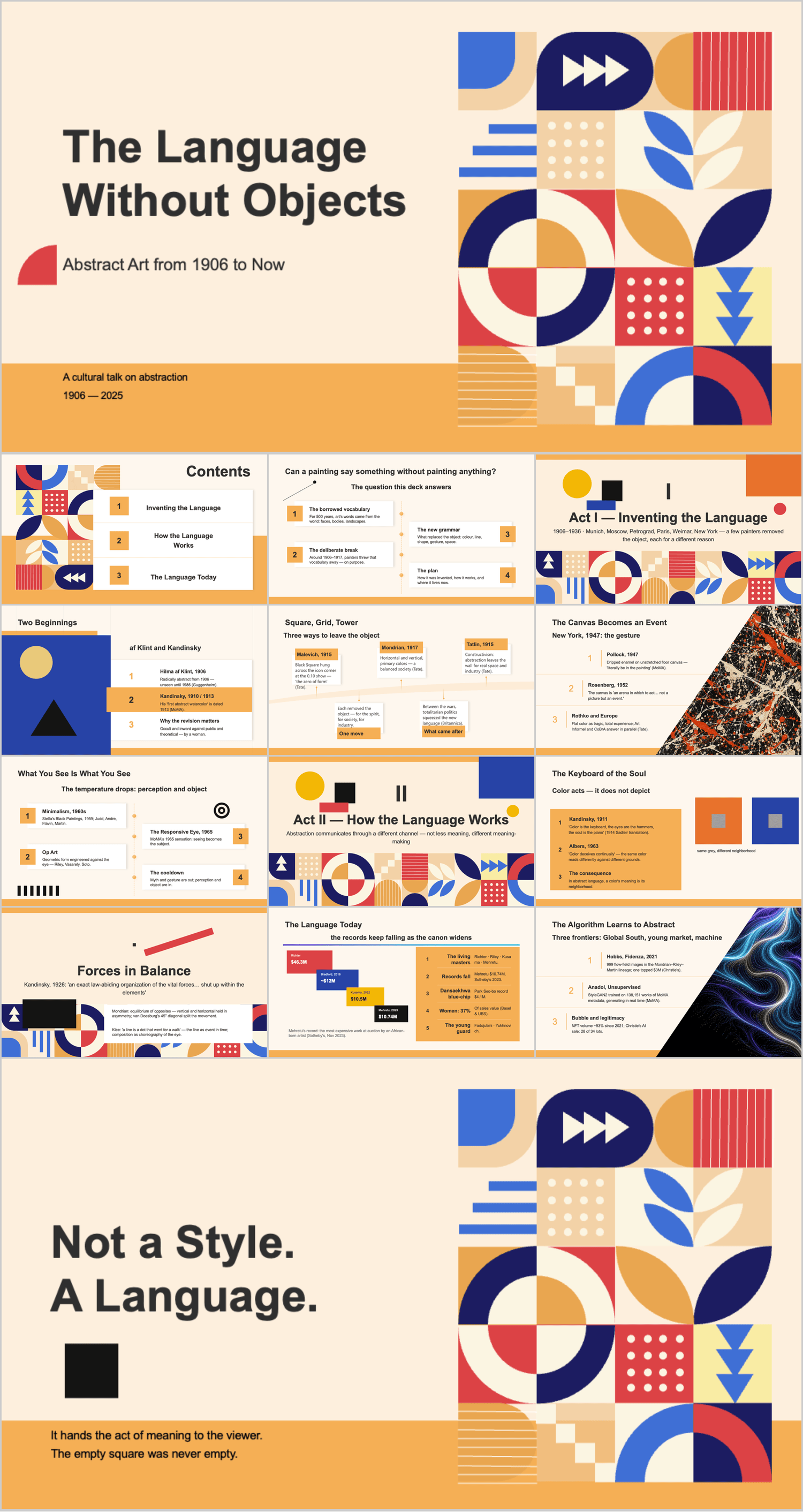}
\caption{Abstract Art}
\label{fig:design-case-abstract}
\end{subfigure}

\vspace{0.8em}

\begin{subfigure}[t]{0.3496\linewidth}
\centering
\includegraphics[width=\linewidth]{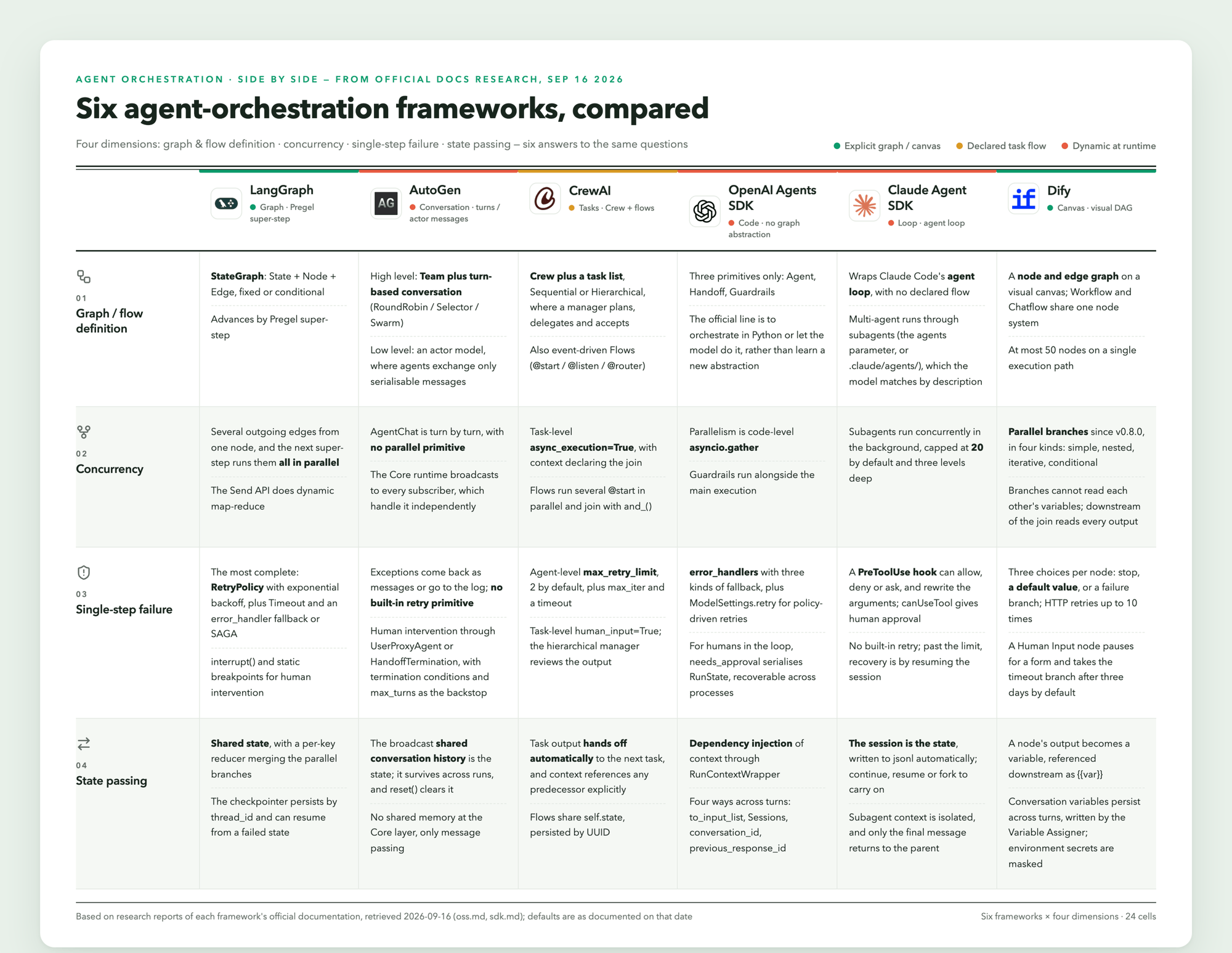}
\caption{Orchestration Frameworks}
\label{fig:design-case-frameworks}
\end{subfigure}\hfill
\begin{subfigure}[t]{0.2941\linewidth}
\centering
\includegraphics[width=\linewidth]{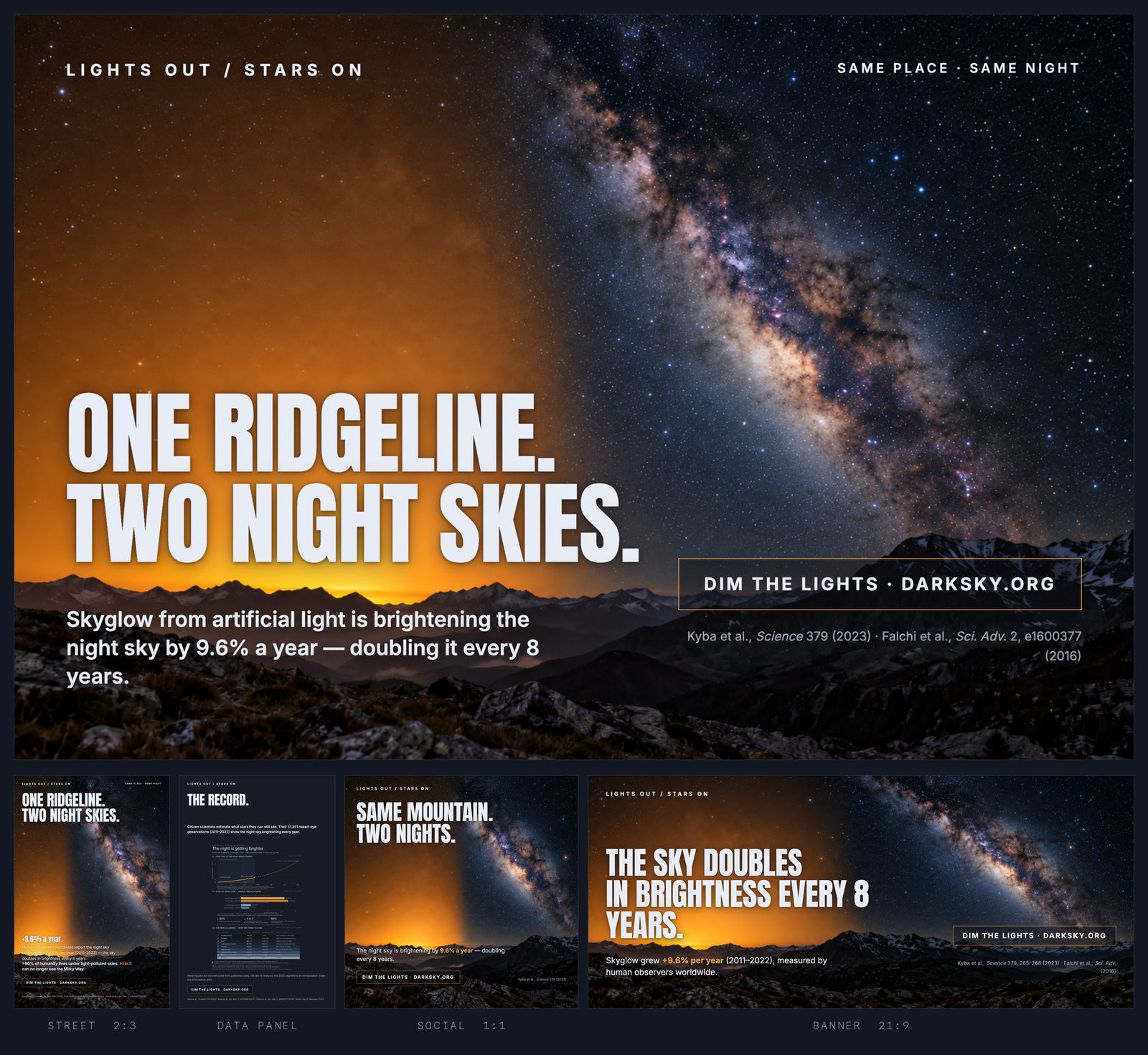}
\caption{Light Pollution}
\label{fig:design-case-light}
\end{subfigure}\hfill
\begin{subfigure}[t]{0.2864\linewidth}
\centering
\includegraphics[width=\linewidth]{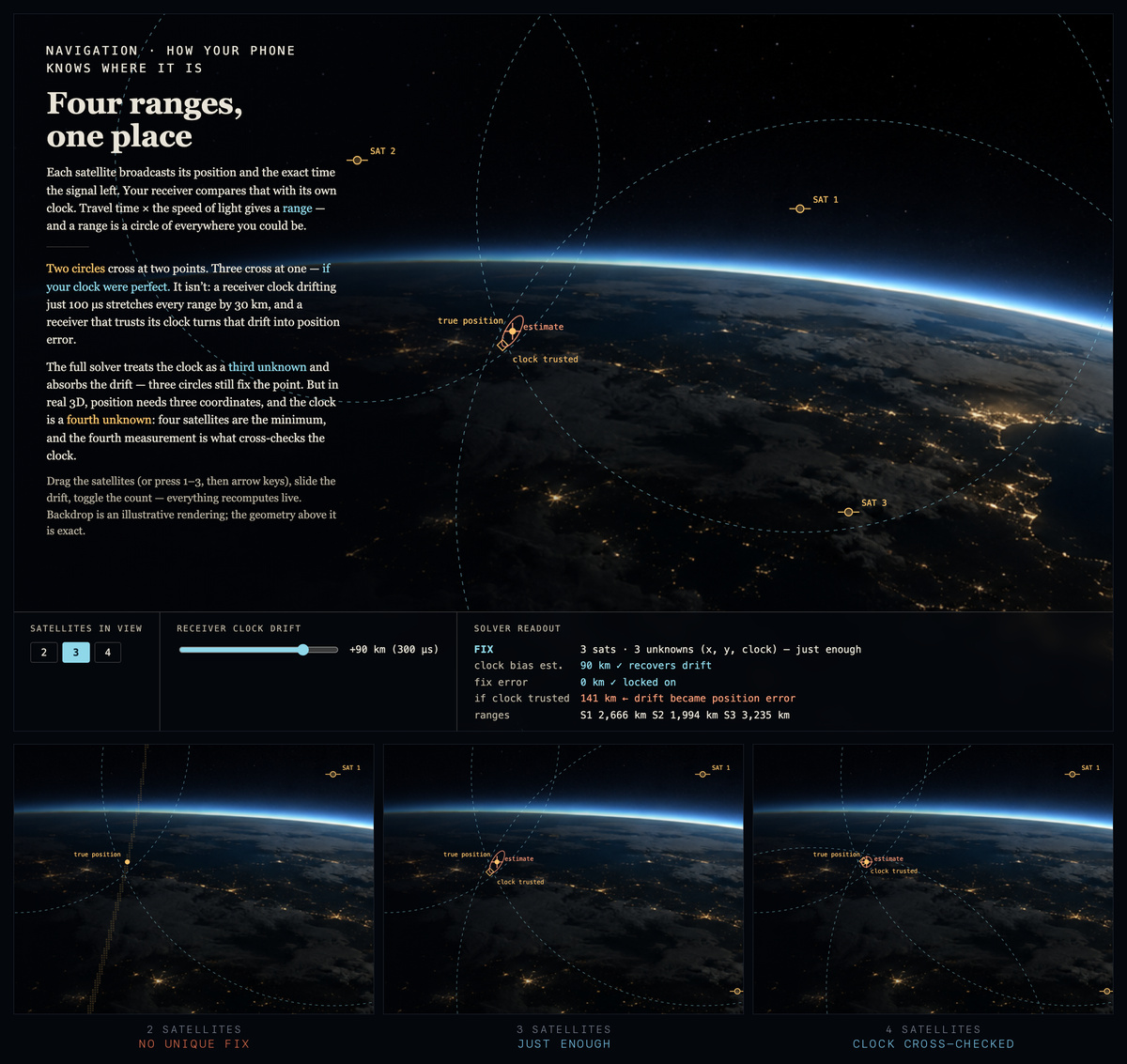}
\caption{GPS Trilateration}
\label{fig:design-case-gps}
\end{subfigure}
\end{minipage}
\caption{Deliverables of the Raven-Design case studies. Panels (a)--(d)
show the cover, twelve interior slides, and the closing slide of each
deck. Panel (e) shows the complete board, which compares six frameworks
on four dimensions. Panel (f) shows the key visual above its four
derived formats, and panel (g) shows the live page above the position fix with two, three,
and four satellites. \Cref{tab:design-cases} gives the task graph of
each run.}
\label{fig:design-cases}
\end{figure}

\begin{table}[!ht]
\centering
\small
\caption{Task graphs of the case studies in \cref{fig:design-cases}.
R, C, O, and D denote Raven-Research, Raven-Code, Raven-Oncall, and
Raven-Design nodes, a numeral counts parallel nodes of one kind, $\to$
marks a dependency, and $\parallel$ marks branches that run
concurrently. Design time is the wall-clock time of each Raven-Design
node in the task graph, with the image-plate node first where a graph
has two.}
\label{tab:design-cases}
\begin{tabular}{lllr}
\toprule
\textbf{Case} & \textbf{Deliverable} & \textbf{Task Graph} &
\textbf{Design Time (min:s)} \\
\midrule
(a) Song dynasty & Deck (Chinese) &
$3\mathrm{R} \to \mathrm{R} \to \mathrm{D}$ & 18:45 \\
(b) Greek polychromy & Deck (English) &
$3\mathrm{R} \to \mathrm{R} \to \mathrm{D}$ & 55:34 \\
(c) Pop music & Deck (English) &
$3\mathrm{R} \to \mathrm{R} \to \mathrm{D}$ & 47:03 \\
(d) Abstract art & Deck (English) &
$3\mathrm{R} \to \mathrm{R} \to \mathrm{D}$ & 64:50 \\
\midrule
(e) Orchestration frameworks & Comparison board &
$2\mathrm{R} \to \mathrm{D}$ & 6:57 \\
(f) Light pollution & Poster series &
$[(3\mathrm{R} \to \mathrm{C}) \parallel \mathrm{D}] \to \mathrm{D}$ &
2:58, 13:57 \\
(g) GPS trilateration & Interactive website &
$[(2\mathrm{C} \to \mathrm{O}) \parallel \mathrm{D}] \to \mathrm{D}$ &
3:03, 17:03 \\
\bottomrule
\end{tabular}
\end{table}

\paragraph{A Distinct Visual Language for Each Deck.}
Although the four decks share the same task graph, they adopt visual designs
suited to their subjects. In each run, three
Raven-Research nodes study separate aspects of the topic in parallel, a
fourth synthesizes their reports, and Raven-Design builds the deck. The
Chinese deck on Song-dynasty domestic aesthetics sets a vertical title
on ink-wash photography beside a red seal-style block
(\cref{fig:design-case-song}). The deck on Greek polychromy splits its
cover portrait into a white-marble half and a painted half, visually introducing the presentation's central claim
(\cref{fig:design-case-greek}).
The deck on pop music uses pop-art illustrations, and the deck on
abstract art is built from flat geometric tiles
(\cref{fig:design-case-pop,fig:design-case-abstract}). Within each deck, the interior slides retain a consistent palette and
title position while using timelines, card rows, numbered lists, and
image-led layouts. The Raven-Design node took between 19 and 65 minutes per deck,
and the showcase lists a cost of about US\$0.80 for each deck. The
shared workflow accommodates different subjects and visual styles while
preserving consistency within each presentation.

\paragraph{Composing the Work of Other Specialists.}
The posters and website illustrate how Raven-Design incorporates evidence
and computations from other specialists. The comparison board of six
orchestration frameworks is built from two Raven-Research reports on
their official documentation, and it arranges 24 cells, one for each
framework and dimension, under a three-color key for how each framework
defines its flow (\cref{fig:design-case-frameworks}). For the
light-pollution campaign, a Raven-Code node computes the figures for
the data panel from three research reports while a separate Raven-Design
node generates the key visual. A final Raven-Design node then composes
the key visual and four derived formats, a 2:3 street poster, the data
panel, a 1:1 social square, and a 21:9 banner
(\cref{fig:design-case-light}). Both posters print their sources. For
the website, two Raven-Code nodes write the GPS trilateration
computation and an independent counterpart, Raven-Oncall cross-checks
them, and Raven-Design builds a page on which the reader drags
satellites, changes the receiver clock drift, and sees the position fix
recomputed (\cref{fig:design-case-gps}). The page labels its Earth
backdrop as illustrative and its geometry as exact. In the last two
graphs the Host Agent assigns Raven-Design both a narrow subtask, the
image plate, and the final composition. These cases demonstrate how
Raven composes research, computation, verification, and design into
complete visual deliverables. The task graph also allows a specialist to
contribute at different stages, according to the artifact dependencies
of the request.

\subsection[Raven-Oncall]{Raven-Oncall: Unattended Run-and-Watch Automation}
\label{sec:eval-oncall}
\label{sec:agents-oncall}

\ravenoncall manages long-running jobs on registered machines, including
training runs, solver cases, and parameter sweeps. Such jobs require
monitoring for divergence, efficient waiting between state changes, and
final reports grounded in recorded measurements. To meet these requirements,
each job is represented as a campaign, a persistent record containing its
target, measurements, permitted actions, and budget. The agent submits
one round at a time and registers a wake-up with Raven's proactivity
scheduler under the campaign identifier. It then waits without issuing
model calls. A monitor that operates without model calls wakes the agent
early if the round finishes or diverges. The tools enforce the campaign's
operating constraints. Every decision to wait, terminate, or resubmit must cite a recent
measurement, compute requests are
clamped to the remaining budget, and machine capacity is allocated
across campaigns. A campaign closes only through a report gate that
rejects reports without a baseline and claims contradicted by the
recorded state.

\paragraph{AI4AI Setup.}
The AI4AI task is autoresearch \citep{karpathy2026autoresearch}, a
public single-GPU derivative of nanochat \citep{karpathy2025nanochat}
in which an agent improves the pretraining recipe of a GPT model with
about 50M parameters. Each training run receives 300 seconds on one GPU,
each campaign has a training budget of 5 GPU-hours on two A800 80GB
GPUs, and the metric is validation bits per byte (BPB), averaged over
several seeds of the final configuration. Runtime is the wall-clock
time of the whole campaign, including work outside the training runs. The unmodified starting script reaches 1.108 BPB.
Each system ran one campaign. Claude Code received a separately written
task statement with the same budget and constraints.
\Cref{app:eval-oncall} gives the rules and accounting.

\begin{figure}[htbp]
\centering
\includegraphics[width=0.95\linewidth]{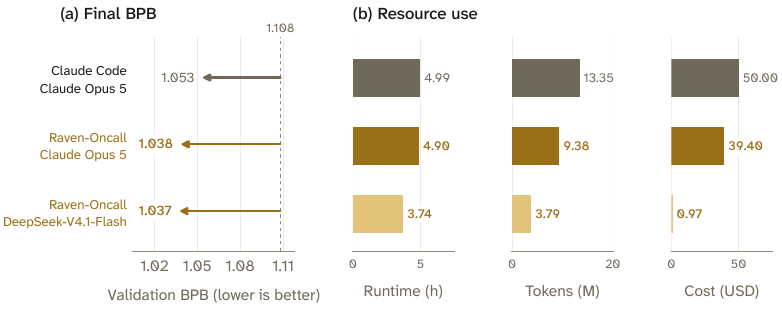}
\caption{AI4AI nanochat pretraining. (a) Final validation BPB of each campaign,
starting from 1.108 (dashed line). (b) Runtime, reported tokens, and USD cost.
Lower is better throughout. The lighter bars use the DeepSeek backbone.}
\label{fig:ai4ai-results}
\end{figure}

\paragraph{Lower BPB at Lower Cost Under \modelname{Claude Opus 5}.}
Raven-Oncall achieves better pretraining quality with lower reported token
use and cost under the shared \modelname{Claude Opus 5} backbone. It reaches
1.038 BPB, compared with 1.053 for Claude Code, in a similar runtime of
4.90 versus 4.99 hours (Figure~\ref{fig:ai4ai-results}). From the
starting value of 1.108, these results correspond to reductions of 0.070
and 0.055. With the same training budget, Raven-Oncall therefore finds a
better final configuration while spending fewer model tokens on managing
the campaign. With \modelname{DeepSeek-V4.1-Flash}, it reaches 1.037 BPB
in 3.74 hours at a reported cost of 0.97 USD, demonstrating that effective
campaign execution is also possible with a low-cost backbone.

\paragraph{AI4S Setup.}
To evaluate a broader range of sustained tasks, we use the AI4S benchmark,
an internal set of 17 automatic scientific
research tasks. The tasks execute solvers and training jobs on registered machines, covering structural
mechanics, computational fluid dynamics, molecular simulation, numerical
computation, LLM batch inference, embedding-model training, retrieval
tuning, and platform monitoring. Success rate is the number of
successful tasks divided by 17, and runtime, tokens, and cost are
averages over the tasks. \Cref{app:eval-oncall} describes the tasks and
the success judgment.

\begin{figure}[htbp]
\centering
\includegraphics[width=0.95\linewidth]{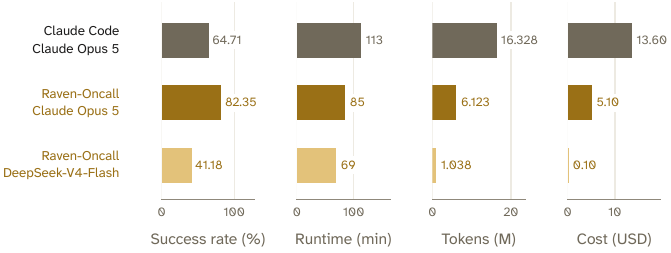}
\caption{AI4S internal benchmark. Success rate on the 17 tasks, and
runtime, tokens, and USD cost averaged per task. The lighter bars use
the DeepSeek backbone.}
\label{fig:ai4s-results}
\end{figure}

\paragraph{Higher Task Success with Lower Resource Use.}
Raven-Oncall solves 14 of the 17 scientific tasks with
\modelname{Claude Opus 5}, compared with 11 for Claude Code, while reducing
average runtime, token use, and cost (Figure~\ref{fig:ai4s-results}).
Completing more tasks with lower resource use demonstrates the practical
value of its campaign-based approach to sustained execution. With
\modelname{DeepSeek-V4-Flash}, it solves 7 tasks at an average cost of
0.10 USD, offering a lower-cost operating point. Success is judged
manually against task-specific reference values and supporting evidence
withheld from the agent (\cref{app:eval-oncall}). Together, the training
and scientific-task results support Raven-Oncall's effectiveness in
managing real jobs through completion under explicit resource budgets.

\subsection{Skill Retrieval and Reuse}\label{sec:eval-skills}

We evaluate how retrieved procedures improve task performance with the
backbone and harness held fixed. The published SkillCorpus experiments
\citep{wang2026skillcorpus} attach the
curated catalog of \cref{sec:skillforge-hub} and its reference
retrieval-and-selection stack to two agent harnesses, one of which is
Raven. The skill library stays fixed throughout these experiments,
allowing us to assess both the benefit of reusable procedural knowledge
and the harness's ability to translate that knowledge into successful
execution.

\paragraph{Setup.}
The evaluation covers 407 tasks from three third-party benchmarks.
SkillsBench \citep{li2026skillsbench} contributes 87 tasks across eight
domains with deterministic verifiers. GDPval \citep{patwardhan2025gdpval}
contributes 220 professional tasks graded by an LLM judge, and
QwenClawBench \citep{qwen2026qwenclawbench} contributes 100 tasks drawn
from a real-user distribution and graded by automated checks together
with an LLM judge. Raven and OpenClaw \citep{steinberger2025openclaw},
two independently developed harnesses that load \texttt{SKILL.md}
skills, are each paired with \modelname{Qwen3.5-27B} and \modelname{Qwen3.5-397B-A17B}.
Each of the four harness--backbone cells attempts every task
under a no-skill baseline and with SkillCorpus. In the latter condition,
the reference stack recalls and reranks catalog candidates, and an LLM
selector injects at most two full skill bodies into the agent's context.
Both harnesses receive identical precomputed selections, allowing their
performance gains to be compared under the same procedural input. SkillsBench reports Pass@1, and GDPval and QwenClawBench
report mean rewards on \([0,1]\) multiplied by 100. Every configuration
is averaged over three runs, and gains are reported in points on these
scales. \Cref{app:eval-skills} gives the grading and accounting details.

\begin{figure}[htbp]
\centering
\includegraphics[width=0.95\linewidth]{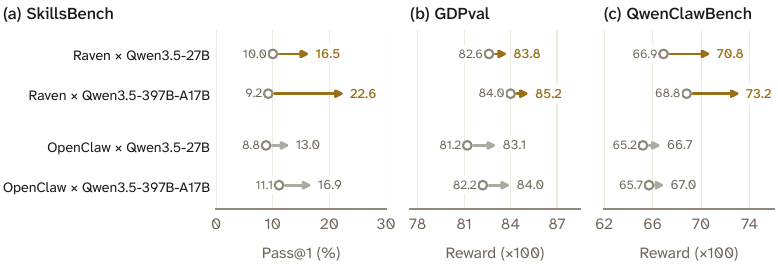}
\caption{Effect of the curated skill library reported by SkillCorpus
\citep{wang2026skillcorpus}. Each arrow runs from the no-skill baseline
(hollow marker) to the condition with SkillCorpus skills for one
harness--backbone cell. Values are three-run means. Horizontal axes
start at different values. \Cref{tab:skills-source} lists the values.}
\label{fig:skills-source}
\end{figure}

\paragraph{Consistent Gains Across Benchmarks.}
Skill retrieval improves task performance in all twelve
cell--benchmark comparisons in \cref{fig:skills-source}. This consistent
pattern spans both harnesses, both backbones, and all three benchmarks.
SkillCorpus pools each benchmark by
averaging each task's paired gain over the four cells and computing a
task-clustered standard error. The pooled gains are \(+7.5\pm2.3\)
points on SkillsBench, \(+1.51\pm0.49\) on GDPval, and \(+2.79\pm0.70\)
on QwenClawBench, with \(z\)-scores of 3.2, 3.1, and 4.0.
The gains span different baseline performance levels: SkillsBench starts
near 10\% Pass@1, whereas GDPval and QwenClawBench start between 65 and
85 points.
A single-run check with \modelname{Claude Opus 4.7} in OpenClaw raises
SkillsBench Pass@1 from 39.1\% to 47.1\%, showing that the library also
provides useful procedural knowledge to a stronger backbone.

\paragraph{Larger Gains with the Raven Harness.}
Raven obtains larger gains than OpenClaw from identical skill selections
on SkillsBench and QwenClawBench with both backbones. On SkillsBench,
Raven gains 6.5 and 13.4 points with \modelname{Qwen3.5-27B} and
\modelname{Qwen3.5-397B-A17B}, compared with 4.2 and 5.8 points for
OpenClaw. On QwenClawBench, the corresponding gains are 3.9 and 4.4
points for Raven and 1.5 and 1.3 points for OpenClaw. With
\modelname{Qwen3.5-397B-A17B}, Raven reaches the highest SkillsBench
Pass@1 in the grid, 22.6\%. In one run, adding skills resolves 19 tasks
that fail without skills, with regressions on 2 tasks (exact
McNemar test, \(p<0.001\)). These gains demonstrate Raven's ability to
turn the same retrieved procedures into more successful task executions.

Trace inspection in SkillCorpus provides a qualitative explanation for
this difference. On tasks that Raven solves and OpenClaw fails, OpenClaw
often stops after reasoning, leaves generated scripts unexecuted, or
finishes without running the verifier. Raven completes an
execute--verify--fix loop. The traces connect its stronger skill
utilization to carrying procedures through execution and checking their
outcomes.

On GDPval, however, all four cells improve by 1.2 to 1.9 points, with
per-task reward standard deviations of 20 to 27 points. These smaller
individual-cell gains are not statistically distinguishable from zero.

\paragraph{Contributions of Curation and Retrieval.}
To examine which components contribute to these gains, SkillCorpus reports
single-run ablations of catalog curation and trained retrieval. The analysis
uses the strongest configuration, Raven with
\modelname{Qwen3.5-397B-A17B} on SkillsBench
(\cref{tab:skills-ablation}). Replacing the
fine-tuned retriever and reranker with off-the-shelf Qwen3 embedding and
reranking models lowers Pass@1 from 22.6\% to 13.8\%. Replacing the
curated catalog with the raw crawl of about 821,000 skill files lowers
it to 14.9\%. Both ablated configurations remain above the no-skill
baseline of 9.2\%, while the full pipeline performs best. The results
support combining a curated procedural library with retrieval models
that can identify applicable skills.
As a standalone retriever on the 75 core SkillsBench
tasks, the fine-tuned stack matches SkillRouter \citep{zheng2026skillrouter} on its Easy tier
with a Hit@1 of 0.760 and exceeds it on the Hard tier, 0.760 against
0.720. This improvement on the harder retrieval tier complements the
downstream gains from the trained retrieval pipeline.

\paragraph{Larger Gains with Better Skill Coverage.}
Gains also vary with skill coverage: tasks with higher skill-coverage scores
show larger performance gains.
SkillCorpus uses the reranker score of the
highest-ranked skill as a proxy for task coverage. Pooled over the four
cells, the mean SkillsBench
gain rises from 2.2 points for the 26 tasks scoring below 0.45 to 6.2
points for the 40 tasks between 0.45 and 0.75 and 25.1 points for the
17 tasks at or above 0.75. Per task, the score correlates with the gain
at Pearson \(r=0.31\) to \(0.40\) across the four cells. The score is
essentially uncorrelated with the no-skill baseline, and controlling for
the baseline leaves partial correlations of 0.34 to 0.40.
However, Mathematics with 4 tasks and Finance \& Economics with 9 tasks show
gains near zero on Raven with \modelname{Qwen3.5-397B-A17B}. The overall
relationship between coverage and improvement supports task-aware skill
retrieval and identifies catalog coverage as a concrete direction for
extending its benefits across domains.

\section{Related Work}\label{sec:related-work}

\subsection{Multi-Agent Collaboration and Orchestration}

Multi-agent collaboration can be organized around conversation, a fixed
workflow, a dynamically revised task state, or an executable dependency
graph. The choice determines what the system can inspect before execution,
how it transfers intermediate results, and where coordination failures are
observed. Raven requires an orchestration interface that accommodates heterogeneous
model--harness pairs and their different execution capabilities.

Conversation-oriented frameworks coordinate agents through configurable
message exchanges.
AutoGen lets developers compose agents through messages and configurable
conversation patterns \citep{wu2023autogen}. MetaGPT assigns software
roles and encodes a development procedure in prompts and artifacts
\citep{hong2023metagpt}. These systems support role specialization through communication patterns
and roles configured by the developer. However, such a configuration does not by
itself provide a typed graph that can be validated before execution. The same
issue appears in systems that use a prescribed sequence of agent calls:
the sequence is inspectable, but adapting it to a new task often requires
editing the workflow itself.

Other systems make the execution structure more explicit or more adaptive.
MacNet represents interaction with a directed acyclic graph (DAG), although
its agents are drawn from a comparatively homogeneous collaboration pool
\citep{qian2025macnet}. Magentic-One maintains a task ledger and assigns
agents as the ledger changes during execution \citep{fourney2024magentic}.
MasRouter learns to route requests among large language models in a
multi-agent system, while Dang et al. evolve orchestration strategies over
tasks \citep{yue2025masrouter,dang2025orchestration}. These approaches
address executor selection or workflow adaptation, but they generally
assume a stable roster and a training or evaluation distribution from which
routing behavior can be learned. However, executor selection alone does not specify the artifact contract required
for a subsequent handoff.

Studies of scaling show that agent count alone is insufficient to predict
collaboration performance. Kim et al. report that coordination gains
depend on the interaction between task structure and system architecture
\citep{kim2026scaling}. More agents can add coverage when their capabilities
are complementary, while redundant messages and extra handoffs can consume
the budget without improving the result. AgentBench exposes related limits
in long-horizon agent tasks, where tool use, state tracking, and recovery
remain coupled \citep{liu2024agentbench}. MAST provides an empirical
taxonomy of multi-agent failures that includes coordination, message, and
execution effects \citep{cemri2025mast}. Together these results motivate
evaluating the structure of a collaboration plan separately from the final
quality of an artifact.

Formal methods offer useful vocabulary for making that structure testable.
Hoare logic connects a program's preconditions, commands, and postconditions
\citep{hoare1969axiomatic}. Separation logic explains how local reasoning
can remain sound when concurrent components own disjoint resources
\citep{ohearn2007resources}. Assume--guarantee reasoning and its probabilistic
variants express how one component's assumptions constrain another
component's guarantees \citep{kwiatkowska2010assume}. Recent discussions of
contracts for language-model agents argue that deployment needs several
layers of probabilistic guarantees rather than an informal message protocol
\citep{bensalem2026contracts}. These approaches motivate explicit executor identities, input scopes,
output conditions, and resource limits, although verifying open-ended
language-model execution remains unresolved.

Raven uses this contract perspective at the orchestration boundary. A Host
Agent can delegate directly or submit a DAG whose nodes identify an
executor, a prompt template, dependencies, input bindings, and optional
stateful handles. The runtime checks identifiers, acyclicity, path scope,
backend capabilities, registration status, recursion, and the shared
dispatch budget before it starts a worker. Every completed node produces a
recorded artifact that later nodes reference by content or path. The host
therefore selects a model--harness pair at invocation time from a
heterogeneous registry, while the graph fixes the predecessor relation and
the handoff representation. This separates planning from execution and
allows the same interface to represent a one-agent delegation or a
multi-stage workflow.

\subsection{Agent Harnesses and Self-Evolution}

An agent harness is the execution policy around a language model. It can
define the initial context, reasoning and action loop, tool interfaces,
memory access, iteration limits, verification rules, and recovery behavior.
This view distinguishes the model that proposes tokens from the system that
turns those tokens into a bounded task execution. ReAct couples reasoning
traces with actions in a prompt-level loop \citep{yao2023react}, while
SWE-agent shows that an agent--computer interface can be a decisive part of
software-engineering performance \citep{yang2024sweagent}. Harness
engineering therefore includes the interaction surface and control policy
alongside the backbone model \citep{weng2026harness}.

Several methods improve only part of this execution policy. GEPA evolves
prompts through reflective feedback and keeps the underlying agent program
fixed \citep{agrawal2026gepa}. ADAS searches over agent designs, and AFlow
generates workflow programs before deployment \citep{hu2024adas,zhang2025aflow}.
AgentSquare recombines four predefined modules, planning, reasoning, tool
use, and memory, within a modular design space \citep{shang2025agentsquare}.
These methods show that textual instructions, workflow structure, and
modular control logic can all be optimization variables. Their search spaces and evaluation
loops differ, however, so an improvement in one prompt or workflow does not
imply a reusable representation for changing another harness or another
agent family.

Open-ended evolution extends the search to executable agent code. The
Darwin G\"odel Machine edits its own implementation and evaluates descendants
in an open-ended loop \citep{zhang2026dgm}. Growing Harness takes a narrower
position on a frozen model. The model and tools remain fixed while an offline optimizer revises the
harness using execution traces and diagnostic artifacts, including
evaluator feedback and errors unavailable to the deployed harness \citep{li2026growing}.
Compared with AgentSquare's module recombination, the optimization target
is the harness as a whole, and development and service run in separate
loops. The resulting specialist harnesses are then reused across tasks.
However, separating the loops still leaves a credit-assignment problem. A candidate may change a
recovery branch, a memory instruction, or a tool policy while the measured
task score reflects all of them together. Without a stable edit vocabulary
and a fixed screening protocol, the evolver can select a patch whose causal
effect is unclear or whose gain does not transfer.

Beyond generating candidates, evolutionary search can retain multiple solutions.
MAP-Elites maintains high-performing elites across cells that partition
the search space by behavior
\citep{mouret2015illuminating}. HarnessBank applies this principle to
agent-harness self-evolution with a semantic gene bank, diagnosed failure
categories, recombination, and gated verification \citep{luo2026harnessbank}.
Its screening procedure checks that a proposed intervention actually
activates, that the score comparison is paired on the same tasks, and that
the observed gain passes an operational threshold. These checks combine evidence of activation with paired performance
measurements during candidate selection.

In Raven, the harness is factored into Memory, Planning, Capability, and
Action strategy interfaces. This boundary lets the evolver propose a targeted
change and the evaluator attribute it to an edit category within the
diagnosis, gene-bank, and gated-screening procedure. The development loop searches and
screens candidates, then freezes the selected harness before service use.
The Host Agent remains responsible for choosing among the resulting
model--harness pairs and composing them with other agents. This separation
distinguishes Raven from prompt-only optimization, workflow synthesis, and
unrestricted self-editing while retaining their useful search primitives.

\subsection{Long-Term Memory for Agents}

Agent memory research varies along several independent dimensions: the type
of record, the time scale, the retrieval index, the treatment of provenance,
and whether maintenance can proceed while the agent serves a request. A
recent survey separates factual, experiential, and working memory and
emphasizes that a memory system is more than a vector store
\citep{hu2026memorysurvey}. This taxonomy is useful for Raven because user
facts, agent cases, task artifacts, and reusable procedures have different
owners and different consequences when they are retrieved.

Early systems emphasize context management and reflection. MemGPT treats
the model as an operating system that moves information between memory
tiers with different access costs \citep{packer2023memgpt}. Generative
Agents record observations, retrieve relevant episodes, and periodically
reflect to form higher-level memories \citep{park2023generative}. A-MEM
constructs and updates a network of linked note cards, while Zep represents
agent memory as a temporal knowledge graph with source and validity
information \citep{xu2025amem,rasmussen2025zep}. These systems illustrate
different trade-offs between a compact working context, associative
retrieval, and explicit temporal structure. These distinctions motivate explicit scope and temporal information for
records used in later decisions.

Beyond context management, production-oriented memory systems address scale and operational updates,
using different principles to organize memory. Mem0 runs a two-phase
extraction and update pipeline and leaves hierarchical memory architectures
to future work \citep{chhikara2025mem0}. MemoryOS divides storage by
conversational time scale into short-term, mid-term, and long-term personal
memory \citep{kang2025memoryos}. MemOS instead distinguishes memory by
substrate, separating plaintext, activation, and parameter memory
\citep{li2025memos}. MIRIX keeps procedural
memory alongside other memory types in one multi-agent system
\citep{wang2025mirix}. Letta's sleep-time updates demonstrate that memory
consolidation can run asynchronously so that a foreground request does not
wait for every update \citep{letta2025sleeptime}. These systems improve
availability and reuse, yet a memory write and a task dependency remain
different kinds of relation. A record may describe a previous attempt
without being a required predecessor for the current computation.

EverOS organizes memory through extraction, consolidation, and retrieval
processes for long-horizon reasoning, with separate user and agent tracks
\citep{hu2026evermemos,evermind2026everos}.
Its user track can
retain profiles and episodes, while its agent track records cases and
experience-derived skills. The distinction gives an agent a durable source
of experience without treating another user's private context as a shared
resource. It also supports asynchronous writes and retrieval adapters that
can be selected by the host application.

Raven uses these memory facilities to supply context within the
orchestration structure. On a user-facing invocation, the host may add a
lexical profile slice and selected EverOS user memories. For a specialist,
the adapter can retrieve agent cases and skills within that agent's scope.
The collaboration layer decides what is visible at each invocation, and
the prompt marks retrieved text as untrusted context. A worker's session
record can be stored for later memory extraction, while the graph's
predecessor artifacts remain explicit in the artifact ledger. This separation keeps retrieved experience distinct from the declared
artifact dependencies and authorized inputs of each invocation.

\subsection{Skill Acquisition, Retrieval, and Evolution}

Skills provide a middle ground between one-off tool calls and changes to a
general harness. They package a task class, instructions, and optional
resources so that an agent can reuse a procedure across requests. Voyager
learns executable code skills in an embodied environment
\citep{wang2024voyager}. The Agent Skills specification standardizes a
lightweight \texttt{SKILL.md} format with a name, description, and
on-demand loading \citep{agentskills2025}. ExpeL stores natural-language
insights from successful experience, and SkillWeaver discovers and refines
procedures for web interaction \citep{zhao2024expel,zheng2025skillweaver}.
LATM instead treats reusable tool construction and request-conditioned tool
selection as the central problem \citep{cai2024latm}. These methods differ
in whether the learned object is code, instructions, or a tool wrapper, yet
all make procedural knowledge available outside the immediate context.

The open-skill ecosystem introduces a separate corpus problem. A public
catalog can contain duplicate names, near-duplicate bodies, malformed
metadata, unsafe instructions, and procedures whose descriptions do not
match their implementation. SkillCorpus addresses these issues with source
collection, structural filtering, semantic deduplication, quality and
safety judgments, release admission, and retrieval artifacts
\citep{wang2026skillcorpus}. Its reference stack rewrites a request when
needed, retrieves and reranks candidates, and lets an LLM select a bounded
number of full skill bodies. This separation between corpus curation and
online selection matters because a high retrieval score alone does not show
that a procedure is executable or appropriate for the receiving agent.

Skill maintenance is increasingly treated as an explicit operation.
Memp separates skill construction, retrieval, and update, which makes it
possible to evaluate procedural memory independently from initial discovery
\citep{zhang2026memp}. ExpeL and SkillWeaver update procedures from
experience, but their learned artifacts are tied to the task environments
and feedback signals used during acquisition. A persistent skill also needs
versioning, supporting cases, and a policy for rejecting regressions. These
requirements become more demanding when one catalog serves agents with
different tools and context limits.

Raven builds a single task-aware selection path over three sources: local
skills, memory-derived skills, and SkillHub. It fuses source hits by name,
keeps provenance, applies bounded retrieval and body budgets, and uses a
compatibility gate before injecting selected procedures. The update path is
separate. An execution case enters EverOS, where clustering and skill
operators can propose a new or revised procedure for later indexing.
Retrieval can therefore continue while revisions are pending. Persisted
revisions become available for subsequent selection after indexing.

Taken together, the four research lines are complementary at different levels of the Raven
architecture. Orchestration determines which executor runs and which
artifacts it receives. Harness evolution changes how that executor reasons,
uses tools, and recovers. Memory preserves user context and agent
experience, while Skill Forge exposes selected procedures to future runs.
Raven preserves explicit boundaries between these mechanisms. Memory records
remain distinct from graph edges, skill selection remains distinct from tool
permission, and harness updates remain distinct from the host's routing
policy. The evaluation therefore reports planning, harness adaptation,
specialist execution, and skill reuse as separate capability levels.

\section{Conclusion and Future Work}\label{sec:conclusion}

We introduced \raven, \emph{The Harness of Harnesses}, an open-source
multi-agent ecosystem that addresses the growing complexity of harness
design and the difficulty of composing specialized capabilities across
domains. Raven treats each executable model--harness pair as a unit of intelligence
that can be constructed, improved through experience, and composed with
other agents. The Host Agent organizes
native and third-party agents through explicit task dependencies, while
modular harness evolution, persistent memory, and Skill Forge support
adaptation and reuse across tasks. Our theory establishes sufficient
conditions under which composition expands reliable task coverage beyond
that of the available individual agents under a shared resource budget,
accounting for planning and coordination costs.

To assess the system's components, our evaluation examines planning quality, specialist performance, harness
adaptation, and skill reuse.
On MAOB, \raven ranks first among the compared systems on all four
planning metrics under both tested backbones, improving Exact Match over the strongest baseline by $10.4$ and $10.5$
percentage points, respectively.
Evaluations of the four native specialists characterize task performance
and resource use in research, software engineering, visual design, and
sustained operation. The published HarnessBank experiments demonstrate
the benefits of harness adaptation with a frozen backbone, while the
SkillCorpus experiments show that a curated skill library improves
\raven on three benchmarks, with larger gains than OpenClaw on two of
them.

Building on this ecosystem, our future work centers on an interconnected multi-agent system
spanning everyday devices and environments. We envision a \raven
instance for each device or environment, managing and orchestrating all
agents within its scope. We will develop communication and coordination
mechanisms between these instances, using protocols for agent-to-agent (A2A) communication to exchange
capability information, delegate tasks, and return results. Each agent will remain managed by its local \raven instance, while the
instances coordinate collaboration across devices. Connecting all of
a user's agents in this way would form an \emph{All-Domain Collaboration
Network}, making the network's capabilities accessible through requests from any
connected device or interaction interface.

We aim for this network to function as a persistent AI assistant that
maintains continuity across devices and adapts through experience.
To provide this continuity, we will connect the memory systems of \raven
instances, transfer task context, and synchronize execution progress so
that users can move between interfaces while ongoing work continues. User
preferences and accumulated experience will inform coordination across
the network, allowing an interaction begun on one device to continue on
another without requiring the user to reconstruct its history. Harness
evolution and reusable skills will support continued adaptation as the
network encounters new tasks and environments. End-to-end evaluations
will guide this development by measuring task success, continuity of
user experience, and coordination cost.

\begingroup
\raggedright
\bibliographystyle{min4}
\bibliography{references}
\endgroup

\clearpage
\appendix
\crefalias{section}{appendix}
\crefalias{subsection}{subappendix}
\section{Additional Theory Details}\label{app:theory-details}

\subsection{Notation and Scope}\label{app:theory-notation}

\begin{table}[htbp]
\centering
\small
\caption{Principal notation. States and correctness predicates are analytical
objects. An executor receives only its permitted observations.}
\label{tab:theory-notation}
\begin{tabularx}{\linewidth}{@{}lX@{}}
\toprule
Symbol & Meaning \\
\midrule
$a_i,M_i,h_i$ & Execution policy, underlying model, and harness. \\
$\mathcal H_i,\mathcal U_i,\Delta$ & Permitted histories, actions, and probability distributions. \\
$H,a_0,\mathcal A$ & Host controller, its nondelegating executor, and the pool including $a_0$. \\
$S,S_H$ & A generic system, either the composed system or a standalone agent, and the composed system $H[\mathcal A]$. \\
$t,r_t,\mu_t,\mathsf E_t,\mathsf V_t$ & Task, request, initial law, response kernel, and objective correctness. \\
$\omega,\operatorname{ref},s,\ell,Y$ & Reference conditions, their read-only state projection, full state, immutable-record ledger, and final record. \\
$B,C_S,\delta$ & Total resource budget, actual accrued cost, and tolerated failure probability. \\
$p_S,\mathcal C_S$ & Budgeted success probability and reliable task set. \\
$\pi,G,\sigma,\lambda$ & Annotated plan, DAG, executor assignment, and expanded operation order. \\
$K_{\rm op},\lambda_j,\mathcal X_v,\mathcal Y_v$ & Number of operations, the $j$th operation, and node input and output spaces. \\
$\Phi_v,d,\mathcal D_{\rm pkt},\mathcal M_e$ & Input assembly, authorized packet, packet space, and edge message space. \\
$P_v,Q_v,J_e$ & Node precondition, node transition guarantee, and handoff guarantee. \\
$A_j,D_j,I_j$ & Operation entry/transition predicates and the invariant after its prefix. \\
$\Gamma,\mathbf b$ & Contract/invariant annotations and planning/node/edge allocations. \\
$\kappa,\mathsf K_\kappa,\nu$ & Operation provider, conditional outcome kernel, and entry law. \\
$\mathcal Z_\kappa,\mathcal O_\kappa,\Pr_{\nu,\kappa}$ & Provider configuration and outcome spaces, and the local experiment. \\
$\mathfrak L,\Xi_j,O_j$ & Admissible entry laws and actual invocation configuration/outcome. \\
$\Sigma_j,L_j,C_j$ & State, ledger, and charged cost at operation $j$. \\
$Z,\Pi,\mathsf G_t,\mathsf S_t,q_t$ & Planning transcript, selected plan, valid-planning event, success event, and selected-plan success. \\
$R_{j-1},\mathsf{Ok}_j,F_j$ & Observable prefix, operation-success event, and success through the current operation. \\
$\epsilon_j,\epsilon(\zeta,\pi)$ & Local conditional failure bound and its sum for a selected plan. \\
$\eta_H,\bar\epsilon,\eta_\star,\epsilon_\star$ & Planning-failure bound, uniform valid-plan execution envelope, and their family-level bounds. \\
$\mathcal R_v,\mathcal C^{\rm new},\mathcal T_\star$ & Eligible realizers, new-task coverage, and a specified task family. \\
$\operatorname{rec}_\pi,\mathcal I,\mathcal L$ & Final-record projection, admitted initial pairs, and ledger space. \\
$\Omega_{\rm tr},\mathcal W,\mathcal Y_{\rm rec}$ & Trace, reference-condition, and execution-record spaces. \\
\bottomrule
\end{tabularx}
\end{table}

All task, state, contract, and probability objects refer to fixed system
versions and a stated input law. In conditional arguments, the planning
values $(\zeta,\pi)$ are fixed while execution randomness remains. Conditional probabilities are used only for conditioning events with
positive probability. The countable encoding convention makes this
conditioning unambiguous without requiring probability densities.

\subsection{Operation Records and Execution Semantics}
\label{app:theory-protocol}

The ledger is a finite map whose existing records are immutable. Write
$\ell[k]$ for the record stored under identifier $k$ and $\ell[k\mapsto r]$
for extension by a fresh identifier $k$. The extension is defined
only when $k\notin\operatorname{dom}(\ell)$. A node record $r_v$ contains
its actual input $\operatorname{in}(r_v)$, output
$\operatorname{out}(r_v)$, and relevant invocation trace. These fields
record execution behavior independently of the worker's account of it.
An edge record contains the delivered message. The root record contains
the authorized packet and task-reference bookkeeping. Only permitted
fields are exposed to an executor.

For $v\in V$, let
$x_v=\Phi_v(d,(\ell[e])_{e\in\operatorname{In}(v)})$. The entry and
transition predicates in the main text are defined by
\begin{align}
 A_v(s,\ell)
   &\Longleftrightarrow
     v\notin\operatorname{dom}(\ell)
     \land \operatorname{In}(v)\subseteq\operatorname{dom}(\ell)
     \land P_v(x_v,s),\label{eq:node-entry-record}\\
 D_v(s,\ell,s',\ell')
   &\Longleftrightarrow
     \exists r_v:\ \ell'=\ell[v\mapsto r_v]\notag\\
   &\qquad{}\land\operatorname{in}(r_v)=x_v
     \land Q_v(x_v,s,\operatorname{out}(r_v),s').
     \label{eq:node-transition-record}
\end{align}
All referenced records must also have their declared types. Otherwise,
these predicates are false. Every $D_v$ and $D_e$ additionally requires
$\operatorname{ref}(s')=\operatorname{ref}(s)$, preserving the read-only task reference.
For $e=(u,v)$, put
$y_u=\operatorname{out}(\ell[u])$. The transfer predicates are
\begin{align}
 A_e(s,\ell)&\Longleftrightarrow
    e\notin\operatorname{dom}(\ell)
       \land u\in\operatorname{dom}(\ell),\notag\\
 D_e(s,\ell,s',\ell')&\Longleftrightarrow
    \exists m:\ \ell'=\ell[e\mapsto m]\land J_e(y_u,s,m,s').
 \label{eq:edge-records}
\end{align}
Costs and acceptance are outcome fields, so successful completion also
requires acceptance and compliance with the allocation. Transitions
$D_v,D_e$ express semantic correctness independently of those fields.
The correspondence with the schedule is
$(A_j,D_j)=(A_{\lambda_j},D_{\lambda_j})$ and
$\epsilon_v=\epsilon_{\lambda^{-1}(v)}$,
$\epsilon_e=\epsilon_{\lambda^{-1}(e)}$.

With these record definitions, serial execution starts from the post-planning
state $(\Sigma_0,L_0)$. Operation
$\lambda_j$ is attempted only after its predecessors in $\lambda$ have completed
and been accepted. Its outcome supplies $(\Sigma_j,L_j)$, cost $C_j$, and its
acceptance status. The protocol need not observe the truth of $D_j$.
A rejection, crash, timeout, or nontermination can prevent plan completion.
If every operation succeeds, the protocol completes the finite sequence,
delivers the final execution record
\(\operatorname{rec}_\pi(L_{K_{\rm op}},\Sigma_{K_{\rm op}})\) in the final operation, and terminates. Final
delivery and its cost belong to that operation. The success event therefore includes any delay or failure associated
with final delivery.

Extending the analysis to a concurrent protocol requires each of its
runs to have a measurable serial representation respecting $\lambda$ with
the same invocation inputs and outputs, relevant state transitions,
acceptance outcomes, terminal record, and total charged cost. Probability
bounds then use the law induced by the actual concurrent runs on these
representations. Standalone serial error rates cannot be substituted.
This trace-preservation premise permits the pathwise deterministic lemma
and the probability argument to apply. Footprint separation is useful for
proving such premises in suitable program models
\citep{ohearn2007resources}, but does not alone guarantee preservation of
timing-sensitive service responses, acceptance, or cost. The grammar
compilation result below is proved for its serial protocol. Applying it to concurrent execution requires the same additional premise.

\paragraph{Ledger Extension.}
A compiled fragment reads only its own record identifiers and its declared
input interface. If a ledger $\ell_{\rm ext}$ has a disjoint domain and none
of its keys is read by the fragment, adding $\ell_{\rm ext}$ leaves its input
resolvers and entry predicates unchanged. Each transition above lifts by
retaining those extra records in both $\ell$ and $\ell'$. This follows by
substitution in \cref{eq:node-entry-record,eq:node-transition-record,eq:edge-records}:
freshness and all referenced values are preserved, and the state predicates
do not inspect unrelated records. The fragment's invariants lift by
restricting the larger ledger to the fragment's identifiers and interface
records. No claim is made for arbitrary predicates that inspect the size
or unrelated contents of the global ledger.

\subsection{A Finite-Budget Information Limit}\label{app:theory-search}

The following search problem illustrates the effect of a finite information
budget in a black-box setting consistent with no-free-lunch results
\citep{wolpert1997nfl}. Let $\mathcal X$ contain $N_{\rm pos}$ positions, where
$N_{\rm pos}\ge2$ is an integer, and let the unknown target be $z\in\mathcal X$.
The oracle returns $f_z(x)=0$ at $x=z$ and $1$ otherwise. An algorithm $A$
has no side information about $z$, its internal seed is independent of
$z$, and it returns an evaluated position $\widehat x$ after at most
$n_{\rm qry}$ distinct queries, where $n_{\rm qry}$ is an integer with $1\le n_{\rm qry}<N_{\rm pos}$.

\begin{proposition}[Finite-query coverage limit]
\label{prop:finite-query-limit}
Every possibly randomized and adaptive algorithm in this setting satisfies
\begin{equation}
    \frac{1}{N_{\rm pos}}\sum_{z\in\mathcal X}
       \mathbb P_A(\widehat x=z\mid f_z)\le\frac{n_{\rm qry}}{N_{\rm pos}}.
    \label{eq:query-limit}
\end{equation}
Thus some target has success probability at most $n_{\rm qry}/N_{\rm pos}$.
\end{proposition}
\begin{proof}
Fix the internal seed and follow the branch on which every response is
$1$. It queries at most $n_{\rm qry}$ positions. If $z$ lies outside this set, the
actual transcript follows exactly that branch, so the returned evaluated
position cannot be the target. Hence at most $n_{\rm qry}$ targets permit success
for this seed. Averaging over a uniform target and over the independent
seed proves the inequality. At least one target must then have a success probability no greater than
the bound on the mean.
\end{proof}

If $n_{\rm qry}/N_{\rm pos}<1-\delta$, reliable coverage of every target is impossible.
A host and workers sharing $n_{\rm qry}$ total queries induce another adaptive
querying algorithm and obey the same bound because parallel queries can be
serialized without adding observations. If an unevaluated final guess is
allowed, the all-failure branch adds at most one successful target, giving
$\min\{1,(n_{\rm qry}+1)/N_{\rm pos}\}$ instead. The result does not exclude advantages from
structured priors or exhaustive search with enough queries.

\subsection{From Local Experiments to Workflow Reliability}
\label{app:theory-risk}

Fix the actual law $\mathbb P=\mathbb P_{t,S_H,B}$ and a valid planning
pair $(\zeta,\pi)$. For a host-visible prefix value $\varrho$, let
\[
 \mathsf H_{j,\varrho}
    =F_{j-1}\cap\{R_{j-1}=\varrho,Z=\zeta,\Pi=\pi\},
 \qquad
 \nu_{j,\varrho}(c_{\rm cfg})=\mathbb P(\Xi_j=c_{\rm cfg}\mid \mathsf H_{j,\varrho}),
\]
whenever $\mathbb P(\mathsf H_{j,\varrho})>0$. The configuration $\Xi_j$ contains
the operation's entry state and ledger as defined by the protocol.
Write $A_j(c_{\rm cfg})$ for evaluation of $A_j$ on those components.

\begin{lemma}[Local-to-workflow consistency]\label{lem:local-to-workflow}
Suppose provider $\kappa_j$ realizes
$(A_j,D_j;b_{\lambda_j},\epsilon_j)$ over $\mathfrak L_j$.
For each positive-probability $\mathsf H_{j,\varrho}$, assume
$\nu_{j,\varrho}\in\mathfrak L_j$ and
\begin{equation}
 \mathbb P(O_j=o\mid\Xi_j=c_{\rm cfg},\mathsf H_{j,\varrho})
       =\mathsf K_{\kappa_j}(o\mid c_{\rm cfg})
 \label{eq:kernel-consistency}
\end{equation}
for all $c_{\rm cfg}$ of positive conditional mass and all outcomes $o$.
Then
\begin{equation}
 \mathbb P(\mathsf{Ok}_j^c\mid R_{j-1}=\varrho,F_{j-1},Z=\zeta,\Pi=\pi)
       \le\epsilon_j,
 \label{eq:history-uniform-error}
\end{equation}
and \cref{eq:conditional-local-error} follows.
\end{lemma}
\begin{proof}
The successful-prefix part of \cref{lem:composition-soundness} gives
$A_j$ at entry. Using the actual definition of $\mathsf{Ok}_j$ and
\cref{eq:kernel-consistency},
\begin{align*}
 \mathbb P(\mathsf{Ok}_j\mid \mathsf H_{j,\varrho})
   &=\sum_{c_{\rm cfg},o}\nu_{j,\varrho}(c_{\rm cfg})\mathsf K_{\kappa_j}(o\mid c_{\rm cfg})
                 \operatorname{ok}_{D_j,b_{\lambda_j}}(c_{\rm cfg},o)\\
   &=\Pr_{\nu_{j,\varrho},\kappa_j}
                 (\operatorname{ok}_{D_j,b_{\lambda_j}})
     \ge1-\epsilon_j .
\end{align*}
The final inequality is contract realization for the specified entry law.
Let $\mathsf H_j=F_{j-1}\cap\{Z=\zeta,\Pi=\pi\}$ and
$w_j(\varrho)=\mathbb P(R_{j-1}=\varrho\mid \mathsf H_j)$.
For $\mathbb P(\mathsf H_j)>0$,
\begin{equation}
 \mathbb P(\mathsf{Ok}_j^c\mid \mathsf H_j)
   =\sum_{\varrho:w_j(\varrho)>0}
        w_j(\varrho)\mathbb P(\mathsf{Ok}_j^c\mid \mathsf H_{j,\varrho})
   \le\epsilon_j .
 \label{eq:history-averaging}
\end{equation}
This is the required successful-prefix bound.
\end{proof}

Kernel consistency links the local experiment to the actual invocation.
Configurations can retain full relevant history or latent state. If
planning or earlier outcomes change the conditional law of an omitted
variable, a kernel calibrated only on visible inputs may not satisfy
\cref{eq:kernel-consistency}. The configuration must then include the missing information, or the
outcome model must be refined. These analytical representations leave
the worker's permitted observations unchanged.

Conditioning on prior correctness is necessary for the stated local bound.
For example, let $U$ be
a fair bit, $R_1$ constant, $\mathsf{Ok}_1=\{U=0\}$, and $\mathsf{Ok}_2=\{U=1\}$.
Then $\mathbb P(\mathsf{Ok}_2^c\mid R_1)=1/2$ but
$\mathbb P(\mathsf{Ok}_2^c\mid R_1,F_1)=1$, while joint success is zero.
Thus the conditional bound cannot generally omit $F_1$, and the joint
success probability cannot generally be obtained by multiplying the two
marginal success probabilities. It does not refute Boole's bound
using actual marginal failure rates. Those rates also yield a union bound,
but isolated-provider error rates are not automatically the actual
marginals or successful-prefix rates of a selected workflow.

Beyond the additive bound, the same successful-prefix bounds imply
\begin{equation}
 \mathbb P(F_{K_{\rm op}}\mid Z=\zeta,\Pi=\pi)
   \ge\prod_{j=1}^{K_{\rm op}}(1-\epsilon_j)
   \ge\left[1-\sum_{j=1}^{K_{\rm op}}\epsilon_j\right]_+.
 \label{eq:conditional-product}
\end{equation}
If all factors are positive, induction ensures positive-probability
prefixes and permits multiplication of their conditional bounds.
If a factor is zero the product lower bound is trivial.
The final inequality follows by expanding two factors,
$(1-a)(1-b)\ge1-a-b$ for $a,b\in[0,1]$, and induction.
This argument does not require independence.

For plans with at most $n_0$ nodes and $r_0$ handoffs, suppose node and
handoff cost allocations are at most $b^{\max}_V$ and $b^{\max}_E$ and their
conditional errors are at most $\epsilon^{\max}_V$ and $\epsilon^{\max}_E$,
respectively. Here $n_0,r_0$ are nonnegative integers,
$b^{\max}_V,b^{\max}_E\ge0$, and $\epsilon^{\max}_V,\epsilon^{\max}_E\in[0,1]$.
A sufficient design-time resource envelope is
\begin{equation}
 b_H+n_0b^{\max}_V+r_0b^{\max}_E\le B.
 \label{eq:family-cost-envelope}
\end{equation}
If valid planning has probability at least $1-\eta_\star$, with $\eta_\star\in[0,1]$,
\cref{thm:composition-reliability,prop:host-reliability} give
\begin{equation}
 p_{S_H}(t;B)\ge(1-\eta_\star)[1-n_0\epsilon^{\max}_V-r_0\epsilon^{\max}_E]_+.
 \label{eq:family-size-bound}
\end{equation}
The envelope gives a sufficient design-time condition for resource
feasibility. A large error sum weakens the lower bound without establishing
a corresponding decrease in actual performance.

These local reliability bounds connect to executor selection through the
eligible sets, as formalized next.

\begin{proposition}[Eligible assignments]\label{prop:eligible-assignment}
Fix a decomposition with nonempty eligible sets $\mathcal R_v$, and
choose $\sigma(v)\in\mathcal R_v$ for every node. Suppose that this assignment
and its instance bindings give a compatible, feasible plan, that the transfers
realize their contracts, and that the actual contexts satisfy
\cref{eq:kernel-consistency} and entry-law membership for all operations.
At every valid planning pair the plan then satisfies
\cref{eq:conditional-local-error}, hence the bound in
\cref{thm:composition-reliability}. If
$\bigcap_v\mathcal R_v=\varnothing$, no constant assignment realizes all
the node contracts under those same allocations and law families.
\end{proposition}
\begin{proof}
Membership in $\mathcal R_v$ gives each selected node's realization.
Together with transfer realization and the context conditions,
\cref{lem:local-to-workflow} gives every successful-prefix error bound.
The reliability theorem applies. A constant realizing assignment would
select an executor belonging to every $\mathcal R_v$, contradicting the
empty intersection.
\end{proof}

Actual invocation laws can depend on the whole assignment. Therefore,
nonempty eligible sets alone do not certify their context-membership
premise. The proposition concerns this decomposition and does not exclude
other ways an individual agent might solve the complete task.

\subsection{A Typed Task Grammar and Its Compilation}
\label{app:theory-family}

To establish compatible-plan existence constructively, we define a restricted
task grammar. Its primitive specifications, input
distributions, and admissible composition rules are fixed before observing
the host's success.

\paragraph{Primitive Specifications and Syntax.}
A primitive $p\in\mathcal P$ has input/output types
$(\mathcal X_p,\mathcal Y_p)$, precondition $P_p(x,s)$, postcondition
$Q_p(x,s,y,s')$, an assigned executor in $\mathcal A$, and a declared state
footprint. Its relational specification is the set of tuples $(x,s,y,s')$
with $P_p(x,s)=1$ and $Q_p(x,s,y,s')=1$.
A postcondition specifies the required result. A separate realization bound
quantifies the executor's ability to achieve it. Expressions
have the syntax
\begin{equation}
    \psi ::= p
        \ \mid\ \operatorname{Seq}_{\phi}(\psi_1,\psi_2)
        \ \mid\ \operatorname{Par}_{p_{\rm fork},p_{\rm join}}(\psi_1,\psi_2).
    \label{eq:task-grammar}
\end{equation}
The primitives $p_{\rm fork},p_{\rm join}\in\mathcal P$ are explicit fork and join operations.
A type judgment $\psi:\mathcal X\rightsquigarrow\mathcal Y$ means that the expression
accepts an input of type $\mathcal X$ and returns an output of type
$\mathcal Y$, with precondition $P_\psi$ and postcondition $Q_\psi$.
Write $\mathcal X_\psi$ and $\mathcal Y_\psi$ for these declared input
and output types. Admitted expressions form the least class
containing the primitives and closed under the two guarded rules below.

\paragraph{Sequential Rule.}
For $\psi_i:\mathcal X_i\rightsquigarrow\mathcal Y_i$, let
$\phi:\mathcal Y_1\to\mathcal X_2$ be a typed projection or serialization
from the specified wiring library.
The sequential expression has type
$\mathcal X_1\rightsquigarrow\mathcal Y_2$ and is admitted only if
\begin{equation}
    P_{\psi_1}(x,s)\land Q_{\psi_1}(x,s,y_1,s_1)
       \Longrightarrow P_{\psi_2}(\phi(y_1),s_1)
    \label{eq:sequential-side-condition}
\end{equation}
for all well-typed values. Define $P_\psi=P_{\psi_1}$ and
\begin{equation}
 \begin{split}
    Q_\psi(x,s,y,s')\ \Longleftrightarrow\
       \exists y_1,s_1:\;&Q_{\psi_1}(x,s,y_1,s_1)\\
          &{}\land Q_{\psi_2}(\phi(y_1),s_1,y,s').
 \end{split}
 \label{eq:sequential-semantics}
\end{equation}
Thus the output relation composes the two specifications through the
declared transfer. The connecting handoff contract requires exactly $\phi(y_1)$ and
preserves the mutable state. Its cost and error are explicit. Any content transformation outside the
wiring library, including deterministic computation, must be represented
as another primitive instead of being included implicitly in $\phi$.

\paragraph{Parallel Rule and State Separation.}
For $\psi_i:\mathcal X_i\rightsquigarrow\mathcal Y_i$, require
$p_{\rm fork}:\mathcal X\rightsquigarrow(\mathcal X_1\times\mathcal X_2)$ and
$p_{\rm join}:(\mathcal Y_1\times\mathcal Y_2)\rightsquigarrow\mathcal Y$.
The constructor is admitted only if its branch implementations and contracts
have a common state decomposition
\[
    \mathcal S=\mathcal S_{\rm sh}\times\mathcal S_1
                \times\mathcal S_2\times\mathcal S_{\rm rest}.
\]
The shared component is read-only during the branches. Subject to its
executor's permitted observations, a branch's state accesses are confined
to the shared and its own components. It may change only its own component.
For $k\in\{1,2\}$, there must exist local predicates
$\widetilde P_k,\widetilde Q_k$ such that
\begin{align}
 P_{\psi_k}(x,s)&=\widetilde P_k(x,s^{[\rm sh]},s^{[k]}),
    \label{eq:branch-local-pre}\\
 Q_{\psi_k}(x,s,y,s')&\ \Longleftrightarrow\
   \widetilde Q_k(x,s^{[\rm sh]},s^{[k]},y,(s')^{[k]})\notag\\
 &\qquad{}\land (s')^{[\rm sh]}=s^{[\rm sh]}
   \land (s')^{[\neg k]}=s^{[\neg k]}.
    \label{eq:branch-frame}
\end{align}
Here $s^{[\neg k]}$ collects the other mutable components. These footprint
conditions also constrain the implementations. Under these conditions, each branch preserves the other's precondition in
the serial composition proof. Nested
partitions refine the parent branch's mutable component and may inherit only
its read-only shared components. They cannot expose a sibling's mutable
state. The join runs after both branches. Concurrent use additionally
requires the trace-preservation premise in \cref{app:theory-protocol}.

The fork must establish both branch preconditions:
\begin{equation}
 \begin{split}
    P_{p_{\rm fork}}(x,s)\land Q_{p_{\rm fork}}(x,s,(x_1,x_2),s_{\rm fork})
       \Longrightarrow
       P_{\psi_1}(x_1,s_{\rm fork})\land P_{\psi_2}(x_2,s_{\rm fork}).
 \end{split}
 \label{eq:parallel-fork-condition}
\end{equation}
Use branch order $\psi_1$ then $\psi_2$ for the logical serialization. The join
condition, universally quantified over the fork and branch results, is
\begin{equation}
 \begin{split}
   &P_{p_{\rm fork}}(x,s)\land Q_{p_{\rm fork}}(x,s,(x_1,x_2),s_{\rm fork})\\
   &\quad{}\land Q_{\psi_1}(x_1,s_{\rm fork},y_1,s_1)
              \land Q_{\psi_2}(x_2,s_1,y_2,s_2)\\
   &\hspace{30mm}\Longrightarrow
            P_{p_{\rm join}}((y_1,y_2),s_2).
 \end{split}
 \label{eq:parallel-join-condition}
\end{equation}
This joint condition requires compatibility between the branch outputs. Set $P_\psi=P_{p_{\rm fork}}$ and define
\begin{equation}
 \begin{split}
 Q_\psi(x,s,y,s')\ \Longleftrightarrow\
  \exists x_1,x_2,y_1,y_2,s_{\rm fork},s_1,s_2:\;&
       Q_{p_{\rm fork}}(x,s,(x_1,x_2),s_{\rm fork})\\
   &{}\land Q_{\psi_1}(x_1,s_{\rm fork},y_1,s_1)\\
   &{}\land Q_{\psi_2}(x_2,s_1,y_2,s_2)\\
   &{}\land Q_{p_{\rm join}}((y_1,y_2),s_2,y,s').
 \end{split}
 \label{eq:parallel-semantics}
\end{equation}
This relation specifies the combined task before any success measurement.
These conditions admit the serial fork--join constructor. Shared mutable
interactions require additional compatibility rules, and parallel
execution retains the separate trace-preservation obligation.

\paragraph{Task Interpretation and Root Binding.}
For the family construction, a root expression has a declared task-state
projection $\theta_\psi:\mathcal S\to\mathcal S_\psi$ through which its root
contract factors:
\[
 P_\psi(x,s)=\overline P_\psi(x,\theta_\psi(s)),\qquad
 Q_\psi(x,s,y,s')=
       \overline Q_\psi(x,\theta_\psi(s),y,\theta_\psi(s')).
\]
The projection includes every component relevant to the root specification.
Controller-private bookkeeping can be excluded only when the root contract
does not depend on it. For each admitted root, fix an input/reference law
$\mu_\psi$ on $(x_{\rm ref},u_{\rm ref})$ supported on
$\overline P_\psi(x_{\rm ref},u_{\rm ref})$.
The task $t_\psi=(r_\psi,\mu_\psi,\mathsf E_\psi,\mathsf V_\psi)$ requests relation $Q_\psi$ on the
supplied input, uses the specified primitive-tool environment
$\mathsf E_\psi$, and evaluates the delivered output $y$ and final state $s'$ recorded in $Y$
by
\[
 \mathsf V_\psi((x_{\rm ref},u_{\rm ref}),Y)
    =\overline Q_\psi(x_{\rm ref},u_{\rm ref},y,\theta_\psi(s')).
\]
Aborted and nonterminating records have value zero. The reference field of
the analytical state carries $(x_{\rm ref},u_{\rm ref})$.

Planning must preserve this reference meaning. A correctly bound entry
has $\theta_\psi(\Sigma_0)=u_{\rm ref}$, and its authorized packet's input field is
$x_{\rm ref}$. Thus planning may change private bookkeeping but may not
change the task state on which the requested relation depends.
Necessary task-state mutations belong in explicit primitive operations.

\paragraph{Compilation with Explicit Interfaces.}
Write $\operatorname{Compile}(\psi,\iota)$ for compilation with a resolver
$\iota$ that maps the authorized packet and declared interface messages to
$\mathcal X_\psi$. It returns a graph with one entry and one exit. At the top level, $\iota(d)=x_{\rm ref}$ is the input field of the authorized packet $d$. The notation
$\operatorname{Compile}(\psi)$ uses this root binding.
For a primitive, its single node has $\Phi_p=\iota$.
All embedded subgraphs receive fresh identifiers, and every identifier
reference is renamed consistently.

For $\operatorname{Seq}_\phi(\psi_1,\psi_2)$, compile $\psi_1$ with $\iota$, add a
handoff from its exit carrying $m=\phi(y_1)$, and compile $\psi_2$ with the
resolver that reads that message. The child entry consumes $m$, not the
parent's original packet input. For
$\operatorname{Par}_{p_{\rm fork},p_{\rm join}}(\psi_1,\psi_2)$, the fork node receives the input resolved by $\iota$ and
returns $(x_1,x_2)$. Two handoffs carry the respective projections, and
each child's resolver reads its corresponding message. Two exit handoffs
carry $y_1,y_2$, and the join's resolver constructs the ordered pair
$(y_1,y_2)$. The join is the new exit. Edge transformations and node
assembly are distinct charged stages, as in \cref{eq:input-assembly}.
All added transfers preserve mutable task state on a successful transition.

Fragment contracts take their input from this explicit interface.
Child invariants use the corresponding input value and entry state. They
are not reused with the parent's original input. Providers access only
declared packet fields and messages. The ledger-extension property in
\cref{app:theory-protocol} permits a fragment to execute with unrelated
parent or sibling records present. Instance bindings must separately
satisfy its state footprint because fresh names alone do not isolate memory.

The serial schedule is recursive, with a primitive call as its base case.
A sequence executes the first child, the connecting handoff, then the second
child. A parallel expression executes the fork, outgoing handoffs, first
child, second child, incoming handoffs, then the join.
Fresh subgraphs and these connections preserve acyclicity.
Only the whole plan's exit performs final delivery, while internal exits return
to their parent. Root-delivery work belongs to the exit's allocation and
provider model.

Let $\ell_{\rm root}(x,u)$ contain the correctly bound packet and the
immutable reference record, with all compiled identifiers initially fresh.
For a reference $(x,u)$ define
\[
 \mathcal I_\psi(x,u)
   =\{(s,\ell_{\rm root}(x,u)):
       \operatorname{ref}(s)=(x,u),\
       \theta_\psi(s)=u,\
       \overline P_\psi(x,u)=1\}.
\]
Allowed unrelated planning records may be added by ledger extension.
The terminal projection extracts the actual delivered result and state.
It cannot replace the original reference by a post-planning one.

\begin{proposition}[Sound compilation of the task grammar]
\label{prop:grammar-compilation}
For an admitted root expression $\psi$, its serial compilation is a finite
plan compatible with $t_\psi$ and every $\mathcal I_\psi(x,u)$.
For every compiled fragment, each nonterminal successful operation prefix
enables its next operation. Complete successful execution establishes the fragment's
declared input/output relation.
\end{proposition}
\begin{proof}
Induct on the syntax tree, retaining the fragment input, its entry state,
and its interface binding. The induction assertion includes internal
enabling conditions, postcondition correctness, and preservation under
unrelated ledger extension.

For the base case, a primitive's resolver supplies its declared input. Its precondition
enables the node and its transition establishes the primitive
postcondition. Freshness and the record definitions give ledger extension.

For a sequence, apply the first child's induction assertion. Every
internal successful prefix enables its next step. On completion it gives
$Q_{\psi_1}(x,s,y_1,s_1)$. The transfer is enabled by its source record,
and its success preserves $s_1$ and delivers $\phi(y_1)$.
Equation~\eqref{eq:sequential-side-condition} enables the second child's
explicitly rebound entry. Apply its full induction assertion, including
all its internal prefixes. The actual intermediate $(y_1,s_1)$ witnesses
\cref{eq:sequential-semantics}.

For a fork--join expression, the fork establishes both branch
preconditions by \cref{eq:parallel-fork-condition}. The outgoing
transfers deliver the respective inputs without changing the task state.
Apply the first child's assertion. Its frame preserves the second
precondition, so the second child can be invoked with its own bound input.
Its assertion gives all internal enabling steps and $Q_{\psi_2}$.
The exit records enable the two incoming transfers, and their ordered
messages satisfy the join precondition by
\cref{eq:parallel-join-condition}. The successful join gives $Q_{p_{\rm join}}$.
The actual intermediate states and outputs witness
\cref{eq:parallel-semantics}. Disjoint ledger identifiers and the extension
property justify each child invocation.

For each resulting graph, take $I_k$ to express reachability from its
declared initial set through the first $k$ transitions satisfying $D_j$.
The prefix assertions just proved establish \cref{eq:contract-enable}.
Appending a specified transition establishes \cref{eq:contract-preserve}.
The final relation gives
$Q_\psi(x,s_0,y,s')=\overline Q_\psi(x,u,y,\theta_\psi(s'))$ because
$\theta_\psi(s_0)=u$. This is $\mathsf V_\psi=1$, establishing
\cref{eq:contract-terminal}. Finite syntax yields a finite graph.
\end{proof}

This proposition proves semantic compatibility for serial execution.
Extension to a parallel implementation requires the trace-preservation
condition in \cref{app:theory-protocol}, beyond the graph structure and
state footprints.

\paragraph{Resource and Error Recurrences.}
To account for resource use and error in these compiled plans, let $n_{\rm node}(\psi)$
and $n_{\rm edge}(\psi)$ count expanded nodes and handoffs. Let $\operatorname{cost}(\psi)$ be
the sum of their allocated costs, excluding host planning, and let
$\operatorname{err}(\psi)$ sum their conditional error allocations. For a primitive,
\[
    n_{\rm node}(p)=1,\quad n_{\rm edge}(p)=0,\quad \operatorname{cost}(p)=b_p,\quad
    \operatorname{err}(p)=\epsilon_p.
\]
Here $p$ denotes an instantiated primitive occurrence, and $b_p$ and
$\epsilon_p$ are that call's cost and conditional-error allocations.
Repeated occurrences may carry different allocations. An occurrence at
the whole plan's exit includes final delivery in its allocation and
provider model, while an internal occurrence returns to its parent.
For a sequence with connecting handoff allocation $b_\phi^{\rm hd}$ and error
$\epsilon_\phi$,
\begin{align}
 n_{\rm node}(\psi)&=n_{\rm node}(\psi_1)+n_{\rm node}(\psi_2),&
 n_{\rm edge}(\psi)&=n_{\rm edge}(\psi_1)+n_{\rm edge}(\psi_2)+1,\label{eq:sequence-size}\\
 \operatorname{cost}(\psi)&=\operatorname{cost}(\psi_1)+\operatorname{cost}(\psi_2)+b_\phi^{\rm hd},&
 \operatorname{err}(\psi)&=\operatorname{err}(\psi_1)+\operatorname{err}(\psi_2)+\epsilon_\phi .
 \label{eq:sequence-budget}
\end{align}
For a parallel expression, denote its four connecting handoff allocations
by $(b_k^{\rm hd},\epsilon_k^{\rm hd})$, $k=1,\ldots,4$. Then
\begin{equation}
 \begin{aligned}
 n_{\rm node}(\psi)&=n_{\rm node}(\psi_1)+n_{\rm node}(\psi_2)+2,\\
 n_{\rm edge}(\psi)&=n_{\rm edge}(\psi_1)+n_{\rm edge}(\psi_2)+4.
 \end{aligned}
 \label{eq:parallel-size}
\end{equation}
\begin{align}
 \operatorname{cost}(\psi)&=\operatorname{cost}(\psi_1)+\operatorname{cost}(\psi_2)
                         +b_{p_{\rm fork}}+b_{p_{\rm join}}\notag\\
                       &\quad+\sum_{k=1}^4 b_k^{\rm hd},\label{eq:parallel-budget}\\
 \operatorname{err}(\psi)&=\operatorname{err}(\psi_1)+\operatorname{err}(\psi_2)
                         +\epsilon_{p_{\rm fork}}+\epsilon_{p_{\rm join}}\notag\\
                       &\quad+\sum_{k=1}^4\epsilon_k^{\rm hd}.
 \label{eq:parallel-error}
\end{align}
The fresh-copy construction gives these equalities by counting each fork,
join, and input/output transfer as an explicit operation.

However, these sums yield reliability bounds only if the declared contract
realizations apply to every invocation law induced by the compiled graphs,
conditional on planning and successful prefixes. Typing and structural
induction do not establish this distributional premise. For example,
a join input assembled from earlier outputs may have a different law
from its standalone evaluation inputs. The stated entry-law membership
condition must cover it.

\paragraph{An Independent Bounded Task Family.}
The primitive specifications, constructor rules, root projections, and
laws defining $t_\psi$ are fixed independently of observed agent success.
For integers $n_0\ge1$ and $r_0\ge0$, let
\begin{equation}
 \mathcal T_{n_0,r_0}
   =\{t_\psi:\psi\text{ is an admitted root},\
       n_{\rm node}(\psi)\le n_0,\ n_{\rm edge}(\psi)\le r_0\}.
 \label{eq:bounded-task-family}
\end{equation}
The family has bounded graph size but may contain infinitely many tasks.

For a task $t_\psi$, define $\mathsf B_\psi$ as the $(Z,\Pi)$-measurable event that
planning finishes within $b_H$, returns $\operatorname{Compile}(\psi)$ up to
consistent identifier renaming, and uses the declared serial protocol
(or its certified parallel implementation). Correct root binding must hold
throughout the conditional support of $(\Sigma_0,L_0)$: the packet supplies
$x_{\rm ref}$, $\theta_\psi(\Sigma_0)=u_{\rm ref}$, references and fresh identifiers
are preserved, and fragment interfaces are instantiated as specified.
These are planning-pair properties. Future success is not included in
$\mathsf B_\psi$.

\begin{corollary}[Coverage through correct compilation]
\label{cor:compiler-coverage}
For every $t_\psi\in\mathcal T_{n_0,r_0}$, suppose
$\mathbb P(\mathsf B_\psi)\ge1-\eta_\star$, where $\eta_\star\in[0,1]$.
On every positive-probability planning pair in $\mathsf B_\psi$, assume the
provider realizations and kernel-consistency hypotheses of
\cref{lem:local-to-workflow}. Let node and edge allocations be bounded
uniformly by $b^{\max}_V$ and $b^{\max}_E$ and their conditional errors by
$\epsilon^{\max}_V$ and $\epsilon^{\max}_E$.
If
\begin{equation}
 b_H+n_0b^{\max}_V+r_0b^{\max}_E\le B,\qquad
 \eta_\star+(1-\eta_\star)(n_0\epsilon^{\max}_V+r_0\epsilon^{\max}_E)\le\delta<1,
 \label{eq:grammar-coverage-conditions}
\end{equation}
then $\mathcal T_{n_0,r_0}\subseteq\mathcal C_{S_H}(B,\delta)$.
\end{corollary}
\begin{proof}
Correct binding and \cref{prop:grammar-compilation} establish semantic
compatibility on $\mathsf B_\psi$. The resource envelope establishes
feasibility, so $\mathsf B_\psi\subseteq\mathsf G_{t_\psi}$.
Local realization and consistency supply the successful-prefix bounds.
Let $\epsilon_\star=n_0\epsilon^{\max}_V+r_0\epsilon^{\max}_E$. The risk inequality implies
$\eta_\star<1$ and $\epsilon_\star<1$.
Thus \cref{thm:composition-reliability} gives
$q_{t_\psi}(\zeta,\pi)\ge1-\epsilon_\star$ on $\mathsf B_\psi$. Averaging on that event,
\[
 p_{S_H}(t_\psi;B)
   \ge\mathbb E[\mathbb I[\mathsf B_\psi]q_{t_\psi}(Z,\Pi)]
   \ge(1-\eta_\star)(1-\epsilon_\star)\ge1-\delta.
\]
Apply the capability definition to each task.
\end{proof}

The construction gives one sufficient way to meet the coverage conditions.
Other planning outcomes may yield successful executions that increase the
overall success probability beyond this lower bound. The general result
\cref{cor:relative-completeness} continues to apply to any host meeting
its stated valid-plan and reliability premises.

\clearpage
\section{Method Implementation Details}\label{app:method-details}

This appendix records the interfaces, implementation defaults, and
protocol provenance for \crefrange{sec:collaboration}{sec:skillforge}.
Experimental configurations are described in \cref{app:specialist-eval}.

\subsection{Notation Across Method Modules}\label{app:method-notation}

The theory's notation is collected in \cref{tab:theory-notation}.
Complementing this notation, Tables~\ref{tab:method-runtime-notation} and
\ref{tab:method-learning-notation} define the main method-specific objects.
Ordinary loop indices and finite-set dummy variables are local to their
stated subsection. A font, superscript, or descriptive subscript that
distinguishes an object type is part of its notation.

\begin{table}[!ht]
\centering
\small
\caption{Runtime and Memory Notation.}
\label{tab:method-runtime-notation}
\begin{tabularx}{\linewidth}{@{}lX@{}}
\toprule
Symbol & Meaning \\
\midrule
\(G_{\rm rt},G_{\rm ann}\) & Submitted graph and its analytical expansion. \\
\(\mathsf{Tpl}_v,\operatorname{Dep}(v),\operatorname{Bind}_v\) & Prompt template, dependency identifiers, and slot bindings. \\
\(\operatorname{Cfg}_v,\mathrm{inst}_v\) & Invocation capability overrides and stateful-instance handle. \\
\(j_{\rm rt},z_{j_{\rm rt}}(v)\) & Scheduler-event index and node status at that event. \\
\(x_v^{\rm task},x_v^{\rm base},x_v\) & Rendered task, task plus memory paths, and final dispatched input. \\
\(\mathsf{Paths}_v,\operatorname{MemPath}(u)\) & Advertised ancestor-memory paths and a node's reserved memory-file location. \\
\(\mathsf L_a,g_v\) & Owner-partitioned library and a worker's group-memory context. \\
\(\mathsf{Verd}_{r,a}^{(p)},\mathrm{id}_H\) & Agent assessment at round/phase and its host author identifier. \\
\(\mathsf{Hist},\mathsf{seg}_i,\mathsf{ep}_i\) & Interaction stream, source segment, and episode narrative. \\
\(\mathsf{scene}_j,\boldsymbol{\mu}_j,n_j,t_{j,{\rm last}}\) & Scene, centroid, member count, and latest timestamp. \\
\(\mathsf{Cell}_i,\mathbf z_i,\mathbf z_c\) & Memory cell, episode embedding, and case embedding. \\
\(\operatorname{Rel}_i(q),\operatorname{Rel}_f(q)\) & Calibrated episode and fact relevance scores. \\
\(k_{\rm group},k_{\rm verdict}\) & Per-library group recall and recent-verdict caps. \\
\(k_{\rm mem},k_{\rm recall}\) & Backend memory-candidate and final adapter-result caps. \\
\(\mathsf{Prof}_u,\operatorname{Segments}\) & User profile and source-segment collection operator. \\
\(\Delta_{\rm time},\epsilon_{\rm num}\) & Clustering time window and numerical stabilizer. \\
\(c_{\rm RRF}^{\rm mem},\alpha_{\rm mix}\) & Memory-retrieval rank damping and fact-score mixing weight. \\
\(\operatorname{MemContext}(q)\) & Host archive and user recall rendered for query \(q\). \\
\(\mathbin{+\!\!+},\operatorname{Fence}_{B_{\rm char}}\) & Ordered concatenation and untrusted-context wrapper under a character budget. \\
\(L^{\rm hit}_\ell(q),\operatorname{RRF}_i(q),d_{\rm emb}\) & Channel hit list, episode expansion priority, and embedding dimension. \\
\bottomrule
\end{tabularx}
\end{table}

\begin{table}[!ht]
\centering
\small
\caption{Harness Evolution and Skill Notation.}
\label{tab:method-learning-notation}
\begin{tabularx}{\linewidth}{@{}lX@{}}
\toprule
Symbol & Meaning \\
\midrule
\(\mathrm{id}_h,\operatorname{Exec}(h)\) & Candidate identity and executable components. \\
\(K_{\rm att},N_{\rm keep},R_{\rm round}\) & Attempts per task, survivor cap, and round cap. \\
\(\mathsf{cell},\operatorname{desc}(h)\) & Archive cell and the candidate's assigned descriptor. \\
\(\Pi_{\rm eval},\lambda_{\rm screen}\) & Fixed evaluation protocol and paired-gain screening threshold. \\
\(\mathcal X_{\rm edit},\mathcal K_{\rm fixed}\) & Editable and protected harness locations. \\
\(\Delta^{\rm patch}\) & Executable harness patch. \\
\(\operatorname{Util},\mathsf{Log}\) & Task-averaged utility and indexed evaluation record. \\
\(\mathsf{Log}^{\rm sc}_{h'},\mathsf{Log}_h\) & Fresh screening record and cached full-training record. \\
\(\mathrm{active}_{i,k},\mathrm{stale}\) & Activation indicator and consecutive non-updating round count. \\
\(\gamma_i,\bar\gamma_r,t_{\rm pair}\) & Paired task gain, mean paired gain, and paired statistic. \\
\(\mathcal B_r,\bot\) & Gene bank at round \(r\) and the value of an empty cell. \\
\(\mathsf{s},\mathsf{h}\) & Skill record and source retrieval hit. \\
\(\mathrm{util}_{\mathsf s},\mathrm{rob}_{\mathsf s},\mathrm{safe}_{\mathsf s}\) & Normalized utility, robustness, and safety facets. \\
\(\mathcal U_{\rm tool},\mathcal A_{\rm spec}\) & Receiving agent's available tools and specialist roster. \\
\(\operatorname{qid},\operatorname{name}\) & Qualified hit identifier and display-name extraction. \\
\(\mathsf n,L_j^{\rm hit},\mathsf{Hits}_{\mathsf n}\) & Display-name value, source hit list, and hits sharing that name. \\
\(\mathsf{Pool},\mathsf{IDs},S_{\rm sel}\) & Filtered candidate list, gate-returned identifiers, and selected entries. \\
\(k_{\rm pool},k_{\rm sel},k_{\rm fb}\) & Fused skill-pool, normal selection, and fallback caps. \\
\(c_{\rm RRF}^{\rm skill},\mathcal J_{\rm src}\) & Skill rank damping and retrieval-source set. \\
\(\operatorname{RRF}(\mathsf n),\mathsf{Ops},\mathsf{op}_l\) & Fused name score, skill-update operation list, and one operation. \\
\(\operatorname{own}(a),\mathsf{Skills}_C,\mathsf{Cases}_C\) & Agent-memory partition, existing skill list, and supporting-case pool. \\
\(\mathsf{Catalog}_C,\mathsf{UsedTargets},\mathsf{Emitted}\) & Cluster skill collection, reserved update targets, and emitted records. \\
\(\mathrm{qual}_c,\mathrm{qual}_{\rm cur}(\mathsf{s}),\chi(\mathsf{s})\) & Case quality, catalog quality, and skill confidence. \\
\bottomrule
\end{tabularx}
\end{table}

\subsection{Graph and Prompt Contracts}\label{app:method-contracts}

\Cref{tab:method-contracts} specifies the fields exchanged at the graph
and model interfaces. Graph admission follows \cref{tab:collab-admission},
and skill response handling follows \cref{eq:skill-selection-decision}.
The case-to-skill operation format is defined in
\cref{sec:skillforge-evolution}.

\begin{center}
\begin{minipage}{\linewidth}
\captionsetup{hypcap=false}
\centering
\small
\captionof{table}{Graph and Model Interface Fields.}
\label{tab:method-contracts}
\begin{tabularx}{\linewidth}{@{}>{\raggedright\arraybackslash}p{0.17\linewidth}>{\raggedright\arraybackslash}X@{}}
\toprule
\textbf{Interface} & \textbf{Fields and Constraints} \\
\midrule
Graph call & Required: nonblank \texttt{task\_summary} and a \texttt{nodes} array.
Optional booleans: \texttt{background} and \texttt{confirm}. \\
Node & Required nonblank strings: \texttt{id}, \texttt{subagent},
\texttt{node\_summary}, \texttt{prompt\_template}.
Optional: \texttt{depends\_on} (identifier array), \texttt{inputs} (map),
and \texttt{instance} (nonblank string). \\
Input value & Exactly one literal string, \texttt{\{"file": "path"\}},
or \texttt{\{"node": "id"\}}. Reference objects have one key and a
nonempty value. Template references and declared input keys must agree. \\
Node verdict & Required \texttt{outcome}: \texttt{accomplished} or
\texttt{not\_accomplished}. Optional negative-verdict fields:
\texttt{category}, \texttt{what\_is\_missing}, \texttt{evidence}. \\
Verdict category & \texttt{missing\_user\_input}, \texttt{missing\_credential},
\texttt{tool\_failure}, \texttt{dependency\_output\_unusable},
\texttt{output\_limit}, or \texttt{other}. \\
Node resolution & Requires \texttt{run\_id}, \texttt{node\_id}, and
\texttt{decision}. Both \texttt{continue} and \texttt{replan} require a
nonblank \texttt{message}. Replanning also requires a nonempty successor
\texttt{nodes} array. \texttt{abandon} requires neither payload. \\
Query rewriting & JSON fields: \texttt{need\_retrieval} (boolean) and
\texttt{rewritten\_query} (string or null). \\
Skill selection & JSON fields: \texttt{plan} (rationale string) and
\texttt{skills} (array of source-qualified identifiers, possibly empty).
The parser requires the array and permits an omitted rationale. \\
\bottomrule
\end{tabularx}
\end{minipage}
\end{center}

\subsection{Implementation Defaults}\label{app:method-defaults}

\Cref{tab:method-defaults} collects numerical settings used by the
method implementations. EverOS defines the memory and case-to-skill defaults. Raven defines the
defaults for graph execution, recall adaptation, and skill routing.
Experiment-specific settings are described in \cref{app:specialist-eval}.

\begin{center}
\begin{minipage}{\linewidth}
\captionsetup{hypcap=false}
\centering
\small
\captionof{table}{Implementation Defaults. Harness-evolution run settings are
separated from these defaults in \cref{app:harness-screening-policy}.}
\label{tab:method-defaults}
\begin{tabularx}{\linewidth}{@{}l>{\raggedright\arraybackslash}Xl@{}}
\toprule
\textbf{Mechanism} & \textbf{Setting} & \textbf{Value} \\
\midrule
DAG execution & Shared concurrency for graph nodes and spawns & 8 \\
DAG execution & Private fallback concurrency without a shared gate & 5 \\
DAG admission & Dispatch budget for graph runs and spawns & 30 per hour \\
DAG verdict & Judge timeout / transcript-tail budget & 180 s / 8,000 characters \\
DAG recovery & Additional continuations per node & 2 \\
DAG recovery & Decision deadline for a suspended node & 600 s \\
DAG output & Terminal-output cap in the run report & 128,000 characters \\
Clarification & Host autofill timeout & 20 s \\
EverOS clustering & Cosine threshold / time window & 0.65 / 7 days \\
EverOS user retrieval & RRF damping / facts per expansion / unchanged limit & 60 / 3 / 10 \\
Raven user recall & Result clamp / profile-text cap & 1--100 / 1,200 characters \\
Raven user recall & Timeout / default result count & 5 s / 5 \\
Raven store queue & Session / global / in flight & 64 / 256 / 4 \\
Raven store retry & Attempts / waits of 2, 5, 15, 30 s & 5 \\
Skill rewriting & Request prefix & 2,000 characters \\
Skill routing & Over-fetch factor / gate pool / ungated pool & 2 / 10 / 2 \\
Skill routing & Local / memory / SkillHub weights & 1.0 / 0.9 / 0.85 \\
Skill routing & RRF damping & 10 \\
Skill selection & Normal cap / failure fallback & 2 / 5 \\
Skill selection & Description / body-excerpt limits & 200 / 300 characters \\
SkillHub detail policy & Minimum parseable safety score & 0.7 \\
Case admission & Minimum quality \(\theta_{\rm skip}\) for clustering & 0.2 \\
Case clustering & Recall cap / direct-merge similarity & 30 / 0.85 \\
Skill extraction & Existing-skill cap \(k_{\rm skill}\) / supporting-case cap \(k_{\rm case}\) & 10 / 9 \\
Skill extraction & Supporting summaries per existing skill & 3 \\
Skill extraction & Description / body prefix budgets & 400 / 5,000 tokens \\
Skill extraction & Supporting-approach prefix budget & 200 tokens \\
\bottomrule
\end{tabularx}
\end{minipage}
\end{center}

\paragraph{SkillHub Detail Policy.}
When discovery returns only metadata and a detail client is configured,
the minimum safety threshold applies if the retrieved detail score is
present and parseable.
Missing or malformed scores do not trigger that rejection. Detail checks
also reject blocked skills and incompatible external-home-path references.
The corpus release predicate in \cref{eq:skill-admission} is applied
separately. Token-prefix budgets in the table exclude omission markers.

\subsection{Harness Evolution Protocol and Evidence}
\label{app:harness-screening-policy}

The seven benchmark results in \cref{sec:eval-evolution,tab:harnessbank-source}
are reproduced from HarnessBank v2, Section 4 and Table 1
\citep{luo2026harnessbank}. The reported experiments use a frozen \modelname{Qwen3.6-27B} task
model, \modelname{Claude Opus 4.8} as the evolver, and three attempts per task, with
training-based selection followed by held-out testing. The method specification
in \cref{alg:evolution-screen,alg:evolution} states per-round sampling, cached
parent records, invalid-run retries, descriptor assignment, ties, and zero
variance explicitly. The published results do not establish the exact
values of every run-specific setting.

The source paper does not provide a complete per-run manifest for the seven
results. In particular, its result table does not identify each run's
screening subset and proposal order or all values of \((n,J,N_{\rm keep},R_{\rm round},P)\) and
retry limits. Reconstructing those runs therefore requires their run records.

\subsection{Source Prior for Skill Curation}
\label{app:method-skill-prior}

For corpus source \(a\), let \(n_a\) denote its skill count,
\(\overline{\mathrm{qual}}_a\) its mean normalized \emph{content} quality before the
source-prior mixture, and \(\mathrm{lic}_a\in[0,1]\) its license-coverage
rate under the corpus policy. These source statistics belong to the
curation snapshot. \(\overline{\mathrm{qual}}_a\) is not the final composite
\(\mathrm{qual}_{\rm cur}\). This distinction avoids circular score computation.
SkillCorpus uses
\begin{align}
    \operatorname{track}(a)
      &=\frac{n_a\overline{\mathrm{qual}}_a+K_0\mu_0}{n_a+K_0},\\
    p_{\rm src}(a)
      &=\operatorname{clip}_{[0.30,1]}
        \bigl(0.95\operatorname{track}(a)+0.05\,\mathrm{lic}_a+
           0.10\,\mathbb I[a\in\mathsf{TrustedSrc}]\bigr),
\end{align}
with shrinkage pseudo-count \(K_0=10\), prior mean content quality
\(\mu_0=0.685\), and
\(\mathsf{TrustedSrc}=\{\texttt{anthropics},
\texttt{antigravity},\texttt{karanb192}\}\).
For an empty source, define \(n_a\overline{\mathrm{qual}}_a=0\), so its track score
reduces to \(\mu_0\).
The shrinkage estimate of empirical quality contributes most to the prior.
The source reports that these settings were fixed before downstream
evaluation and were not optimized using benchmark outcomes.

\clearpage
\section{Evaluation Details}\label{app:specialist-eval}

This appendix supplements \cref{sec:evaluation} with scoring conventions,
evaluation configurations, grading procedures, and cost accounting.
It records departures from official benchmark protocols.
The main-text figures report orchestration and specialist scores. The
tables below reproduce the harness-evolution results, skill-library
results, and skill ablation from their cited sources.

\subsection{Multi-Agent Orchestration}\label{app:eval-maob}

\paragraph{Scoring Conventions.}\label{app:maob-conventions}
The metric definitions in \cref{subsec:mao-metrics} follow the DAGbench
scorer used for these results. Domain folding drops unrecognized agent
entries, flags dangling dependencies, removes within-domain edges, and
unions duplicate inter-domain edges. When distinct domains interleave in
an otherwise valid invocation graph, folding may introduce a cycle. The
adapter marks this as a projection conflict. Node F1 remains defined,
while Edge F1, POA, and Exact Match are unavailable for that record.

By contrast, a cyclic prediction without that projection-conflict flag receives zero
Node F1 and Edge F1 and a false Exact Match indicator. Its POA is zero
when a shared pair exists and unavailable otherwise. For ordinary acyclic
predictions, no shared node gives POA zero and a single shared node gives
POA one. Two empty edge sets have F1 one. These values follow explicit scoring conventions for degenerate cases.

Reference annotations can admit alternative edge sets, node folds, or
pair orders. The scorer expands explicit edge alternatives and allowed
node folds, evaluates each resulting reference, and selects the candidate
with the largest lexicographic tuple (Node F1, Edge F1, POA). Unavailable
order scores count as zero only for this selection, and ties retain the
first candidate. An empty prediction is scored against the original
reference without benefiting from node folds. Within a selected reference,
Edge F1 uses the best accepted reduced edge set, while POA and Exact Match
use the annotated accepted pair relations. With no such annotations the
reference family is a singleton.

Each reported metric is the arithmetic mean of its defined values over
scored multi-domain records. Unavailable values are omitted from that
metric's denominator, and an empty collection has no reported mean.
Defined zero scores, including ordinary empty-delegation predictions,
remain in the mean. Repeated attempts are separate records. Consequently,
different metrics can have different denominators, and the nominal
140-task count alone does not determine them.

\subsection{Harness Self-Evolution}\label{app:eval-evolution}

\Cref{tab:harnessbank-source} lists the values plotted in
\cref{fig:harnessbank-source}.

\begin{table}[htbp]
\centering
\caption{Held-out Pass@1 (\%) before and after harness self-evolution
with frozen \modelname{Qwen3.6-27B}. Scores and percentage-point gains are reproduced
as reported by HarnessBank.}
\label{tab:harnessbank-source}
\begin{tabular}{lrrr}
\toprule
\textbf{Benchmark} & \textbf{Initial} & \textbf{Evolved} & \textbf{Gain} \\
\midrule
Terminal-Bench-2 & 36.1 & 45.4 & +9.3 \\
LiveCodeBench   & 58.1 & 71.8 & +13.7 \\
Omni-MATH       & 54.3 & 66.0 & +11.7 \\
BrowseComp+     & 16.9 & 30.8 & +13.9 \\
GDPval          & 43.7 & 52.9 & +9.2 \\
AppWorld        & 41.3 & 56.7 & +15.4 \\
SWE-bench Verified & 47.4 & 52.6 & +5.1 \\
\bottomrule
\end{tabular}
\end{table}

\subsection{Raven-Research}\label{app:eval-research}

\paragraph{Benchmark Composition.}
DeepResearch Mixed combines questions from BrowseComp (English), FRAMES,
Humanity's Last Exam, and xBench-DeepSearch release 2510, which is in Chinese. Humanity's Last Exam is
evaluated in its text-only exact-match setting.

\paragraph{Grading, Tokens, and Cost.}
An LLM judge (\modelname{GPT-5.6 Luna Pro} at temperature 0) grades the short final
answer of each system for semantic equivalence with the reference
answer. Accuracy is the fraction of questions answered correctly, and
questions without an answer, including timeouts, count as incorrect.
Regrading the same answers changed accuracy by approximately 1.3
percentage points. Token
counts are per-question means of provider-reported usage, with cache
reads included in the input. For the \modelname{DeepSeek-V4-Flash} group, cost applies one
DeepSeek list-price schedule to the recorded usage of all three systems.
This gives the three systems a common pricing basis, although the
resulting estimates differ from the amounts billed. By contrast, costs of the commercial services are not on a
common basis: the Perplexity cost is measured, the Gemini cost converts
the billed rate, and the MiroThinker cost applies its list price.

\paragraph{Systems and Budgets.}
\modelname{Qwen3.6-35B-A3B} and a 4-bit GPTQ build of \modelname{Qwen3.5-397B-A17B} are
self-hosted on internal infrastructure with thinking enabled, and
\modelname{DeepSeek-V4-Flash} is deepseek-v4-flash, accessed through the official
DeepSeek API. MiroFlow \citep{su2026miroflow} runs its main agent with a 150-turn
limit and extracts its final answer with \modelname{Qwen3.6-35B-A3B} in all three
groups, including the \modelname{DeepSeek-V4-Flash} group. DeepSeek-Harness
\citep{deepseek2026harness} runs its default loop. The three local harnesses use Serper for
search and Jina Reader for page retrieval, and each keeps its own loop,
context policy, and sampling settings. Raven-Research was evaluated with
the two web tools, at most 150 tool iterations, 16,384 output tokens per
call, high reasoning effort, and a per-question limit of 90 minutes on
\modelname{DeepSeek-V4-Flash} (120 and 200 minutes on the two self-hosted models). The
three commercial services
\citep{miromind2026mirothinker17,perplexity2025sonardr,google2026geminidr}
ran end to end on the same questions with their own models and search.
Each system answered each question once, between August 10 and August
18, 2026.

\paragraph{Comparability.}
Because the harnesses retain different turn limits, time limits, and
context policies, these results provide an observational comparison
without isolating the effect of individual harness components. The same questions were also used while the research flow was
developed.
None of the evaluated systems blocked access to public copies of the
benchmarks. The absolute BrowseComp and Humanity's Last Exam scores may therefore
overestimate performance in an uncontaminated evaluation.

\subsection{Raven-Code}\label{app:eval-code}

\paragraph{Harness and Baselines.}
The self-run coding results were produced with AgentEval, an internal
evaluation framework that runs each attempt in a disposable Docker
container, records every model request through a usage proxy, and
grades the result in a fresh container with each benchmark's own
grading code. AgentEval invokes Raven-Code as a single-turn command-line
agent. Baselines were installed from their public packages: Claude Code
through claude-agent-sdk, OpenCode, Hermes Agent,
and DeepSeek-Harness through its SDK. Unless stated otherwise, each task
was attempted once, attempts that failed for infrastructure reasons
count as unsolved, and no task was excluded.

\paragraph{SWE-bench Pro.}
The evaluation uses the full public set of 731 tasks from 11
repositories with the official images and grading scripts. Generation
is offline: the agent container has no network interface, model traffic
passes through a Unix socket, and git objects that could reveal the
reference fix are removed. Grading is online because some test suites
install packages. Both systems use \modelname{Qwen3.8-27B} served by vLLM with a
budget of 3,600 seconds per task.

\paragraph{SWE-bench Verified.}
The 500 tasks are graded with the upstream grading functions, with 3,600
seconds per task. In contrast to SWE-bench Pro, the agent container has network access during
generation. The \modelname{DeepSeek-V4-Flash} systems run through the same AgentEval
setup, and the Claude Code entry for \modelname{Claude Opus 4.8} is the VALS.ai value for
the Claude Code harness \citep{valsai2026swebench}.

\paragraph{WorkBuddy-Code.}
Following the official protocol, each of the 80 tasks receives three
attempts of up to 3,600 seconds, and its reward is the mean hidden-test
pass rate over the attempts. The \modelname{Claude Opus 4.8} baselines are official
leaderboard values \citep{workbuddy2026leaderboard}, and all other
entries are self-run. The Raven-Code runs
with \modelname{DeepSeek-V4-Flash} and \modelname{Claude Opus 4.8} also blocked 35 code-hosting and search
domains, a stricter setting than that of the other self-run systems.

\paragraph{SWE-Refactor.}
Each of the 20 stack migrations is scored from 0 to 100 by a three-stage procedure: a model-based migration audit that acts as a gate, an
all-or-nothing behavioral test stage, and six adversarial verification
rounds. Each Raven-Code attempt had network access limited to the model
gateway, and the reported score combines task results from several
runs. For one task whose audit gate the benchmark maintainers
acknowledged as unreliable, the audit was repeated for Raven-Code. The
baselines are official leaderboard values at maximum reasoning effort
\citep{einsia2026srbleaderboard}. Raven-Code also uses maximum reasoning
effort with \modelname{GPT-5.6 Luna}. Reasoning effort is reported separately
from the model name.

\paragraph{DataAgentBench.}
Following the leaderboard protocol, Raven-Code ran five trials per query
with benchmark hints enabled. The official validator checked each answer.
Pass@1 averages the per-query pass rates within each dataset and then
averages the dataset scores. Pass@5 assigns one to a query
if at least one of its five trials succeeds and zero otherwise, then
uses the same dataset-level averaging. All 270 planned trials in this
run were graded, so these quantities require no missing-trial convention
for the reported run. The run took place on August 24,
2026, with a limit of 5,400 seconds per trial and 24 concurrent
trials. It used \modelname{Claude Opus 5} through
OpenRouter as the main model, together with a data-agent tool set and
DAB-specific checks. One
semantic-mapping step also used \modelname{Claude Haiku 4.5} and \modelname{Claude Sonnet 5}, so under the
leaderboard's categories the entry counts as a tuned-prompt system.
The entry has not been submitted to the leaderboard maintainers.

\subsection{Raven-Design}\label{app:eval-design}

\paragraph{Harness and Baselines.}
The design results were produced with AgentEval, with one attempt per
task, in containers with 2 CPUs and 8 GB of memory, in August 2026.
The evaluated deck pipeline fixes slide coordinates and typography, and
the agent fills the resulting layout slots. The visual-artifact
results use the rendering and preview tools, Task
State, and always-loaded design skills. Claude Code ran through claude-agent-sdk and Hermes Agent through the Hermes Agent
Agent package, with
built-in web search disabled for Claude Code.

\paragraph{PresentBench.}
Each task provides an instruction and background materials mounted
read-only, and the agent must produce an editable PPTX file. The
official evaluator converts the deck to PDF and its judge,
\modelname{Gemini 3 Flash}, answers about 54 binary checklist items per task. In
our runs the judge receives the rendered deck through a different file
transport from the official evaluator. The public leaderboard entries
were evaluated by the benchmark authors on all 238 tasks. However, not every
Raven-Design and Claude Code configuration was evaluated on the full
task set, and failed generations are excluded from the reported mean.
For Raven-Design with \modelname{GPT-5.6 Luna}, 229 tasks produced a valid deck and
228 were scored, and that run did not use image generation or image
search.

\paragraph{ArtifactsBench and GDPval.}
ArtifactsBench contributes its SVG generation tasks and its
data-visualization dashboard tasks after the removal of tasks with known
source defects. GDPval contributes its public gold subset, restricted to
tasks without audio or video that do not require live web search.
Agents run in multiple turns and declare their final files. An internal grader renders each deliverable and asks \modelname{GPT-5.6 Luna} to
assess it. This differs from both benchmarks' official protocols. For
configurations generated with \modelname{GPT-5.6 Luna}, the generator and judge use
the same model. For ArtifactsBench the grader captures timed
screenshots, replays frozen interaction intents with no-action controls,
and scores the official ten-item checklist from 0 to 100. For GDPval it
scores each rubric item of the public release as met or not met and
computes a weighted score from 0 to 100.

\subsection{Raven-Oncall}\label{app:eval-oncall}

\paragraph{AI4AI.}
The agent may modify only the training script of autoresearch. The data pipeline, tokenizer, held-out validation shard, and evaluation
function remain fixed. Package installation is prohibited, and the
parameter count may not exceed the starting model's 50,332,176 parameters. Every training run has a budget of 300 seconds of single-GPU
training time, and each campaign has 5 GPU-hours on a machine with two
A800 80GB GPUs. BPB is measured over 20,971,520 tokens of the held-out
shard, and each reported BPB is the mean over several seeds of the final
configuration (5 seeds for Claude Code, 6 for Raven-Oncall with \modelname{Claude Opus 5},
and 10 with \modelname{DeepSeek-V4.1-Flash}). The two Raven-Oncall campaigns share one
task statement, whereas Claude Code received a separately written
statement with the same budget and constraints. Runtime is the wall-clock time
of the whole campaign, including work outside the training runs, and
cost is the campaign total. The \modelname{Claude Opus 5} cost of Raven-Oncall is the amount billed
through OpenRouter, the \modelname{DeepSeek-V4.1-Flash} cost applies DeepSeek's peak-hour
prices, and the Claude Code cost applies the \modelname{Claude Opus 5} API prices to its
recorded usage.

\paragraph{AI4S.}
The 17 tasks cover engineering simulation, compute tuning, model training,
and other scientific and operational workloads. Seven tasks concern
engineering simulation. Three use CalculiX: an elastic-plastic cantilever
limit-load problem, a contact problem requiring evidence that the
specified penalty stiffness was retained, and a 25-minute limit-load task
that tests whether the agent requests additional budget. Four use
OpenFOAM: two dam-break runs with a divergent fixed step and a loose
Courant limit, a transient dam-break case with a seeded viscosity error,
and a steady turbulent bump case with approximately 880,000 cells.

Two compute-tuning tasks use \modelname{Qwen3-8B} for batch inference. One minimizes
batch wall time, and the other prohibits truncation. Three model-training
tasks use \modelname{Qwen3-Embedding-0.6B} on NFCorpus: fine-tuning, resuming an
interrupted job, and continuing a partially completed campaign from a
previous shift. The remaining tasks concern a two-dimensional heat
equation with an analytic solution, BM25 parameter tuning on NFCorpus, a
GitLab merge-request review bot and its pipeline, and two molecular
simulations. The molecular tasks estimate the density and self-diffusion
of liquid ethanol and the hydration free energy of phenol by
thermodynamic integration.

Each task has a specified budget in wall-clock, core, or GPU minutes.
Success is judged manually against the task statement and a task-specific
rubric containing reference values and supporting evidence withheld from
the agent. All systems received the same budgets and machines. For
simulation and batch inference, Claude Code's task statements also
included machine addresses and job commands, which Raven-Oncall selects
itself.

\subsection{Skill Retrieval and Reuse}\label{app:eval-skills}

\Cref{tab:skills-source} lists the values plotted in
\cref{fig:skills-source}, and \cref{tab:skills-ablation} lists the
ablation. All values are reproduced as reported by SkillCorpus
\citep{wang2026skillcorpus}.

\paragraph{Grading and Accounting.}
SkillsBench reports the binary pass rate over all 87 tasks. The 4 tasks
that do not build in the evaluation environment count as failures under
both conditions. A \modelname{GPT-4o} judge, from a different model family than the
evaluated backbones, grades GDPval. QwenClawBench uses its official
configuration of automated checks and an LLM judge. \modelname{Qwen3.5-397B-A17B}
runs as a 4-bit GPTQ build. Each backbone runs with its default
inference settings, with greedy decoding for the Qwen models and the
default thinking effort for \modelname{Claude Opus 4.7}. Both conditions of a cell
share these settings. Absolute baseline scores are therefore not directly comparable with
leaderboard values obtained under other configurations. Comparisons
within each cell, however, retain the same inference settings. A pooled gain averages each task's
paired difference over the four cells, and its standard error divides
the standard deviation of these averages by the square root of the
number of tasks. This task-clustered error is more conservative than
treating cell--task pairs as independent observations.

\paragraph{Frontier Check, Ablation, and Retrieval.}
The \modelname{Claude Opus 4.7} check runs OpenClaw on SkillsBench once per
condition. The ablation replaces exactly one component of the full
pipeline in a single run. The standalone retrieval evaluation follows
the SkillRouter protocol \citep{zheng2026skillrouter} on the 75 core
SkillsBench tasks, which exclude 12 generic-only cases.
Hit@1 is the fraction of these tasks for which the top-ranked skill
belongs to the task's designated ground-truth skill set under that protocol.

\begin{table}[htbp]
\centering
\small
\setlength{\tabcolsep}{4pt}
\caption{Effect of SkillCorpus skills on each harness--backbone cell.
Each benchmark gives the three-run mean without skills (None), with
skills (Skills), and their difference in points. SkillsBench reports
Pass@1 (\%), and GDPval and QwenClawBench report mean rewards multiplied
by 100. The last two rows give the task-clustered pooled gain with one
standard error and its \(z\)-score.}
\label{tab:skills-source}
\begin{tabular}{lrrrrrrrrr}
\toprule
& \multicolumn{3}{c}{\textbf{SkillsBench}}
& \multicolumn{3}{c}{\textbf{GDPval}}
& \multicolumn{3}{c}{\textbf{QwenClawBench}} \\
\cmidrule(lr){2-4}\cmidrule(lr){5-7}\cmidrule(lr){8-10}
\textbf{Cell} & None & Skills & Gain & None & Skills & Gain
& None & Skills & Gain \\
\midrule
Raven \(\times\) \modelname{Qwen3.5-27B}          & 10.0 & 16.5 & +6.5  & 82.6 & 83.8 & +1.2 & 66.9 & 70.8 & +3.9 \\
Raven \(\times\) \modelname{Qwen3.5-397B-A17B}    &  9.2 & 22.6 & +13.4 & 84.0 & 85.2 & +1.2 & 68.8 & 73.2 & +4.4 \\
OpenClaw \(\times\) \modelname{Qwen3.5-27B}       &  8.8 & 13.0 & +4.2  & 81.2 & 83.1 & +1.9 & 65.2 & 66.7 & +1.5 \\
OpenClaw \(\times\) \modelname{Qwen3.5-397B-A17B} & 11.1 & 16.9 & +5.8  & 82.2 & 84.0 & +1.8 & 65.7 & 67.0 & +1.3 \\
\midrule
Pooled gain & \multicolumn{3}{r}{\(+7.5\pm2.3\)}
& \multicolumn{3}{r}{\(+1.51\pm0.49\)}
& \multicolumn{3}{r}{\(+2.79\pm0.70\)} \\
\(z\)-score & \multicolumn{3}{r}{3.2} & \multicolumn{3}{r}{3.1}
& \multicolumn{3}{r}{4.0} \\
\bottomrule
\end{tabular}
\end{table}

\begin{table}[htbp]
\centering
\caption{Single-run ablation on Raven \(\times\) \modelname{Qwen3.5-397B-A17B}, SkillsBench
Pass@1 (\%) over 87 tasks. Each ablation replaces one component of the full pipeline, and the gain is measured against the no-skill baseline.}
\label{tab:skills-ablation}
\begin{tabular}{llrr}
\toprule
\textbf{Catalog} & \textbf{Retrieval Stack} & \textbf{Pass@1} & \textbf{Gain} \\
\midrule
\multicolumn{2}{l}{No skills} & 9.2 & -- \\
Curated & Fine-tuned & \textbf{22.6} & \textbf{+13.4} \\
Curated & Off-the-shelf Qwen3 & 13.8 & +4.6 \\
Raw crawl & Fine-tuned & 14.9 & +5.7 \\
\bottomrule
\end{tabular}
\end{table}

\clearpage
\section{Author List}\label{app:author_list}

\textbf{Project Leader:} Chuanrui Hu

\textbf{Core Contributors:} Dizhan Xue, Chuanrui Hu

\textbf{Contributors:} Zuyi Zhou, Hongda Chen, Xingze Gao, Zhao Wang, Pengfei Yao,
Zhengwei Wu, Tong Li, Ethan Wang, Xiaotian Luo, Juwei Yue, Jie Huang,
Hui Zhang, Weixiang Chen, Chang Zhang, Yuqi Yang, Yifan Chen, Yunyun Han

\textbf{Corresponding Author:} Yafeng Deng

\end{document}